\RequirePackage[l2tabu, orthodox]{nag}
\documentclass[12pt]{article}
\usepackage{setspace}
\usepackage{makecell}
\newcolumntype{M}[1]{>{\centering\arraybackslash}m{#1}}
\usepackage[T1]{fontenc}

\usepackage{fouriernc}

\usepackage{amsmath}
\usepackage{amsthm,lipsum}
\makeatletter
\def\thm@space@setup{%
  \thm@preskip=10pt plus 4pt minus 6pt
  \thm@postskip=\thm@preskip % or whatever, if you don't want them to be equal
}
\makeatother

\usepackage[scale=0.92]{tgschola}
\usepackage[scaled=0.88]{PTSans}
\usepackage{scalefnt,letltxmacro}
\LetLtxMacro{\oldtextsc}{\textsc}
\renewcommand{\textsc}[1]{\oldtextsc{\scalefont{1.25}#1}}
\usepackage{inconsolata}

\usepackage[ttdefault=true]{AnonymousPro}

\usepackage{amssymb}
\usepackage{mathbbol}
\usepackage{PTSans}
\usepackage{inconsolata}

\usepackage[usenames,dvipsnames]{xcolor}
\definecolor{shadecolor}{gray}{0.9}

\usepackage[english]{babel}
\usepackage[parfill]{parskip}
\usepackage{afterpage}
\usepackage{framed}
\usepackage{nicefrac}

{\endMakeFramed}

\DeclareRobustCommand{\parhead}[1]{\textbf{#1}~}

\usepackage{lineno}

\usepackage{ragged2e}

\usepackage{marginnote}

\newcounter{parcount}

\usepackage{graphicx}
\usepackage[labelfont=bf]{caption}
\usepackage[format=hang]{subcaption}
\usepackage{wrapfig}

\usepackage{booktabs}
\usepackage{multirow}
\usepackage{longtable}

\usepackage{natbib}
\usepackage{bibunits}

\usepackage[algoruled]{algorithm2e}
\usepackage{listings}
\usepackage{fancyvrb}
\fvset{fontsize=\normalsize}

\SetKwComment{Comment}{$\triangleright$\ }{}

\usepackage[colorlinks,linktoc=all]{hyperref}
\usepackage[all]{hypcap}
\hypersetup{citecolor=Violet}
\hypersetup{linkcolor=black}
\hypersetup{urlcolor=MidnightBlue}

\usepackage[nameinlink]{cleveref}

\usepackage[acronym,smallcaps,nowarn]{glossaries}
\usepackage{listings}
\usepackage{lstbayes}
\lstdefinestyle{mystyle}{
    commentstyle=\color{OliveGreen},
    numberstyle=\tiny\color{black!60},
    stringstyle=\color{BrickRed},
    basicstyle=\ttfamily\scriptsize,
    breakatwhitespace=false,
    breaklines=true,
    captionpos=b,
    keepspaces=true,
    numbers=none,
    numbersep=5pt,
    showspaces=false,
    showstringspaces=false,
    showtabs=false,
    tabsize=2
}
\usepackage{enumitem}

\usepackage{centernot}

\usepackage{pifont}% http://ctan.org/pkg/pifont
\newcommand{\cmark}{\ding{51}}%
\newcommand{\xmark}{\ding{55}}%

\DeclareRobustCommand{\mb}[1]{\ensuremath{\mathbf{\boldsymbol{#1}}}}
\DeclareMathOperator*{\argmin}{arg\,min}

\crefname{lemma}{lemma}{lemmas}
\Crefname{lemma}{Lemma}{Lemmas}
\crefname{thm}{theorem}{theorems}
\Crefname{thm}{Theorem}{Theorems}
\crefname{prop}{proposition}{propositions}
\Crefname{prop}{Proposition}{Propositions}

\newtheorem{thm}{Theorem} % reset theorem numbering for each chapter
\newtheorem{defn}{Definition} % definition numbers are dependent on theorem numbers
\newtheorem{prop}[thm]{Proposition}
\newtheorem{lemma}[thm]{Lemma}

\newtheorem{corollary}[thm]{Corollary}
\newcommand\independent{\protect\mathpalette{\protect\independenT}{\perp}}
\def\independenT#1#2{\mathrel{\rlap{$#1#2$}\mkern2mu{#1#2}}}

\renewcommand{\mid}{~\vert~}

\newcommand{\mbA}{\mb{A}}
\newcommand{\mba}{\mb{a}}

\newcommand\dif{\mathop{}\!\mathrm{d}}

\newcommand{\bbR}{\mathbb{R}}

\newcommand{\cD}{\mathcal{D}}

\newcommand{\cN}{\mathcal{N}}

\newcommand{\g}{\, | \,}

\newcommand{\indpt}{\protect\mathpalette{\protect\independenT}{\perp}}
\newcommand{\E}[2]{\mathbb{E}_{#1}\left[#2\right]}

\usepackage{booktabs,arydshln}
\makeatletter
\def\adl@drawiv#1#2#3{%
        \hskip.5\tabcolsep
        \xleaders#3{#2.5\@tempdimb #1{1}#2.5\@tempdimb}%
                #2\z@ plus1fil minus1fil\relax
        \hskip.5\tabcolsep}
\newcommand{\cdashlinelr}[1]{%
  \noalign{\vskip\aboverulesep
           \global\let\@dashdrawstore\adl@draw
           \global\let\adl@draw\adl@drawiv}
  \cdashline{#1}
  \noalign{\global\let\adl@draw\@dashdrawstore
           \vskip\belowrulesep}}
\makeatother

\newenvironment{proofsk}{%
  \proof}{\endproof}

\newacronym{KL}{kl}{Kullback-Leibler}
\newacronym{ELBO}{elbo}{\emph{evidence lower bound}}
\newacronym{POPELBO}{pop-elbo}{\emph{population evidence lower bound}}
\newacronym{PROELBO}{pro-elbo}{\emph{profile evidence lower bound}}

\newacronym{SVI}{svi}{stochastic variational inference}

\newacronym{ADVI}{advi}{automatic differentiation variational inference}

\newacronym{GMM}{gmm}{Gaussian mixture model}
\newacronym{LDA}{lda}{latent Dirichlet allocation}
\newacronym{PF}{pf}{Poisson factorization}
\newacronym{DEF}{def}{deep exponential family}
\newacronym{RMSE}{rmse}{root mean squred error}
\newacronym{SEM}{sem}{structural equation model}

\newacronym{SUTVA}{sutva}{stable unit treatment value assumption}

\newacronym{PPC}{ppc}{posterior predictive check}
\newacronym{GWAS}{gwas}{genome-wide association studies}
\newacronym{PPCA}{ppca}{probabilistic principal component analysis}
\newacronym{PCA}{pca}{principal component analysis}
\newacronym{KDE}{kde}{kernel density estimate}

\newacronym{LMM}{lmm}{linear mixed model}
\newacronym{LFA}{lfa}{logistic factor analysis}

\newacronym{SNP}{snp}{single-nucleotide polymorphism}
\newacronym{GLM}{glm}{generalized linear model}
\newacronym{CSI}{csi}{conditioning set inference}

\newacronym{TPR}{tpr}{true positive rate}
\newacronym{FPR}{fpr}{false positive rate}
\newacronym{FDR}{fdr}{false discovery rate}
\newacronym{BBVI}{bbvi}{black box variational inference}

\newacronym{SNR}{snr}{signal-to-noise ratio}
\newacronym{NUTS}{nuts}{No-U-Turn sampler}

\usepackage{tikz}
\usetikzlibrary{bayesnet}
\usepackage{pgfplots}
\pgfplotsset{compat=newest}
\pgfplotsset{plot coordinates/math parser=false}
\usepgfplotslibrary{statistics}

\pgfdeclarelayer{edgelayer}
\pgfdeclarelayer{nodelayer}
\pgfsetlayers{edgelayer,nodelayer,main}

\definecolor{hexcolor0xbfbfbf}{rgb}{0.749,0.749,0.749}

\tikzset{>=latex}
\tikzstyle{none}   = [inner sep=0pt]
\tikzstyle{line}   = [ thick, -, shorten <=1pt, shorten >=1pt ]
\tikzstyle{arrow}  = [ thick,  ->, shorten <=1pt, shorten >=1pt ]
\tikzstyle{ardash} = [ thick dotted, ->, shorten <=1pt, shorten >=1pt ]

\tikzstyle{empty}=[circle,opacity=0.0,text opacity=1.0,minimum width=4pt,minimum height=4pt]
\tikzstyle{box}=[rectangle,fill=White,draw=Black]
\tikzstyle{filled}=[circle,fill=hexcolor0xbfbfbf,draw=Black]
\tikzstyle{hollow}=[circle,fill=White,draw=Black]
\tikzstyle{param}=[rectangle,fill=Black,draw=Black,inner sep=0pt,minimum width=4pt,minimum height=4pt]
\tikzstyle{paramhollow}=[rectangle,fill=White,draw=Black,inner sep=0pt,minimum
width=4pt,minimum height=4pt]

\usetikzlibrary{calc,arrows.meta,positioning,decorations.pathreplacing,shapes, quotes}
\usetikzlibrary{bayesnet}

\newcommand*\patchAmsMathEnvironmentForLineno[1]{%
  \expandafter\let\csname old#1\expandafter\endcsname\csname
    #1\endcsname
  \expandafter\let\csname oldend#1\expandafter\endcsname\csname
    end#1\endcsname
  \renewenvironment{#1}%
  {\linenomath\csname old#1\endcsname}%
  {\csname oldend#1\endcsname\endlinenomath}%
}
\newcommand*\patchBothAmsMathEnvironmentsForLineno[1]{%
  \patchAmsMathEnvironmentForLineno{#1}%
  \patchAmsMathEnvironmentForLineno{#1*}%
}
\patchBothAmsMathEnvironmentsForLineno{equation}
\patchBothAmsMathEnvironmentsForLineno{align}
\patchBothAmsMathEnvironmentsForLineno{flalign}
\patchBothAmsMathEnvironmentsForLineno{alignat}
\patchBothAmsMathEnvironmentsForLineno{gather}
\patchBothAmsMathEnvironmentsForLineno{multline}

\usepackage[margin=1in]{geometry}
\usepackage{tcolorbox}

\allowdisplaybreaks

\usepackage[defaultlines=4,all]{nowidow}

\title{\textbf{The Blessings of Multiple Causes}}

\usepackage{times}

\author{
  Yixin Wang\\
  Department of Statistics\\
  Columbia University\\
  \texttt{yixin.wang@columbia.edu} \\
\\
  David M.~Blei\\
  Department of  Statistics\\
  Department of Computer Science\\
  Columbia University\\
  \texttt{david.blei@columbia.edu} \\
}
\date{\today}

\begin{document}
\maketitle

% !TEX root = latent_confounder.tex
\begin{abstract}
  Causal inference from observational data often assumes
  ``ignorability,'' that all confounders are observed. This assumption
  is standard yet untestable.  However, many scientific studies
  involve multiple causes, different variables whose effects are
  simultaneously of interest.  We propose the deconfounder, an
  algorithm that combines unsupervised machine learning and predictive
  model checking to perform causal inference in multiple-cause
  settings.  The deconfounder infers a latent variable as a substitute
  for unobserved confounders and then uses that substitute to perform
  causal inference.  We develop theory for the deconfounder, and show
  that it requires weaker assumptions than classical causal inference.
  We analyze its performance in three types of studies: semi-simulated
  data around smoking and lung cancer, semi-simulated data around
  genome-wide association studies, and a real dataset about actors and
  movie revenue.  The deconfounder provides a checkable approach to
  estimating closer-to-truth causal effects.
\end{abstract}

%%% Local Variables:
%%% mode: latex
%%% TeX-master: "latent_confounder"
%%% End:


\clearpage

\begin{bibunit}[apa]

% !TEX root = latent_confounder.tex
\section{Introduction}
\label{sec:introduction}

Here is a frivolous, but perhaps lucrative, causal inference problem.
\Cref{tab:toprevmov} contains data about movies. For each movie, the
table shows its cast of actors and how much money the movie made.
Consider a movie producer interested in the causal effect of each
actor; for example, how much does revenue increase (or decrease) if
Oprah Winfrey is in the movie?

The producer wants to solve this problem with the potential outcomes
approach to causality~\citep{Imbens:2015, rubin1974estimating,
rubin2005causal}. Following the methodology, she associates each movie
to a \textit{potential outcome function}, $y_i(\mba)$.  This function
maps each possible cast $\mba$ to its revenue if the movie $i$ had
that cast.  (The cast $\mba$ is a binary vector with one element per
actor; each element encodes whether the actor is in the movie.)  The
potential outcome function encodes, for example, how much money
\textit{Star Wars} would have made if Robert Redford replaced Harrison
Ford as Han Solo.  When doing causal inference, the producer's goal is
to estimate something about the population distribution of
$Y_i(\mba)$.  For example, she might consider a particular cast $\mba$
and estimate the expected revenue of a movie with that cast,
$\E{}{Y_i(\mba)}$.

Classical causal inference from observational data (like
\Cref{tab:toprevmov}) is a difficult enterprise and requires strong
assumptions.  The challenge is that the dataset is limited; it
contains the revenue of each movie, but only at its assigned cast.
However, what this paper is about is that the producer's problem is
not a classical causal inference. While causal inference usually
considers a single possible cause, such as whether a subject receives
a drug or a control, our producer is considering a \textit{multiple
causal inference}, where each actor is a possible cause.  This paper
shows how multiple causal inference can be easier than classical
causal inference.  Thanks to the multiplicity of causes, the producer
can make valid causal inferences under weaker assumptions than the
classical approach requires.

Let's discuss the producer's inference in more detail: how can she
calculate $\E{}{Y_i(\mba)}$?  Naively, she subsets the data in
\Cref{tab:toprevmov} to those with cast equal to $\mba$, and then
computes a Monte Carlo estimate of the revenue.  This procedure is
unbiased when $\E{}{Y_i(\mba)} = \E{}{Y_i(\mba) \g \mbA_i=\mba}$.

But there is a problem. The data in \Cref{tab:toprevmov} hide
\textit{confounders}, variables that affect both the causes and the
effect. For example, every movie has a genre, such as comedy, action, 
or romance. This genre has an effect on both who is in the cast and
the revenue.  (E.g., action movies cast a certain set of actors and
tend to make more money than comedies.) When left unobserved, the
genre of the movie produces a statistical dependence between whether
an actor is in it and its revenue; this dependence biases the causal
estimates, $\E{}{Y_i(\mba) \g \mbA_i=\mba} \neq \E{}{Y_i(\mba)}$.

Thus the main activities of classical causal inference are to
identify, measure, and control for confounders. Suppose the producer
measures confounders for each movie $w_i$. Then inference is simple:
use the data (now with confounders) to take Monte Carlo estimates of
$\E{}{\E{}{Y_i(\mba) \g W_i, \mbA_i=\mba}}$; this iterated expectation
``controls'' for the confounders.  But the problem is that whether the
estimate is equal to $\E{}{Y_i(\mba)}$ rests on a big and uncheckable
assumption: there are no other confounders. For many applied causal
inference problems, this assumption is a leap of faith.

% !TEX root = latent_confounder.tex

\begin{table*}[t]
\scriptsize
  \begin{center}
    \begin{tabular}{llr} 
     \toprule
        Title & Cast & Revenue\\
        \midrule
        \emph{Avatar}& \{Sam Worthington, Zoe Saldana, Sigourney Weaver, Stephen Lang, $\ldots$ \}&\$2788M\\
        \emph{Titanic}&\{Kate Winslet, Leonardo DiCaprio, Frances Fisher, Billy Zane, $\ldots$ \}&\$1845M\\
        \emph{The Avengers}&\{Robert Downey Jr., Chris Evans, Mark Ruffalo, Chris Hemsworth, $\ldots$ \}&\$1520M\\
        \emph{Jurassic World}&\{Chris Pratt, Bryce Dallas Howard, Irrfan Khan, Vincent D'Onofrio, $\ldots$ \}&\$1514M\\
        \emph{Furious 7}&\{Vin Diesel, Paul Walker, Dwayne Johnson, Michelle Rodriguez, $\ldots$ \}&\$1506M\\
      $\quad \quad \vdots$ & $\quad \vdots \quad$ & $\quad \vdots \quad$ \\
      \bottomrule
    \end{tabular}
    \caption{Top earning movies in the TMDB
      dataset \label{tab:toprevmov}}
  \end{center}
\end{table*}

%%% Local Variables:
%%% mode: latex
%%% TeX-master: "latent_confounder"
%%% End:

We develop \textit{the deconfounder}, an alternative method for the
producer who worries about missing a confounder.  First the producer
finds and fits a good latent-variable model to capture the dependence
among actors.  It should be a factor model, one that contains a
per-movie latent variable that renders the assigned cast conditionally
independent.  (Probabilistic principal component
analysis~\citep{tipping1999probabilistic} is a simple example, but
there are many others.)  Given the model, she then estimates the
per-movie variable for each cast in the dataset; this estimated
variable is a substitute for unobserved confounders. Finally, she
controls for the substitute confounder and obtains valid causal
inferences.

The deconfounder capitalizes on the dependency structure of the
observed casts, using patterns of how actors tend to appear together
in movies as indirect evidence for confounders in the data. Thus the
producer replaces an uncheckable search for possible confounders with
the checkable goal of building a good factor model of observed casts.

All methods for causal inference using observational data are based on
assumptions. Here we make two. First, we assume that the fitted
latent-variable model is a good model of the assigned causes. Happily, 
this assumption is testable; we will use predictive checks to assess
how well the fitted model captures the data.  Second, we assume that
there are no unobserved single-cause confounders, variables that
affect one cause (e.g., actor) and the potential outcome function
(e.g., revenue). While this assumption is not testable, it is weaker
than the usual assumption of ignorability, i.e., no unobserved
confounders.

Beyond making movies, many causal inference problems, especially from
observational data, also classify as multiple causal inference.  Such
problems arise in many fields.
\begin{itemize}[leftmargin=*]
\item \textbf{Genome-wide association studies (GWAS).} In GWAS, 
biologists want to know how genes causally connect to traits
\citep{stephens2009bayesian, visscher201710}.  The assigned causes are
alleles on the genome, often encoded as either being common
(``major'') or uncommon (``minor''), and the effect is the trait under
study.  Confounders, such as shared ancestry among the population, 
bias naive estimates of the effect of genes.  We study GWAS problems
in \Cref{subsec:gwasstudy}.

\item \textbf{Computational neuroscience.} Neuroscientists want to
know how specific neurons or brain measurements affect behavior and
thoughts \citep{churchland2012neural}.  The possible causes are
multiple measurements about the brain's activity, e.g., one per
neuron, and the effect is a measured behavior. Confounders, 
particularly through dependencies among neural activity, bias the
estimated connections between brain activity and behavior.

\item \textbf{Social science.} Sociologists and policy-makers want to
know how social programs affect social outcomes, such as poverty
levels and upward mobility~\citep{Morgan:2015}.  However, individuals
may enroll in several such programs, blurring information about their
possible effects.  In social science, controlled experiments are
difficult to engineer; using observational data for causal inference
is typically the only option.

\item \textbf{Medicine.} Doctors want to know how medical treatments
affect the progression of disease. The multiple causes are medications
and procedures; the outcome is a measurement of a disease (e.g., a lab
test).  There are many confounders---such as when and where a patient
is treated or the treatment preferences of the attending doctor---and
these variables bias the estimates of effects.  While gold-standard
data from clinical trials are expensive to obtain, the abundance of
electronic health records could inform medical practices.

\end{itemize}

Causal inference in each of these fields can use the deconfounder.
Fit a good factor model of the assigned causes, infer substitute
confounders, and use the substitutes in causal inference.

\parhead{Related work.} The deconfounder relates to several threads of
research in causal inference.

\textit{Probabilistic modeling for causal inference.}
\citet{stegle2010probabilistic} use Gaussian processes to depict
causal mechanisms; \citet{zhang2009identifiability} study
post-nonlinear causal models and their identifiability;
\citet{mckeigue2010sparse} builds on sparse methods to infer causal
structures; \citet{moghaddass2016factorized} generalize the
self-controlled case series method to multiple causes and multiple
outcomes using factor models. More recently, \citet{louizos2017causal}
use variational autoencoders to infer unobserved confounders,
\citet{shah2018rsvp} develop projection-based techniques for
high-dimensional covariance estimation under latent confounding, and
\citet{kaltenpoth2019we} leverages information theory principles to
differentiate causal and confounded connections.

With a related goal, \citet{tran2017implicit} build implicit causal
models. Like the GWAS example in this paper (\Cref{subsec:gwasstudy}),
they take an explicit causal view of \gls{GWAS}, treating the
\glspl{SNP} as the multiple causes. They connect implicit probabilistic
models and nonparametric structural equation models for causal
inference~\citep{Pearl:2009a}, and develop inference algorithms for
capturing shared confounding.  \citet{heckerman2018accounting} studies
the same scenario with multiple linear regression, where observing
many causes makes it possible to account for shared confounders.
Multiple causal inference and latent confounding was also formalized
by \citet{ranganath2018multiple}, who take an information-theoretic
approach.

Our work complements all of these works. These works rest on Pearl's
causal framework~\citep{Pearl:2009a}; they hypothesize a causal graph
with confounders, causes, and outcomes. We develop the deconfounder in
the context of the potential outcomes framework \citep{Imbens:2015,
rubin1974estimating, rubin2005causal}.

\textit{Analyzing \gls{GWAS}.}  In \gls{GWAS}, latent population
structure is an important unobserved confounder.
\citet{pritchard2000association} propose a probabilistic admixture
model for unsupervised ancestry inference.  \citet{price2006principal}
and \citet{astle2009population} estimate the unobserved population
structure using the principal components of the genotype matrix.
\citet{yu2006unified} and \citet{kang2010variance} estimate the
population structure via the ``kinship matrix'' on the genotypes.
\citet{song2015testing} and \citet{hao2015probabilistic} rely on
factor analysis and admixture models to estimate the population
structure. \citet{GTEx2017} adopt a similar idea to study the effect
of genetic variations on gene expression levels. These methods can be
seen as variants of the deconfounder (see \Cref{subsec:connections}).
The deconfounder gives them a rigorous causal justification, provides
principled ways to compare them, and suggests an array of new
approaches. We study \gls{GWAS} data in \Cref{subsec:gwasstudy}.

\textit{Assessing the ignorability assumption.}
\citet{rosenbaum1983central} demonstrates that ignorability and a good
propensity score model are sufficient to perform causal inference with
observational data. Many subsequent efforts assess the plausibility of
ignorability. For example, \citet{robins2000sensitivity, 
gilbert2003sensitivity, imai2004causal} develop sensitivity analysis
in various contexts, though focusing on data with a single cause.  In
contrast, this work uses predictive model checks to assess
unconfoundedness with multiple causes.  More recently, 
\citet{sharma2016split} leveraged auxillary outcome data to test for
confounding in time series data; \citet{janzing2018detecting, 
janzing2018detecting2, liu2018confounder} developed tests for
non-confounding in multivariate linear regression.  Here we work
without auxiliary data, focus on causal estimation, as opposed to
testing, and move beyond linear models.

\textit{The (generalized) propensity score.}
\citet{schneeweiss2009high, mccaffrey2004propensity, lee2010improving}
and many others develop and evaluate different models for assigned
causes. In particular, \citet{chernozhukov2017double} introduce a
semiparametric assignment model; they propose a principled way of
correcting for the bias that arises when regularizing or overfitting
the assignment model.  This work introduces latent variables into the
model.  The multiplicity of causes enables us to infer these latent
variables and then use them as substitutes for unobserved confounders.

\textit{Classical causal inference with multiple treatments.}
\citet{lopez2017estimation, mccaffrey2013tutorial, zanutto2005using, 
  rassen2011simultaneously, lechner2001identification, 
  feng2012generalized} extend classical matching, subclassification, 
and weighting to multiple treatments, always assuming
ignorability. This work relaxes that assumption.

\parhead{This paper.}  \Cref{sec:deconfounder} reviews classical
causal inference, sets up multiple causal inference, presents the
deconfounder, and describes its identification strategy and
assumptions.  \Cref{sec:study} presents three empirical studies, two
semi-synthetic and one real. \Cref{sec:theory} further develops theory
around the deconfounder and establishes causal identification.
Finally, \Cref{sec:discussion} concludes the paper with a discussion.

%%% Local Variables:
%%% mode: latex
%%% TeX-master: "latent_confounder"
%%% End:

% !TEX root = latent_confounder.tex

\section{Multiple causal inference with the deconfounder}
\label{sec:deconfounder}

\subsection{A classical approach to multiple causal inference}
\label{subsec:classical}

We first describe multiple causal inference. There are $m$
\textit{possible causes}, encoded in a vector $\mba = (a_{1}, \ldots,
a_{m})$.  We can consider a variety of types: real-valued causes,
binary causes, integer causes, and so on.  In the actor example, the
causes are binary: $a_{j}$ encodes whether actor $j$ is in the movie.

For each individual $i$ (movie) there is a \textit{potential outcome
function} that maps configurations of causes to the outcome (revenue).
We focus on real-valued outcomes.  For the $i$th movie, the potential
outcome function maps each possible cast to the log of its revenue,
$y_i(\mba) : \{0,1\}^m \rightarrow \bbR$.  Note $y_i(\mba)$ is a
function.  It maps every possible cast of actors to the movie's
revenue for that cast.

The goal of causal inference is to characterize the sampling
distribution of the potential outcomes $Y_i(\mba)$ for each
configuration of the causes $\mba$. This distribution provides causal
inferences, such as the expected outcome for a particular array of
causes (a particular cast of actors) $\mu(\mba)=\E{}{Y_i(\mba)}$ or
the average effect of individual causes (how much a particular actor
contributes to revenue).

To help make causal inferences, we draw data from the sampling
distribution of assigned causes $\mba_i$ (the cast of movie $i$) and
realized outcomes $y_i(\mba_i)$ (its revenue).\footnote{We use the
term \textit{assigned causes} for the vector of what some might call
the ``assigned treatments.''  Because some variables may not exhibit a
causal effect, a more precise term would be ``assigned potential
causes'' (but it is too cumbersome).}  The data is $\cD = \{(\mba_i,
y_i(\mba_i)\} \,\,\, i = 1, \ldots, n$. Note we only observe the
outcome for the assigned causes $y_i(\mba_i)$, which is just one of
the values of the potential outcome function.  But we want to use such
data to characterize the full distribution of $Y_i(\mba)$ for any
$\mba$; this is the ``fundamental problem of causal
inference''~\citep{Holland:1986}.

To estimate $\mu(\mba)$, consider using the data to calculate
conditional Monte Carlo approximations of $\E{}{Y_i(\mba) \g
\mbA_i=\mba}$.  These estimates are simply averages of the outcomes
for each configuration of the causes. But this approach may not be
accurate.  There might be \textit{unobserved confounders}---hidden
variables that affect both the assigned causes $\mbA_i$ and the
potential outcome function $Y_i(\mba)$. When there are unobserved
confounders, the assigned causes are correlated with the observed
outcome. Consequently, Monte Carlo estimates of $\mu(\mba)$ are
biased,
\begin{align}
\label{eq:biased-estimate}
\E{}{Y_i(\mba) \g \mbA_i = \mba} \neq \E{}{Y_i(\mba)}.
\end{align} We can estimate $\E{}{Y_i(\mba) \g \mbA_i = \mba}$ with
the dataset; but the goal is to estimate $\E{}{Y_i(\mba)}$.
\footnote{Here is the notation.  Capital letters denote a random
variable. For example, the random variable $\mbA_i$ is a randomly
chosen vector of assigned causes from the population.  The random
variable $Y_i(\mbA_i)$ is a a randomly chosen potential outcome from
the population, evaluated at its assigned causes. A lowercase letter
is a realization.  For example, $\mba_i$ is in the dataset---it is the
vector of assigned causes of individual $i$.  The left side of
\Cref{eq:biased-estimate} is an expectation with respect to the random
variables; it conditions on the random vector of assigned causes to be
equal to a certain realization $\mbA_i = \mba$.  The right side is an
expectation over the same population of the potential outcome
functions, but always evaluated at the realization $\mba$.}

Suppose we measure covariates $x_i$ and append to each data point,
$\cD =
\{(\mba_{i}, x_i, y_i(\mba_i)\} \,\,\, i = 1, \ldots, n$.  If these
covariates contain all confounders then
\begin{align}
\label{eq:condition-on-confounders}
\E{}{\E{}{Y_i(\mba) \g X_i, \mbA_i=\mba}} = \E{}{Y_i(\mba)}.
\end{align} 
Using the augmented dataset, we can estimate the left side with Monte
Carlo; thus we can estimate $\E{}{Y_i(\mba)}$.

\Cref{eq:condition-on-confounders} is true when $X$ capture all
confounders. More precisely, it is true under the assumption of
\textit{(weak) ignorability\footnote{Here we describe the weak version
of the ignorability assumption, which requires individual potential
outcomes $Y_i(\mba)$ be marginally independent of the causes $\mbA_i$,
i.e.  $\mbA_i \indpt Y_i(\mba) \g X_i$ for all $\mba$.
\citet{imbens2000role} and \citet{hirano2004propensity} call this
assumption \emph{weak unconfoundedness}. In contrast, strong
ignorability says $\mbA_i \indpt (Y_i(\mba))_{\mba\in\mathcal{A}} \g
X_i$, which requires all possible potential outcomes
$(Y_i(\mba))_{\mba\in\mathcal{A}}$ be jointly independent of the
causes $\mbA_i$.}}~\citep{rosenbaum1983central,imai2004causal}:
conditional on observed $X$, the assigned causes are independent of
the potential outcomes,
\begin{align}
\label{eq:strong-ignorability}
\mbA_i \indpt Y_i(\mba) \g X_i \qquad \forall \mba.
\end{align}
The nuance is that \Cref{eq:strong-ignorability} needs to hold for all
possible $\mba$'s, not only for the value of $Y_i(\mba)$ at the
assigned causes. Ignorability implies no unobserved
confounders.\footnote{We also assume \textit{stable unit treatment
value assumption (SUTVA)} \citep{rubin1980randomization,
rubin1990comment} and \emph{overlap} \citep{imai2004causal}, roughly
that any vector of assigned causes has positive probability. These
three assumptions together identify the potential outcome function
\citep{imbens2000role, hirano2004propensity, imai2004causal}. }

\Cref{eq:condition-on-confounders} underlies the practice of causal
inference: find and measure the confounders, estimate conditional
expectations, and average.  In the introduction, for example, we
pointed out that the genre of the movie is a confounder to causal
inference of movie revenues.  The genre affects both which cast is
selected and the potential earnings of the film.  But the assumption
that there are no unobserved confounders is significant. One of the
central challenges around causal inference from observational data is
that ignorability is untestable---it fundamentally depends on the
entire potential outcome function, of which we only observe one value
\citep{Holland:1986}.

\subsection{The deconfounder: Multiple causal inference without
ignorability}
\label{subsec:deconfounder}

We now develop the \textit{deconfounder}, an algorithm that exploits
the multiplicity of causes to sidestep the search for confounders.
There are three steps. First, find a good latent variable model of the
assignment mechanism $p(z, a_1, \ldots, a_m)$, where $z$ is a local
factor. Second, use the model to infer the latent variable for each
individual $p(z_i \g a_{i1}, \ldots, a_{im})$. Finally, use the
inferred variable as a substitute for unobserved confounders and form
causal inferences.  The deconfounder replaces an uncheckable search
for possible confounders with the checkable goal of building a good
model of assigned causes.

We first explain the method in more detail.  Then we explain why and
when it provides unbiased causal inferences.

In the first step of the deconfounder, define and fit a
\textit{probabilistic factor model} to capture the joint distribution
of causes $p(a_1, \ldots, a_m)$. A factor model posits per-individual
latent variables $Z_i$, which we call local factors, and uses them to
model the assigned causes. The model is
\begin{align}
 \begin{split}
   Z_i &\sim p(\cdot \g \alpha) \quad i = 1, \ldots, n, \\
   A_{ij} \g Z_i&\sim p(\cdot \g z_i, \theta_j) \quad j = 1, \ldots, m,
 \end{split}
 \label{eq:factor-model}
\end{align}
where $\alpha$ parameterizes the distribution of $Z_i$ and $\theta_j$
parameterizes the per-cause distribution of $A_{ij}$.  Notice that
$Z_i$ can be multi-dimensional.  Factor models encompass many methods
from Bayesian statistics and probabilistic machine learning. Examples
include matrix factorization \citep{tipping1999probabilistic}, mixture
models~\citep{mclachlan1988mixture}, mixed-membership
models~\citep{pritchard2000association, blei2003latent,
airoldi2008mixed,erosheva2003bayesian}, and deep generative
models~\citep{neal1990learning, ranganath2015deep,
ranganath2016hierarchical, tran2017deep, rezende2015variational,
mohamed2016learning, kingma2013auto}. One can fit using any
appropriate method, such as maximum likelihood estimation or Bayesian
inference.  And exact fitting is not required; one can use approximate
methods like the EM algorithm, Markov chain Monte Carlo, and
variational inference.  What the deconfounder requires is that the
fitted factor model provides an accurate approximation of the
population distribution of $p(\mbA)$.

In the next step, use the fitted factor model to calculate the
conditional expectation of each individual's local factor weights
$\hat{z}_i = \E{M}{Z_i \g \mbA_i = \mba_i}$.  We emphasize that this
expectation is from the fitted model $M$ (not the population
distribution). Again, one can use approximate expectations.

In the final step, condition on $\hat{z}_i$ as a substitute confounder
and proceed with causal inference.  For example, we can estimate
$\E{}{\E{}{Y_i(\mba) \g \hat{Z}_i, \mbA_i = \mba}}$.  The main idea is
this: if the factor model captures the distribution of assigned
causes---a testable proposition---then we can safely use $\hat{z}_i$
as a variable that contains the confounders.

Why is this strategy sensible? Assume the fitted factor model captures
the (unconditional) distribution of assigned causes $p(a_{i1}, \ldots,
a_{im})$. This means that all causes are conditionally independent
given the local latent factors,
\begin{align}
\label{eq:cond-ind}
p(a_{i1}, \ldots, a_{im}\g z_i) =\prod_{j=1}^m p(a_{ij}\g
z_i).
\end{align}
Now make an additional assumption: there are no \textit{single-cause
confounders}, a variable that affects just one of the assigned causes
and on the potential outcome function.  (More precisely, we need to
have observed all the single-cause confounders.) With this assumption,
the independence statement of \Cref{eq:cond-ind} implies ignorability,
$\mbA_i \indpt Y_i(\mba) \g Z_i.$ Ignorability justifies causal
inference.

\begin{figure}
\begin{center}
  \includegraphics[scale=0.5]{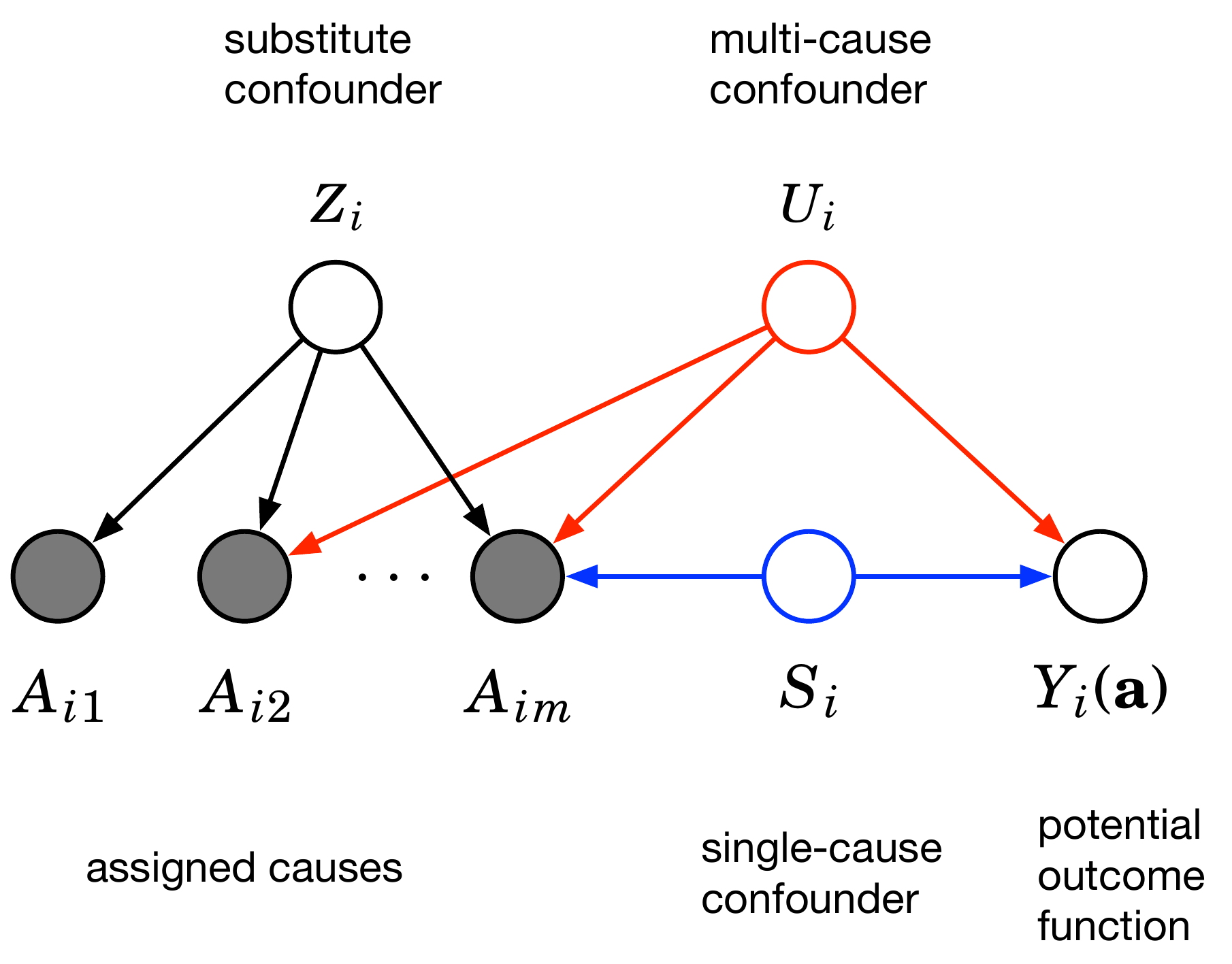}
\end{center}
\caption{A graphical model argument for the deconfounder.  The
  punchline is that if $Z_i$ renders the $A_{ij}$'s conditionally
  independent then there cannot be a multi-cause confounder.  The
  proof is by contradiction.  Assume conditional independence holds,
  $p(a_{i1}, \ldots, a_{im}\g z_i) = \prod_j p(a_{ij}\g z_i)$; if there
  exists a multi-cause confounder $U_i$ (red) then, by $d$-separation,
  conditional independence cannot hold \citep{pearl1988probabilistic}.
  Note we cannot rule out the single-cause confounder $S_i$ (blue).}
\label{fig:graphical-arg}
\end{figure}

The graphical model in \Cref{fig:graphical-arg} justifies the
deconfounder and reveals its
assumptions.\footnote{\Cref{fig:graphical-arg} uses a graphical model
to represent and reason about conditional dependencies in the
population distribution.  It is not a causal graphical model or a
structural equation model.}  Suppose we observe a $Z_i$ such that the
conditional independence in \Cref{eq:cond-ind} holds.  Further suppose
there exists an unobserved multi-cause confounder $U_i$ (illustrated
in red), which connects to multiple assigned causes and the outcome.
If such a $U_i$ exists then the causes would be dependent, even
conditional on $Z_i$.  (This fact comes from $d$-separation.)  But
such dependence leads to a contradiction, that \Cref{eq:cond-ind} does
not hold. Thus $U_i$ cannot exist.

There is a nuance.  The conditional independence in \Cref{eq:cond-ind}
cannot rule out the existence of an unobserved single-cause
confounder, denoted $S_i$ in \Cref{fig:graphical-arg}.  Even if such a
confounder exists, the conditional independence still holds.

Here is the punchline.  If we find a factor model that captures the
population distribution of assigned causes then we have essentially
discovered a variable that captures all multiple-cause confounders.
The reason is that multiple-cause confounders induce dependence among
the assigned causes, regardless of how they connect to the potential
outcome function. Modeling their dependence, for which we have
observations, provides a way to estimate variables that capture those
confounders. This is the blessing of multiple causes.

\subsection{The identification strategy of the deconfounder}

\label{subsec:identificationstrat}

How does the deconfounder identify potential outcomes? The classical
strategy for causal identification is that ignorability, together with
\gls{SUTVA}~and overlap, identifies the potential
outcomes~\citep{imbens2000role, hirano2004propensity, imai2004causal}.
The deconfounder continues to assume \gls{SUTVA} and overlap, but it
weakens the ignorability assumption.

Roughly, ignorability requires that there are no unobserved
confounders. To weaken this assumption, the deconfounder constructs a
substitute confounder that captures all multiple-cause confounders.
(The proof is in \Cref{sec:theory}.)  Uncovering multi-cause
confounders from data weakens the ignorability assumption to one of no
unobserved \textit{single-cause} confounders.

Thus the deconfounder relies on three assumptions: (1) \gls{SUTVA}
\citep{rubin1980randomization, rubin1990comment}; (2) no unobserved
single-cause confounders; (3) overlap \citep{imai2004causal}.

\glsreset{SUTVA} \parhead{\Gls{SUTVA}.} The \gls{SUTVA} requires that
the potential outcomes of one individual are independent of the
assigned causes of another individual. It assumes that there is no
interference between individuals and there is only a single version of
each assigned cause. See \citet{rubin1980randomization,
rubin1990comment} and \citet{Imbens:2015} for discussion.

\parhead{No unobserved single-cause confounders.} Denote $X_i$ as the
observed covariates. (Observed covariates are not necessarily
confounders.) ``No unobserved single-cause confounders'' requires
\begin{align}
A_{ij}\independent Y_i(\mba) \g X_i, \qquad j= 1, \ldots, m.
\label{eq:singleignore}
\end{align}
We call this assumption ``single ignorability.'' Single ignorability
differs from classical ignorability by only requiring marginal
independence between individual causes $A_{ij}$ and the potential
outcome $Y_i(\mba)$.  In contrast, classical ignorability requires
$(A_{i1}, \ldots, A_{im})\independent Y_i(\mba) \g X_i$, i.e., the
joint independence between the causes $(A_{i1}, \ldots, A_{im})$ and
the potential outcome function $Y_i(\mba)$.

Roughly, single ignorability implies that we observe any confounders
that affect only one of the causes; see \Cref{fig:graphical-arg}. This
assumption is weaker than classical ignorability; we no longer need to
observe all confounders. That said, whether the assumption is
plausible depends on the particulars of the problem. Note that single
ignorability reduces to the classical ignorability assumption when
there is only one cause; both requires $\mbA_{i}\independent Y_i(\mba)
\g X_i$, where $\mbA_{i}$ and $\mba$ are one-dimensional.

When might single ignorability be plausible?  Consider the movie-actor
example. One possible confounder is the reputation of the director.
Famous directors have access to a circle of capable actors; they also
tend to make good movies with large revenues. If the dataset contains
many actors, it is likely that several are in the circle of capable
actors; the director's reputation is a multi-cause confounder. (If
only one actor in the dataset is capable then the director's
reputation is a single-cause confounder.)

Or consider the \gls{GWAS} problem. If a confounder affects SNPs---and
we observe 100,000 SNPs per individual---then the confounder may be
unlikely to have an effect on only one.  The same reasoning can apply
to other settings---medications in medical informatics data, neurons
in neuroscience recordings, and vocabulary terms in text data.

By the same token, single ignorability may not be satisfied when there
are very few assigned causes. Consider the neuroscience problem of
inferring the relationship between brain activity and animal behavior,
but where the scientist only records the activity of a small number of
neurons.  While unlikely that a confounder affects only one neuron in
the brain, it may be more possible that a confounder affects only one
of the observed neurons.

In domains where single ignorability is likely not satisfied, we
suggest performing sensitivity analysis \citep{robins2000sensitivity,
gilbert2003sensitivity, imai2004causal} on the deconfounder estimates.
It assesses the robustness of the estimate against unobserved
single-cause confounding. In the context of \gls{GWAS},
\Cref{subsec:gwasstudy} will illustrate the effect of violating single
ignorability.

\parhead{Overlap. } The final assumption of the deconfounder is that
the substitute confounder $Z_i$ satisfies the overlap
condition\footnote{We also require the observed covariates $X_i$
satisfy the overlap condition if they are single-cause confounders,
i.e. $p(A_{ij}\in \mathcal{A} \g X_i) > 0 \text{ for all sets
}\mathcal{A} \text{ with positive measure, i.e. } p(\mathcal{A}) >
0$.}
\begin{align}
p(A_{ij}\in \mathcal{A} \g Z_i) > 0
\text{ for all sets }\mathcal{A} \text{ with positive measure,
i.e. } p(\mathcal{A}) > 0.
\label{eq:overlapmain}
\end{align}
Overlap asserts that, given the substitute confounder, the conditional
probability of any vector of assigned causes is positive.  This
assumption is sometimes stated as the second half of ignorability
\citep{imai2004causal}.

The potential outcome $Y_i(\mba)$ is not identifiable if the
substitute confounder does not satisfy overlap. When the overlap is
limited, i.e. $p(A_{ij}\in\mathcal{A}\g Z_i)$ is small for all values of
$Z_i$, then the deconfounder estimates of the potential outcome
$Y_i(\mba)$ will have high variance.

For many probabilistic factor models, the overlap condition is
satisfied. For example, probabilistic PCA assumes $A_{ij}\g Z_i\sim
\cN(Z_i^\top\theta_j, \sigma^2)$. The normal distribution has support
over the real line, which ensures $P(A_{ij}\in\mathcal{A}\g Z_i) > 0$ for
all $\mathcal{A}$ with positive measure. That said, as the
dimensionality of $Z_i$ increases, overlap often becomes increasingly
limited \citep{d2017overlap}. For example, probabilistic PCA returns
increasingly small $\sigma^2$, which signals $P(A_{ij}\in\mathcal{A}\g
Z_i)$ is small.

We can enforce overlap by constraining the allowable family of factor
models.  With continuous causes, we restrict to models with continuous
densities on $\mathbb{R}$. (We assume the causes are full-rank, i.e.,
that no two causes are measurable with each other; if such a pair
exists, merge them into a single cause.) With discrete causes, we
restrict to factor models with support on the whole $\mathcal{A}$ and
a $Z_i$ lower-dimensional than the causes.

Alternatively, we can merge highly correlated causes as a
preprocessing step. For example, consider two causes---paracetamol and
ibuprofen--that are always assigned the same value. We can merge them
into one cause: we only estimate the potential outcome of either
taking both drugs or taking neither. This merging step prevents the
deconfounder from extrapolating for the assigned causes which the data
carries little evidence. We can also resort to classical strategies of
causal inference under limited overlap, for example subsampling the
population \citep{crump2009dealing}.

How can we assess the overlap with respect to the substitute
confounder? With a fitted factor model, we can analyze the conditional
distribution of the assigned causes given the substitute confounder
$P(A_{ij} \g Z_i)$ for all individual $i$'s. A conditional with low
variance or low entropy signals limited overlap and the possibility of
high-variance causal estimates.

We have described the main assumptions of the deconfounder.  With
\gls{SUTVA}, overlap, and single ignorability, the deconfounder
estimate is unbiased.

\textbf{The deconfounder (informal version of
\Cref{thm:deconfounderfactor}).} \emph{Assume \gls{SUTVA}, single
ignorability (\Cref{eq:singleignore}), and overlap
(\Cref{eq:overlapmain}).  Then the deconfounder provides an unbiased
estimate of the average causal effect:
\begin{align}
&\E{Y}{Y_i(a_1,
\ldots, a_m)} - \E{Y}{Y_i(a'_1,
\ldots, a'_m)}\\
=&\E{X, Z}{\E{Y}{Y_i\g A_{i1}=a_1,\ldots, A_{im}=a_m, X_i, Z_i}} \nonumber\\
&- \E{X, Z}{\E{Y}{Y_i\g A_{i1}=a'_1,\ldots, A_{im}=a'_m, X_i, Z_i}},
  \label{eq:unbiasdcf}
\end{align}}
where $Z_i$ denotes the substitute confounder constructed from the
factor model.

The theorem relies on two properties of the substitute confounder: (1)
it captures all multi-cause confounders; (2) it does not capture
mediators. By its construction from probabilistic factor models, the
substitute confounder captures all multi-cause confounders; again, see
the graphical model argument in \Cref{fig:graphical-arg}. Moreover,
the substitute confounder is constructed with only the observed
causes; no outcome information is used and so it cannot pick up any
mediators. Thus, along with single ignorability and overlap, the
substitute confounder provides full ignorability.  With ignorability
in hand, treat the substitute confounder as if it were observed
covariates and \Cref{eq:unbiasdcf} follows from a classical
conditional independence argument \citep{rosenbaum1983central}.
\Cref{sec:theory} discusses and proves this theorem
(\Cref{thm:deconfounderfactor}).

\subsection{Practical details of the deconfounder}

\sloppy %prevent equation overflow into margin
We next attend to some of the practical details of the deconfounder.
The ingredients of the deconfounder are (1) a factor model of assigned
causes, (2) a way to check that the factor model captures their
population distribution, and (3) a way to estimate the conditional
expectation~$\E{}{Y_i(\mba) \g \hat{Z}_i, \mbA_i = \mba}$~for
performing causal inference.  We discuss each ingredient below
(\Cref{subsec:assignment-model} and \Cref{subsec:outcome-model}) and
then describe the full deconfounder algorithm
(\Cref{subsec:fullalgo}).  We connect the deconfounder to existing
methods in the research literature (\Cref{subsec:connections}) and
answer questions that may come up for the reader (\Cref{subsec:faq}).

\subsubsection{Using the assignment model to infer a substitute
confounder}
\label{subsec:assignment-model}

The first ingredient is a factor model of the assigned causes, as
defined in \Cref{eq:factor-model}, which we call the assignment model.
Many models fall into this category, such as mixture models,
mixed-membership models, and deep generative models.  Each of these
models can be written as \Cref{eq:factor-model}; they each involve a
per-datapoint latent variable $Z_i$ and a per-cause parameter
$\theta_j$. Fitting the factor model gives an estimate of the
parameters $\theta_j, j=1,\ldots, m$. When the fitted factor model
captures the population distribution of the assigned causes then
inferences about $Z_i$ can be used as substitute confounders in a
downstream causal inference.

\parhead{Example factor models.} The deconfounder requires that the
investigator find an adequate factor model of the assigned causes and
then use the factor model to estimate the posterior $p(z_i \g
\mba_i)$.  In the simulations and studies of \Cref{sec:study}, we will
explore several classes of factor models; we describe some of them
here.

One of the most common factor models is \gls{PCA}.  \gls{PCA} is
appropriate when the assigned causes are real-valued.  In its
probabilistic form~\citep{tipping1999probabilistic}, both $z_i$ and
the per-cause parameters $\theta_{j}$ are real-valued $K$-vectors. The
model is
\begin{align}
\begin{split}
  Z_{ik} &\sim \cN(0, \lambda^2), \quad k = 1, \ldots, K,\\
  A_{ij} \g Z_i &\sim \cN\left(z_i^\top \theta_{j}, \sigma^2\right),
  \quad j = 1, \ldots, m.
  \label{eq:prob-PCA}
\end{split}
\end{align}
We can fit probabilistic \gls{PCA} with maximum likelihood (or
Bayesian methods) and use standard conditional probability to
calculate $p(z_i \g \mba_i)$.  Exponential family extensions of
\gls{PCA} are also factor models~\citep{collins2002generalization,
mohamed2009bayesian} as are some deep generative
models~\citep{tran2017deep}, which can be interpreted as a nonlinear
probabilistic PCA.

When the assigned causes are counts then \gls{PF} is an appropriate
factor model \citep{schmidt2009bayesian, cemgil2009bayesian,
gopalan2015scalable}.  \gls{PF} is a probabilistic form of nonnegative
matrix factorization~\citep{lee1999learning, lee2001algorithms}, where
$z_i$ and $\theta_j$ are positive $K$-vectors.  The model is
\begin{align}
\begin{split}
  \label{eq:poisson-factorization}
  Z_{ik} &\sim \textrm{Gamma}(\alpha_0, \alpha_1), \quad k = 1, \ldots, K,\\
  A_{ij} \g Z_i &\sim \textrm{Poisson}(z_i^\top \theta_j),
  \quad j = 1, \ldots, m.
\end{split}
\end{align}
\gls{PF} can be fit to large datasets with efficient variational
methods~\citep{gopalan2015scalable}.  In general, the deconfounder can
use variational methods, or other forms of approximate inference, to
estimate~$p(z_i \g \mba_i)$.

A final example of a factor model is the
\gls{DEF}~\citep{ranganath2015deep}. A \gls{DEF} is a probabilistic
deep neural network.  It uses exponential families to generalize
classical models like the sigmoid belief network
\citep{neal1990learning} and deep Gaussian models
\citep{rezende2014stochastic}. For example, a two-layer \gls{DEF}
models each observation as
\begin{align}
\begin{split}
  \label{eq:deep-exponential-family}
  Z_{2,il} &\sim \textrm{Exp-Fam}_2(\alpha), \quad l = 1, \ldots, L,\\
  Z_{1,ik} \g Z_{2,i} &\sim \textrm{Exp-Fam}_1(g_1(z_{2,i}^\top \theta_{1,k})), \quad k = 1, \ldots, K,\\
  A_{ij} \g Z_{1,i} &\sim \textrm{Exp-Fam}_0(g_0(z_{1,i}^\top \theta_{0,j} )),
  \quad j = 1, \ldots, m.
\end{split}
\end{align}
Here $\textrm{Exp-Fam}$ is an exponential family distribution,
$\theta_*$ are parameters, and $g_*(\cdot)$ are link functions. Each
layer of the \gls{DEF} is a generalized linear model
\citep{mccullagh2018generalized, mccullagh1989generalized}.  The
\gls{DEF} inherits the flexibility of deep neural networks, but uses
exponential families to capture different types of layered
representations and data. For example, if the assigned causes are
counts then $\textrm{Expfam}_0$ can be Poisson; if they are reals then
it can be Gaussian. Approximate inference in \gls{DEF} can be
performed with black box variational
methods~\citep{ranganath2014black}.

\parhead{Predictive checks for the assignment model.}  The
deconfounder requires that its factor model captures the population
distribution of the assigned causes.  To assess the fidelity of the
chosen model, we use predictive checks.  A predictive check compares
the observed assignments with the assignments that would have been
observed under the model.

\newcommand{\heldout}{\textrm{held}}
\newcommand{\rep}{\textrm{rep}}
\newcommand{\obs}{\textrm{obs}}

\glsreset{GWAS}

First hold out a subset of assigned causes for each individual
$a_{i\ell}$, where $\ell$ indexes some held-out causes.  The heldout
assignments are written $\mba_{i,\heldout}$ and note we hold out
randomly selected causes for each individual.  The observed
assignments are written $\mba_{i,\obs}$.

Next fit the factor model to the remaining assignment data $\cD =
\{\mba_{i,\obs}\}_{i=1}^n$. This results in a fitted assignment model
$p(z, \theta \g \mba)$.  For each individual $i$, calculate the local
posterior distribution of $p(z_i \g \mba_{i,\obs})$.

Here is the predictive check.  First sample values for the held-out
causes from their predictive distribution,
\begin{align}
\label{eq:predictive}
p(\mba_{i,\heldout}^{\rep} \g \mba_{i,\obs}) =
\int p(\mba_{i, \heldout} \g z_i) p(z_i \g \mba_{i, \obs}) \dif z_i.
\end{align}
This distribution integrates out the local posterior $p(z_i \g
\mba_{i,\obs})$.  (An approximate posterior also suffices; we discuss
why in \Cref{subsubsec:correct-model}.)

Then compare the replicated data to the held-out data. To compare,
calculate the expected log probability
\begin{align}
t(\mba_{i,\heldout}) = \E{Z}{\log p(\mba_{i,\heldout} \g Z) \g
\mba_{i,\obs}},
\end{align}
which relates to their marginal log likelihood.  In the nomenclature
of posterior predictive checks, this is the ``discrepancy function''
that we use; one can use others.

Finally calculate the predictive score,
\begin{align}
\text{predictive score} = p\left(t(\mba_{i,\heldout}^{\rep}) <
t(\mba_{i, \heldout})\right).
\end{align}
Here the randomness stems from $\mba_{i,\heldout}^{\rep}$ coming from
the predictive distribution in \Cref{eq:predictive}, and we
approximate the predictive score with Monte Carlo.

How to interpret the predictive score?  A good model will produce
values of the held-out causes that give similar log likelihoods to
their real values---the predictive score will not be extreme. A
mismatched model will produce an extremely small predictive score,
often where the replicated data has much higher log likelihood than
the real data. An ideal predictive score is around 0.5. We consider
predictive scores with predictive scores larger than 0.1 to be
satisfactory; we do not have enough evidence to conclude significant
mismatch of the assignment model. Note that the threshold of 0.1 is a
subjective design choice. We find such assignment models that pass
this threshold often yield satisfactory causal estimates in practice.
\Cref{fig:ppc} illustrates a predictive check of a good assignment
model. \Cref{sec:study} shows predictive checks in action.

\begin{figure}[t]
\begin{center}
  \includegraphics[width=0.3\textwidth]{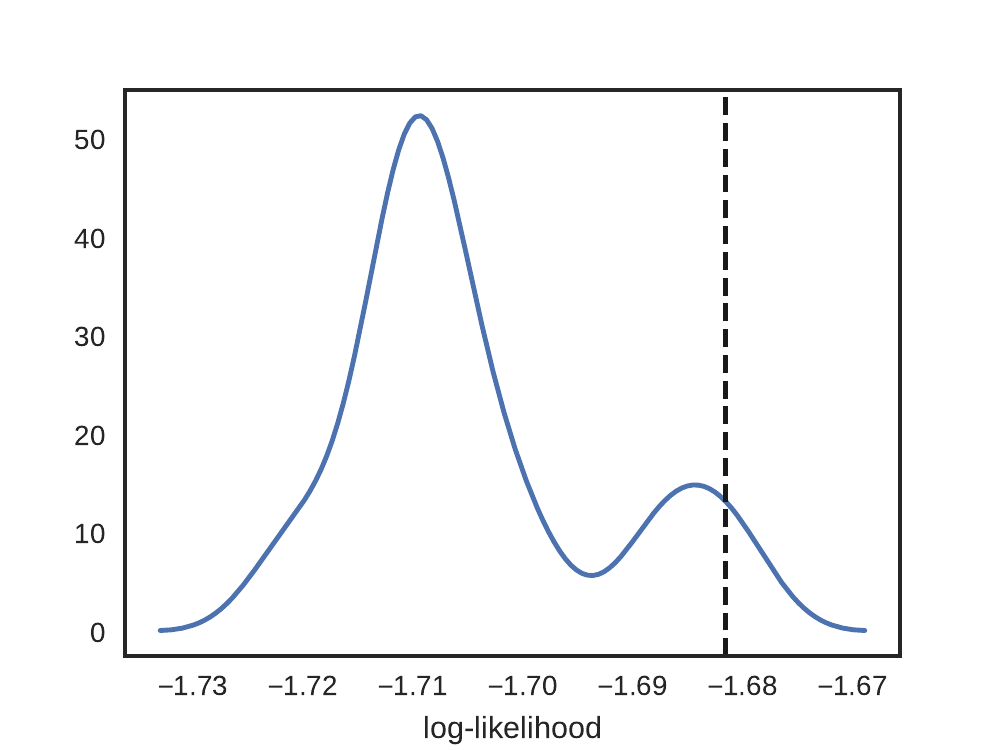}
\end{center}
\caption{Predictive checks for the assignment model. The vertical
dashed line shows $t(\mba_{i, \heldout})$. The blue curve shows the
\gls{KDE} of $t(\mba_{i,\heldout}^{\rep})$. The predictive score is
the area under the blue curve to the left of the vertical dashed line.
The predictive score of this assignment model is larger than 0.1; we
consider it satisfactory.}
\label{fig:ppc}
\end{figure}

Predictive checks blend a circle of related ideas around
\glspl{PPC}~\citep{Rubin:1984}, \glspl{PPC} with realized
discrepancies~\citep{Gelman:1996}, \glspl{PPC} with held-out
data~\citep{Gelfand:1992}, and stage-wise checking of hierarchical
models~\citep{Dey:1998, Bayarri:2007}. They also relate to Bayesian
causal model criticism \citep{tran2016model} and \glspl{PPC} in
\gls{GWAS} \citep{mimno2015posterior}. Finding, fitting, and checking
the factor model also relates to the Box's loop in Bayesian data
analysis \citep{blei2014build, gelman2013bayesian}.

\subsubsection{The outcome model}
\label{subsec:outcome-model}

We described how to fit and check a factor model of multiple assigned
causes.  We now discuss how to fold in the observed outcomes and to
use the fitted factor model to correct for unobserved confounders.

Suppose $p(z_i \g \mba_i, \cD)$ concentrates around a point
$\hat{z}_i$.  Then we can use $\hat{z}_i$ as a confounder.  Follow
\Cref{subsec:classical} to calculate the iterated expectation on the
left side of \Cref{eq:condition-on-confounders}.  However, replace the
observed confounders with the substitute confounder; the goal is to
calculate $\E{}{\E{}{Y_i(\mba) \g \mbA_i = \mba, Z_i}}$.  First,
approximate the outside expectation with Monte Carlo,
\begin{align}
\label{eq:causal-inference}
\E{}{\E{}{Y_i(\mba) \g \mbA_i = \mba, Z_i}} \approx
\frac{1}{n} \sum_{i=1}^{n} \E{Y}{Y_i(\mbA_i) \g \mbA_i = \mba, Z_i =
\hat{z}_i}.
\end{align}
This approximation uses the substitute confounder $\hat{z}_i$,
integrating over its population distribution. It uses the model to
infer the substitute confounder from each data point and then
integrates the distribution of that inferred variable induced by the
population distribution of data.

Turn now to the inner expectation of \Cref{eq:causal-inference}.  We
fit a function to estimate this quantity,
\begin{align}
\label{eq:outcome-model}
\E{}{Y_i(\mbA_i) \g \mbA_i = \mba, Z_i = z} = f(\mba, z).
\end{align}
The function $f(\mba, z)$ is called the \textit{outcome model} and can
be fit from the augmented observed data $\{\mba_i, \hat{z}_i,
y_i(\mba_i)\}$.  For example, we can minimize their discrepancy via
some loss function $\ell$:
\[\hat{f} = \argmin_{f} \sum^n_{i=1}\ell(y_i(\mba_i) - f(\mba_i,
\hat{z}_i)).\] Like the factor model, we can check the outcome
model---it is fit to observed data and should be predictive of
held-out observed data \citep{tran2016model}.

One outcome model we consider is a simple linear function,
\begin{align}
\label{eq:condition-on-z-out}
f(\mba, z) = \beta^\top \mba + \gamma^\top z.
\end{align}
Another outcome model we consider is where $f(\cdot)$ is linear in the
assigned causes $\mba$ and the ``reconstructed assigned causes''
$\hat{\mba}(z) = \E{M}{\mbA \g z}$, an expectation from the fitted
factor model.  This class of functions is
\begin{align}
\label{eq:condition-on-a-hat}
f(\mba, z) = \beta^\top \mba + \gamma^\top \hat{\mba}(z).
\end{align}
This outcome model relates to the generalized propensity score
\citep{imbens2000role, hirano2004propensity}.
\Cref{eq:condition-on-a-hat} can be seen as using $\hat{a}(z)$ as a
proxy for the propensity score, a substitution that is used in
Bayesian statistics \citep{laird1982approximate, tierney1986accurate,
geisser1990validity}; this substitution is justified when higher
moments of the assignment are similar across individuals. In both
models, the coefficient $\beta$ represents the average causal effect
of raising each cause by one individual.

But we are not restricted to linear models. Other outcome models like
random forests \citep{wager2017estimation} and Bayesian additive
regression trees \citep{hill2011bayesian} all apply here.

Note that devising an outcome model is just one approach to
approximating the inner expectation of \Cref{eq:causal-inference}.
Another approach is again to use Monte Carlo.  There are several
possibilities.  In one, group the confounder $\hat{z}_i$ into bins and
approximate the expectation within each bin.  In another, bin by the
propensity score $p(a_i \g \hat{z}_i)$ and approximate the inner
expectation within each propensity-score bin
\citep{rosenbaum1983central, lunceford2004stratification}.  A third
possibility---if the assigned causes are discrete and the number of
causes is small---is to use the propensity score with inverse
propensity weighting \citep{horvitz1952generalization,
rosenbaum1983central, heckman1998characterizing,
dehejia2002propensity}.

\subsubsection{The full algorithm, and an example}

\label{subsec:fullalgo}

We described each component of the deconfounder.
\Cref{alg:deconfounder} gives the full algorithm, a procedure for
estimating \Cref{eq:causal-inference}.  The steps are: (1) find, fit,
and check a factor model to the dataset of assigned causes; (2)
estimate $\hat{z}_i$ for each datapoint; (3) find and fit a outcome
model; (4) use the outcome model and estimated $\hat{z}_i$ to do
causal inference.

\begin{algorithm}[t]
\setstretch{1.10}
\DontPrintSemicolon
\;
\KwIn{a dataset of assigned causes and outcomes $\{(\mba_i, y_i)\}
  ,\,\, i=1, \ldots, n$}
\KwOut{the average potential outcome $\E{}{Y(\mba)}$ for any causes $\mba$}
% \Comment*[l]{DESIGN STAGE}
\Repeat{the assignment check is satisfactory}{
  choose an assignment model from the class in
  \Cref{eq:factor-model}\;
  fit the model to the assigned causes $\{\mba_i\}, \,\, i=1,
  \ldots, n$\;
  check the fitted model $\hat{M}$\;
}

\ForEach{datapoint $i$}{
  calculate $\hat{z}_i = \E{\hat{M}}{Z_i \g \mba_i}$. \;}
% \Comment*[l]{ANALYSIS STAGE}
\Repeat{the outcome check is satisfactory}{
  choose an outcome model from \Cref{eq:outcome-model}\;
  fit the outcome model to the augmented dataset $\{(\mba_i, y_i,
  \hat{z}_i)\} ,\,\, i=1, \ldots, n$\;
  check the fitted outcome model\;
}

estimate the average causal effect $\E{}{Y(\mba)}-\E{}{Y(\mba')}$ by
\Cref{eq:causal-inference}\;
\caption{The Deconfounder}
\label{alg:deconfounder}
\end{algorithm}

% !TEX root = latent_confounder.tex

\glsreset{GWAS}

\glsreset{SNP}

\parhead{Example.} Consider a causal inference problem in \gls{GWAS}
\citep{stephens2009bayesian,visscher201710}: how do human genes
causally affect height? Here we give a brief account of how to use the
deconfounder, omitting many of the details.  We analyze \gls{GWAS}
problems extensively in \Cref{subsec:gwasstudy}.

Consider a dataset of $n = 5,000$ individuals; for each individual, we
measure height and genotype, specifically the alleles at $m = 100,000$
locations, called the \glspl{SNP}.  Each \gls{SNP} is represented by a
count of 0, 1, or 2; it encodes how many of the individual's two
nucleotides differ from the most common pair of nucleotides at the
location.  \Cref{tab:tab2_snp} illustrates a snippet of the data (10
individuals).

% !TEX root = latent_confounder.tex

\begin{table*}[ht]
\scriptsize
  \begin{center}
    \begin{tabular}{lp{9mm}p{9mm}p{9mm}p{9mm}p{9mm}p{9mm}p{9mm}p{9mm}p{9mm}p{3mm}p{12mm}p{15mm}} 
     \toprule
ID $(i)$ & SNP\_1 $(a_{i,1})$ & SNP\_2 $(a_{i,2})$& SNP\_3 $(a_{i,3})$&  SNP\_4 $(a_{i,4})$& SNP\_5 $(a_{i,5})$&SNP\_6 $(a_{i,6})$& SNP\_7 $(a_{i,7})$& SNP\_8 $(a_{i,8})$& SNP\_9 $(a_{i,9})$&$\cdots$&  SNP\_{100K} $(a_{i,100K})$ &Height (feet) $(y_i)$\\
          \midrule
1 & 1 & 0 & 0 & 1 & 0 & 0 & 1 & 2 & 0 & $\cdots$ & 0 & 5.73\\
2 & 1 & 2 & 2 & 1 & 2 & 1 & 1 & 0 & 1 & $\cdots$ & 2 & 5.26\\
3 & 2 & 0 & 1 & 1 & 0 & 1 & 0 & 1 & 1 & $\cdots$ & 2 & 6.24\\
4 & 0 & 0 & 0 & 1 & 1 & 0 & 1 & 2 & 0 & $\cdots$ & 0 & 5.78\\
5 & 1 & 2 & 1 & 1 & 1 & 0 & 1 & 0 & 0 & $\cdots$ & 1 & 5.09 \\
$\vdots$ & & & & & & $\vdots$ & & & & & & $\vdots$ \\
      \bottomrule
    \end{tabular}
    \caption{How do SNPs causally affect height?  This table shows a
      portion of a dataset: simulated SNPs as the multiple causes
      and height as the outcome.\label{tab:tab2_snp}}
  \end{center}
\end{table*}

%%% Local Variables:
%%% mode: latex
%%% TeX-master: t
%%% End:

We simulate such a dataset of genotypes and height. We generate each
individual's genotypes by simulating heterogeneous mixing of
populations \citep{pritchard2000association}. We then generate the
height from a linear model of the \gls{SNP}s (i.e. the assigned
causes) and some simulated confounders.  (The confounders are only
used to simulate data; when running the deconfounder, the confounders
are unobserved.) In this simulated data, the coefficients of the SNPs
are the true causal effects; we denote them $\beta^* = (\beta^*_1,
\ldots, \beta^*_m)$.  See \Cref{subsec:gwasstudy} for more details of
the simulation.

The goal is to infer how the \gls{SNP}s causally affect human height,
even in the presence of unobserved confounders.  The $m$-dimensional
\gls{SNP} vector $\mba_i = (a_{i1}, a_{i2}, ..., a_{im})$ is the
vector of assigned causes for individual $i$; the height $y_i$ is the
outcome. We want to estimate the potential outcome: what would the
(average) height be if we set a person's \gls{SNP} to be $\mba =
(a_{1}, a_{2}, ..., a_{m})$?  Mathematically, this is the average
potential outcome function: $\mathbb{E}[Y_i(\mba)]$, where the vector
of assigned causes $\mba$ takes values in $\{0, 1, 2\}^{m}$.

We apply the deconfounder: model the assigned causes, infer a
substitute confounder, and perform causal inference.  To infer a
substitute confounder, we fit a factor model of the assigned causes.
Here we fit a $50$-factor \gls{PF} model, as in
\Cref{eq:poisson-factorization}.  This fit results in estimates of
non-negative factors $\hat{\theta}_{j}$ for each assigned cause (a
$K$-vector) and non-negative weights $\hat{z}_{i}$ for each individual
(also a $K$-vector).

If the predictive check greenlights this fit, then we take the
posterior predictive mean of the assigned causes as the reconstructed
assignments, $\hat{a}_{j}(z_i) = \hat{z}_i^\top \hat{\theta}_j$. For
brevity, we do not report the predictive check here.  (The model
passes.)  We demonstrate predictive checks for \gls{GWAS} in the
empirical studies of \Cref{subsec:gwasstudy}.

Using the reconstructed assigned causes, we estimate the average
potential outcome function. Here we fit a linear outcome model to the
height $y_i$ against both of the assigned causes $\mba_i$ and
reconstructed assignment $\hat{\mba}(z_i)$,
\begin{align}
  y_i \sim \cN\left(\beta_0 + \beta^\top \mba_i +
  \gamma^\top\hat{\mba}(z_i), \sigma^2\right).
\end{align}
This regression is high dimensional $(m >n)$; for regularization, we
use an $L_2$-penalty on $\beta$ and $\gamma$ (equivalently, normal
priors).  Fitting the outcome model gives an estimate of regression
coefficients $\{\hat{\beta}_0, \hat{\beta}, \hat{\gamma}\}$.  Because
we use a linear outcome model, the regression coefficients
$\hat{\beta}$ estimate the true causal effect $\beta^*$.

\Cref{tab:example} evaluates the causal estimates obtained with and
without the deconfounder. We focus on the \gls{RMSE} of $\hat{\beta}$
to $\beta^*$. (``Causal estimation without the deconfounder'' means
fitting a linear model of the height $y_i$ against the assigned causes
$\mba_i$.) The deconfounder produces closer-to-truth causal estimates.

\glsreset{RMSE}

\begin{table}
  \begin{center}
    \begin{tabular}{lcc}
      \toprule
      & w/o deconfounder & w/ deconfounder \\
      \midrule
      RMSE$\times 10^{-2}$ &49.6&41.2\\ 
     \bottomrule
    \end{tabular}
    \caption{\Gls{RMSE} of the causal coefficients $\hat{\beta}$ with
      and without the deconfounder in a GWAS simulation study. We
      treat this \gls{RMSE} as a metric of how close the estimated
      potential outcome function is to the truth. In this toy problem,
      the deconfounder produces closer-to-truth causal estimates.
      \label{tab:example}}
  \end{center}
\end{table}

%%% Local Variables:
%%% mode: latex
%%% TeX-master: "latent_confounder"
%%% End:

\subsection{Connections to genome-wide association studies}
\label{subsec:connections}
% !TEX root = latent_confounder.tex

Many methods from the research literature, especially around
genome-wide association studies, can be reinterpreted as instances of
the deconfounder algorithm. Each can be seen as positing a factor
model of assigned causes (\Cref{subsec:assignment-model}) and a
conditional outcome model (\Cref{subsec:outcome-model}).

The deconfounder justifies each of these methods as forms of multiple
causal inference and, though predictive checks, points to how a
researcher can usefully compare and assess them.  Most of these
methods were motivated by imagining true unobserved confounding
structure.  However, the theory around the deconfounder shows that a
well-fitted factor model will capture confounders independent of a
researcher imagining what they may be; see the question in
\Cref{subsubsec:correct-model}.

Below we describe many methods from the \gls{GWAS} literature and show
how they can be viewed as deconfounder algorithms.  The \gls{GWAS}
problem is described in \Cref{subsec:fullalgo}.

\parhead{Linear mixed models.} The \gls{LMM} is one the most popular
classes of methods for analyzing GWAS~\citep{yu2006unified,
kang2008efficient, yang2014advantages, lippert2011fast,
loh2015efficient, JMLR:v18:15-143}.  Seen through the lens of the
deconfounder, an
\gls{LMM} posits a linear outcome model that depends on both the SNPs
and a scalar latent factor $Z_i$.

In the \gls{LMM} literature, $Z_i$ is not explicitly drawn from a
factor model; rather, $Z_{1:n}$ are from a multivariate Gaussian whose
covariance matrix, called the ``kinship matrix,'' is calculated from
the observed SNPs $\mba_{1:n}$.  However, this is mathematically
equivalent to posterior latent factors from a one-dimensional
\gls{PCA} model.  Subject to its capturing the distribution of SNPs,
the \gls{LMM} is performing multiple causal inference with a
deconfounder.

\parhead{Principal component analysis.}  A related approach is to
first perform (multi-dimensional) \gls{PCA} on the SNP matrix and then
to estimate an outcome model from the corresponding
residuals~\citep{price2006principal}.  This too is an instance of the
deconfounder.  As a factor model, \gls{PCA} is described in
\Cref{eq:prob-PCA}.  Fitting an outcome model to its residuals is
equivalent to conditioning on the reconstructed assignments,
\Cref{eq:condition-on-a-hat}.

\parhead{Logistic factor analysis.} Closely related to \gls{PCA} is
\gls{LFA}~\citep{song2015testing, hao2015probabilistic}.  \gls{LFA}
can be seen as the following factor model,
\begin{align*}
  \begin{split}
    Z_i &\sim \cN(0, I) \\
    \pi_{ij}\g Z_i &\sim \cN(z_i^\top \theta_j, \sigma^2), \quad j=1,\ldots,m, \\
    A_{ij}\g \pi_{ij} &\sim \textrm{Binomial}(2, \textrm{logit}^{-1}(\pi_{ij})), \quad j=1,\ldots,m.
  \end{split}
\end{align*}
If it captures the SNP matrix well, then $Z_i$ can be viewed as a
substitute confounder.

With \gls{LFA} in hand, \citet{song2015testing} use inverse regression
to perform association tests.  Their approach is equivalent to
assuming an outcome model conditional on the reconstructed assignments
$a(\hat{z}_i)$, again \Cref{eq:condition-on-a-hat}, and subsequently
testing for non-zero coefficients.

In a variant of \gls{LFA}, \citet{tran2017implicit} use a
neural-network based model of the unobserved confounder, connecting
this model to a causal inference with a nonparametric structural
equation model~\citep{Pearl:2009a}. They take an explicitly causal
view of the testing problem.

\parhead{Mixed-membership models.} Finally, many statistical
geneticists use mixed-membership models~\citep{Airoldi:2014} to
capture the latent population structure of SNPs, and then condition on
that structure in downstream analyses~\citep{pritchard2000inference,
pritchard2000association, falush2003inference, falush2007inference}.
In genetics, a mixed-membership model is a factor model that captures
latent ancestral populations.  The latent variable $Z_i$ is on the
$K-1$ simplex; it represents how much individual $i$ reflects each
ancestral population. The observed SNP $A_{ij}$ comes from a mixture
of Binomials, where $Z_i$ determines its mixture proportions.

Using these models, researchers use a linear outcome model conditional
on $z_i$ and devise tests for significant associations
\citep{pritchard2000association, song2015testing, tran2017implicit}.
The deconfounder justifies this practice from a causal perspective,
and underlines the importance of finding a model of population
structure that captures the per-individual distribution of SNPs.

%%% Local Variables:
%%% mode: latex
%%% TeX-master: "latent_confounder"
%%% End:

\subsection{A conversation with the reader}
\label{subsec:faq}
% !TEX root = latent_confounder.tex

In this section, we answer some questions a reader might have.

\subsubsection{Why do I need multiple causes?}

The deconfounder uses latent variables to capture dependence among the
assigned causes. The theory in \Cref{sec:theory} says that a latent
variable which captures this dependence will contain all valid
multi-cause confounders.  But estimating this latent variable requires
evidence for the dependence, and evidence for dependence cannot exist
with just one assigned cause.  Thus the deconfounder requires multiple
causes.

\subsubsection{Is the deconfounder free lunch?}

The deconfounder is not free lunch---it trades confounding bias for
estimation variance. Take an information point of view: the
deconfounder uses a portion of information in the data to estimate a
substitute confounder; then it uses the rest to estimate causal
effects. By contrast, classical causal inference uses all the
information to estimate causal effects, but it must assume
ignorability.  Put differently, while the deconfounder assumes the
weaker assumption of single ignorability, it pays for this flexibility
in the information it has available for causal estimation. Hence the
deconfounder estimate often has higher variance.

Suppose full ignorability is satisfied.  Then both classical causal
inference and the deconfounder provide unbiased causal estimates,
though the deconfounder will be less confident; it has higher
variance. Now suppose only single ignorability is satisfied.  The
deconfounder still provides unbiased causal estimates, but classical
causal inference is biased.

\subsubsection{Why does the deconfounder have two stages? }

\Cref{alg:deconfounder} first fits a factor model to the assigned
causes and then fits the potential outcome function.  This is a two
stage procedure.  Why? Can we fit these two models jointly?

One reason is convenience.  Good models of assigned causes may be
known in the research literature, such as for genetic studies.
Moreover, separately fitting the assignment model allows the
investigator to fit models to any available data of assigned causes,
including datasets where the outcome is not measured.

Another reason for two stages is to ensure that $Z_i$ does not contain
mediators, variables along the causal path between the assigned causes
and the outcome.  Intuitively, excluding the outcome ensures that the
substitute confounders are ``pre-treatment'' variables; we cannot
identify a mediator by looking only at the assigned causes. More
formally, excluding the outcome ensures that the model satisfies
$p(z_i \g \mba_i, y_i(\mba_i)) = p(z_i \g \mba_i)$; this equality
cannot hold if $Z_i$ contains a mediator.

\subsubsection{How does the deconfounder relate to the generalized
propensity score? What about instrumental variables? }

The deconfounder relates to both.

The deconfounder can be interpreted as a generalized propensity score
approach, except where the propensity score model involves latent
variables. If we treat the substitute confounder $Z_i$ as observed
covariates, then the factor model $P(A_i\g Z_i)$ is precisely the
propensity score of the causes $A_i$. With this view, the innovation
of the deconfounder is in $Z_i$ being latent. Moreover, it is the
multiplicity of the causes $A_{i1}, \ldots, A_{im}$ that makes a
latent $Z_i$ feasible; we can construct $Z_i$ by finding a random
variable that renders all the causes conditionally independent.

The deconfounder can also be interpreted as a way of constructing
instruments using latent factor models. Think of a factor model of the
causes with linearly separable noises:
$A_{ij}\stackrel{a.s.}{=}f(Z_i)+\epsilon_{ij}$. Given the substitute
confounder, consider the residual of the causes $\epsilon_{ij}$.
Assuming single ignorability, the variable $\epsilon_{ij}$ is an
instrumental variable for the $j$th cause $A_{ij}$. For example, with
probabilistic
\gls{PCA} the residual is~${\epsilon_{ij} = A_{ij} -
Z_i^\top \theta_j}\sim \cN(0, \sigma^2)$.

The residual $\epsilon_{ij}$ satisfies the requirements of being an
instrument for $A_{ij}$: (1) The residual $\epsilon_{ij}$ correlates
with the cause $A_{ij}$.  (2) The residual $\epsilon_{ij}$ affects the
outcome only through the cause $A_{ij}$; this fact is true because the
substitute confounder $Z_i$ is constructed without using any outcome
information. (3) The residual $\epsilon_{ij}$ cannot be correlated
with a confounder; this is true because $Z_i\perp \epsilon_{ij}$ by
construction from the factor model, where $P(Z_i)$ and $P(A_{ij}\g
Z_i)$ are specified separately.

However, the deconfounder differs from classical instrumental
variables approaches because it uses latent variable models to
construct instruments, rather than requiring that instruments be
observed. The latent variable construction is feasible because the
multiplicity of the causes allows us to construct $Z_i$ and
$\epsilon_{ij}$ from the conditional independence requirement.

\subsubsection{Does the factor model of the assigned causes need to be
  the true assignment model? Which factor model should I choose if
  multiple factor models return good predictive scores?}

\label{subsubsec:correct-model}

Finding a good factor model is not the same as finding the ``true''
model of the assigned causes. We do not assume the inferred variable
$Z_i$ reflects a real-world unobserved variable.

Rather, the deconfounder requires the factor model to capture the
population distribution of the assigned causes and, more particularly,
their dependence structure.  This requirement is why predictive
checking is important.  If the deconfounder captures the population
distribution---if the predictive check returns high predictive
scores---then we can use the inferred local variables $Z_i$ as
substitute confounders.

For the same reason, the deconfounder can rely on approximate
inference methods to infer the substitute confounder.  The predictive
check evaluates whether $Z_i$ provides a good predictive distribution,
regardless of how it was inferred. As long as the model and
(approximate) inference method together give a good predictive
distribution---one close to the population distribution of the
assigned causes---then the downstream causal inference is valid.  We
use approximate inference for most of the factor models we study in
\Cref{sec:study}.

\label{subsubsec:multiplefacpass}

Suppose multiple factor models give similarly good predictive scores
in the predictive check. In this case, we recommend choosing the
factor model with the lowest capacity. Factor models with similar
predictive scores often result in causal estimates with similarly
little bias. But the variance of these estimates can differ. Factor
models with high capacity can compromise overlap and lead to
high-variance estimates; factor models with low capacities tend to
produce lower variance causal estimates.  The empirical study in
\Cref{subsec:smoking} demonstrates this phenomenon.

\subsubsection{Can the causes be causally dependent among themselves?}

\label{subsubsec:factorexist}

When the causes are causally dependent, the deconfounder can still
provide unbiased estimates of the potential outcomes. Its success
relies on a valid substitute confounder.

Note there are cases where a valid substitute confounder cannot exist.
For example, consider a cause $A_1$ that causally affects $A_2$
according to $A_1\sim \cN(0,1), A_2 = A_1 + \epsilon, \epsilon\sim
\cN(0,1)$. In this case, a substitute confounder $Z$ must satisfy
$Z\stackrel{a.s.}{=} A_1$ or $Z\stackrel{a.s.}{=} A_2$, because it
needs to render the two causes conditionally independent. But such a
$Z$ does not satisfy overlap.

On the other hand, causal dependence among the causes does not
necessarily imply the nonexistence of a valid substitute confounder.
Consider a different mechanism for the causal relationship between
$A_1$ and $A_2$,
\begin{align*}
  A_1 &\sim \cN(0,1), \\
  A_2 &= |A_1| + \epsilon, \quad \epsilon\sim \cN(0,1).
\end{align*}
Here $Z \stackrel{a.s}{=} |A_1|$ is a valid substitute confounder; it
satisfies overlap and renders $A_1$ conditionally independent of
$A_2$.

Empirically, it is hard to detect the nonexistence of a valid
substitute confounder without knowing the functional form of how the
causes are structurally dependent. Insisting on using the deconfounder
in this case results in limited overlap and high variance causal
estimates downstream. We will illustrate this phenomenon in
\Cref{subsec:smoking}.

Finally, we recommend applying the deconfounder to non-causally
dependent causes. A valid substitute confounder is guaranteed to exist
in this case; it will both satisfy overlap and render the causes
conditionally independent of each other.

\subsubsection{Should I condition on known confounders and
covariates?}

\label{subsubsec:knownconf} Suppose we also observe known confounders
and other covariates $X_i$. The deconfounder maintains its theoretical
properties when we condition on observed covariates $X_i$ as well as
infer a substitute confounder $Z_i$. In particular, if $X_i$ is
``pre-treatment'' ---it does not include any mediators---then the
causal estimate will be unbiased \citep{imai2004causal} (also see
\Cref{thm:deconfounderfactor} below).  In general, it is good to
condition on observed confounders, especially if they may contain
single-cause confounders.

That said, we do not need to condition on observed confounders that
affect more than one of the causes; it suffices to condition only on
the substitute confounder $Z_i$.  And there is a trade off.
Conditioning on covariates $X_i$ maintains unbiasedness but it hurts
efficiency.  If the true causal effect size is small then large
confidence or credible intervals will conclude these small effects as
insignificant---inefficient causal estimates can bury the real causal
effects. The empirical study in \Cref{subsec:smoking} explores this
phenomenon.

\subsubsection{How can I assess the uncertainty of the deconfounder?}

\label{subsubsec:uncertainty} The uncertainty in the deconfounder
comes from two sources, the factor model and the outcome model. The
deconfounder first fits (and checks) the factor model; it gives a
substitute confounder $Z_i\sim p(z_i\g \mba_i)$. It then uses the mean
of the substitute confounder $\hat{z}_i=\E{\hat{M}}{Z_i\g \mba_i}$ to
fit an outcome model $p(y_i\g \mba_i, \hat{z}_i)$ and compute the
potential outcome estimate $\E{}{Y_i(\mba)}$.

To assess the uncertainty of the deconfounder, we consider the
uncertainty from both sources. We first draw $s$ samples $\{z_i^{(1)},
\ldots, z_i^{(s)}\}$ of the substitute confounder:
$z_i^{(l)}\stackrel{iid}{\sim} p(z_i\g \mba_i), l=1, \ldots, s$. For
each sample $z_i^{(l)}$, we fit an outcome model and compute a point
estimate of the potential outcome. (If the outcome model is
probabilistic, we compute the posterior distribution of its
parameters; this leads to a posterior of the potential outcome.) We
aggregate the estimates of the potential outcome (or its
distributions) from the $s$ samples $\{z_i^{(1)}, \ldots,
z_i^{(s)}\}$; the aggregated estimate is a collection of point
estimates of the potential outcome (or a mixture of its posterior
distributions). The variance of this aggregated estimate describes the
uncertainty of the deconfounder; it reflects how the finite data
informs the estimation of the potential outcome. In a two-cause
smoking study, \Cref{subsec:smoking} illustrates this strategy for
calculating the uncertainty of the deconfounder.

%%% Local Variables:
%%% mode: latex
%%% TeX-master: "latent_confounder"
%%% End:

%%% Local Variables:
%%% mode: latex
%%% TeX-master: "latent_confounder"
%%% End:

% !TEX root = latent_confounder.tex
\section{Empirical studies}

\label{sec:study}

\glsresetall

We study the deconfounder in three empirical studies. Two studies
involve simulations of realistic scenarios; these help assess how well
the deconfounder performs relative to ground truth.  In
\Cref{subsec:smoking} we study semi-synthetic data about smoking; the
causes are a real dataset about smoking and the effect (medical
expenses) is simulated.  In \Cref{subsec:gwasstudy} we study
semi-synthetic data about genetics.  Finally, in
\Cref{subsec:moviereal} we study real data about actors and movie
revenue; there is no simulation.  All three of these studies
demonstrate the benefits of the deconfounder.  They show how
predictive checks reveal potential issues with downstream causal
inference and how the deconfounder can provide closer-to-truth causal
estimates.

Each stage of the deconfounder requires computation: to fit the factor
model, to check the factor model, to calculate the substitute
deconfounder, and to fit the outcome model.  In all these stages, we
use \gls{BBVI} \citep{ranganath2014black} as implemented in Edward, a
probabilistic programming system \citep{tran2017deep, tran2016edward}.
(This was a choice; the deconfounder can be used with other methods
for calculating the posterior and fitting models. For example, we can
also use Stan \citep{carpenter2017stan}, which is a probabilistic
programming language available in R \citep{team2013r}.)

\newcommand{\age}{\textrm{age}}
\newcommand{\exposure}{\textrm{exp}}
\newcommand{\marital}{\textrm{mar}}
\newcommand{\linear}{\textrm{line}}
\newcommand{\quadr}{\textrm{quad}}

\subsection{Two causes: How smoking affects medical expenses}
\label{subsec:smoking}

We first study the deconfounder with semi-synthetic data about
smoking.  The 1987 National Medical Expenditures Survey (NMES)
collected data about smoking habits and medical expenses in a
representative sample of the U.S. population \citep{imai2004causal,
nmes1987}. The dataset contains 9,708 people and 8 variables about
each.  For each person, we focus on the current marital status
($a_{\marital}$), the cumulative exposure to smoking
($a_{\exposure}$), and the last age of smoking ($a_{\age}$). (We
standardize all variables.)

\parhead{A true outcome model and causal inference problem.} We use
the assigned causes from the survey to simulate a dataset of medical
expenses, which we will consider as the outcome variable.  Our true
model is linear,
\begin{align}
  \label{eq:true-model}
  y_i = \beta_{\marital} \, a_{\marital,i}
  + \beta_{\exposure} \, a_{\exposure,i}
  + \beta_{\age} \, a_{\age,i}
  + \varepsilon_i,
\end{align}
where $\varepsilon_i \sim \cN(0,1)$.  We generate the true causal
coefficients from
\begin{align}
  \beta_{\marital} \sim\cN(0,1) \quad
  \beta_{\exposure} \sim\cN(0,1) \quad
  \beta_{\age}\sim\cN(0,1).
\end{align}
and from these coefficients we generate the outcome for each
individual.  The result is a semi-synthetic dataset of 9,708 tuples
$(a_{\marital,i}, a_{\exposure,i}, a_{\age,i}, y_i)$.  The assigned
causes are from the real world, but we know the true outcome model.
Note that the last smoking age is a multi-cause confounder---it
affects both marital status and exposure and is one of the causes of
the expenses.

We are interested in the causal effects of marital status and smoking
exposure on medical expenses.  But suppose we do not observe age; it
is an unobserved confounder.  We can use the deconfounder to solve the
problem.

\begin{figure}[t]
  \centering
  \begin{subfigure}[b]{0.24\textwidth}
    \centering
    \includegraphics[scale=0.4]{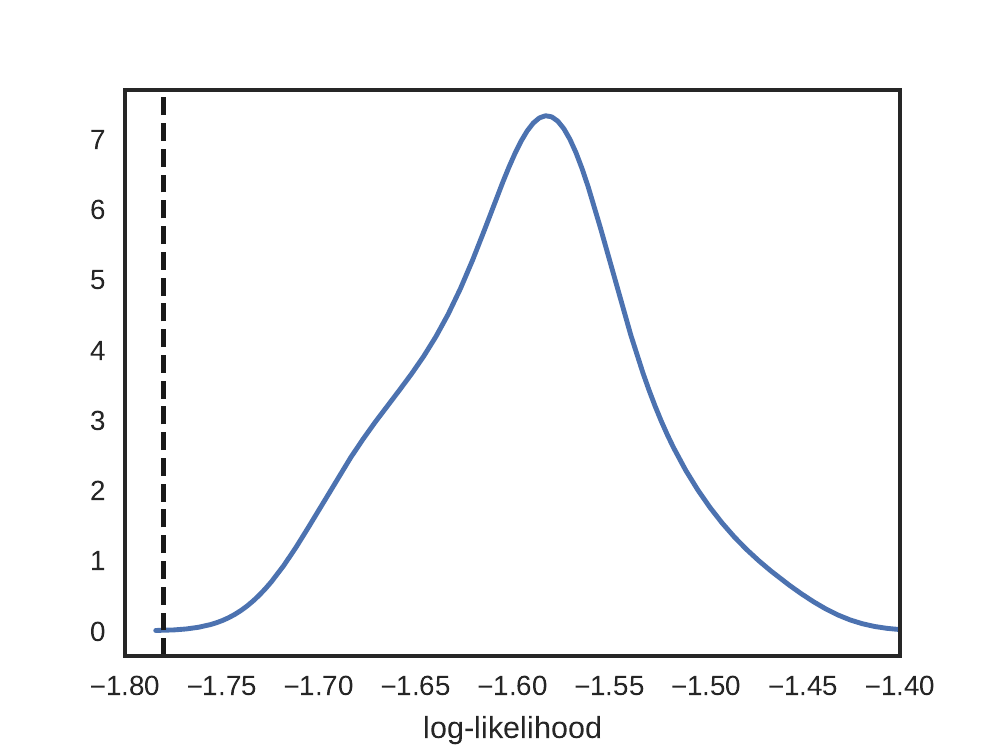}
    \caption{Linear Model\label{fig:linearppc}}
  \end{subfigure}%
  \begin{subfigure}[b]{0.24\textwidth}
    \centering
    \includegraphics[scale=0.4]{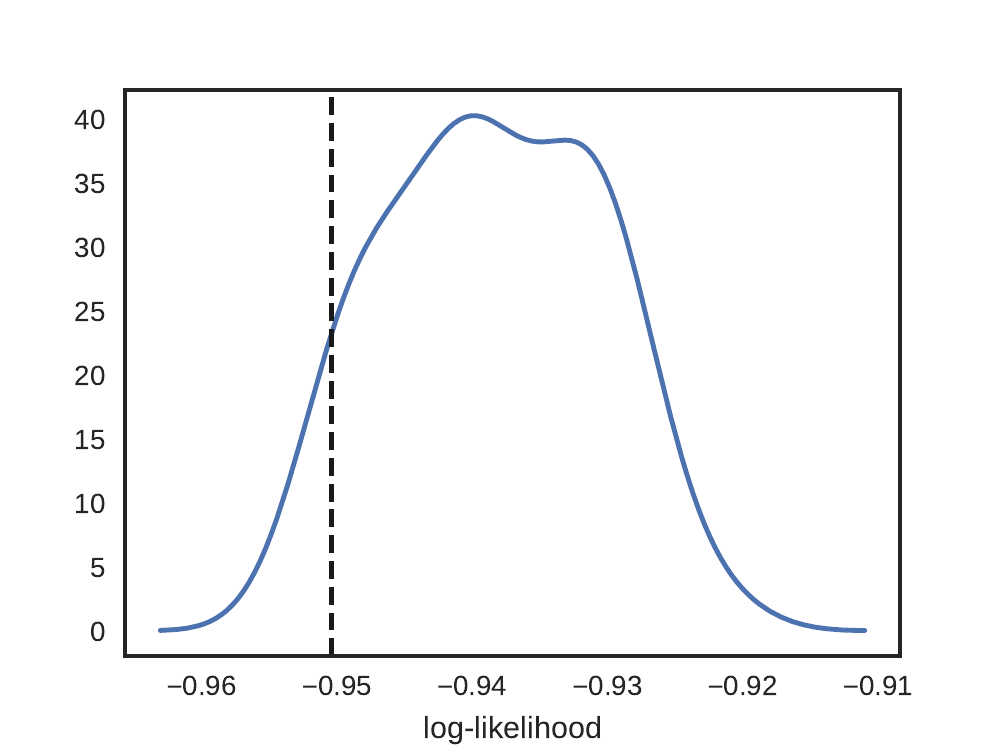}
    \caption{Quadratic Model\label{fig:quadraticppc}}
  \end{subfigure}
  \caption{Predictive checks for the substitute confounder $z$
    obtained from a linear factor model (a) and a quadratic factor
    model (b). The blue line is the \gls{KDE} of the
    test-statistic based on the predictive distribution. The dashed
    vertical line shows the value of the test-statistic on the
    observed dataset. The figure shows that the linear model
    mismatches the data---the observed statistic falls in a low
    probability region of the \gls{KDE}. The quadratic factor model is
    a better fit to the data.}
    \label{fig:smoking-ppc}
\end{figure}

\parhead{Modeling the assigned causes.}  We begin by finding a good
factor model of the assigned causes $(a_{\marital,i}, a_{\exp,i})$.
Because there are two observed assigned causes, we consider models
with a single scalar latent variable for overlap considerations.  (See
\Cref{subsec:identificationstrat}.)  We consider two factor models.

The first is a linear factor model,
\begin{align}
  z_{\linear,i} &\sim \cN(0,\sigma^2) \\
  a_{\marital,i} &= \eta^{(1)}_{\marital} \, z_{\linear,i} + \eta_{\marital}^{(0)} +
  \epsilon_{i,\marital}\\
  a_{\exp,i} &= \eta^{(1)}_{\exp} \, z_{\linear,i} + \eta_{\exp}^{(0)} +
               \varepsilon_{i,\exp},
\end{align}
where all errors are standard normal.  We fit this model with
variational Bayes \citep{blei2017variational}, which gives us
posterior estimates of the substitute confounders $z_{\linear,i}$.
Then we use the predictive check to evaluate it: following
\Cref{subsec:assignment-model}, we hold out a subset of the assigned
causes and using the expected log probability as the test statistic.
The resulting predictive score is 0.03, which signals a model
mismatch.  See \Cref{fig:smoking-ppc} (a).

We next consider a quadratic factor model,
\begin{align}
  z_{\quadr,i} &\sim \cN(0,\sigma^2) \\ a_{\marital,i} & =
               \eta^{(1)}_{\marital} \, z_{\quadr,i} +
               \eta^{(2)}_{\marital} \, z_{\quadr,i}^2 +
               \eta_{\marital}^{(0)} + \varepsilon_{i,\marital} \\
  a_{\exp,i} &= \eta^{(1)}_{\exp} \, z_{\quadr,i} +
               \eta_{\exp}^{(2)} \, z_{\quadr,i}^2 +
               \eta_{\exp}^{(0)} + \varepsilon_{i, \exp},
\end{align}
where all errors are standard normal.  We again fit this model with
variational Bayes and used a predictive check.  The resulting
predictive score is 0.12, \Cref{fig:smoking-ppc} (b).  This value
gives the green light.  We use the model's posterior estimates
$\hat{z}_i \sim p_{\quadr}({Z \g A=a_i})$ to form a substitute
confounder in a causal inference.

\parhead{Deconfounded causal inference.} Using a factor model to
estimate substitute confounders, we proceed with causal inference.  We
set the outcome model of $\E{}{Y(A_{\marital}, A_{\exposure}) \g A,
Z}$ to be linear in $a_{\marital}$ and $a_{\exposure}$.  In one form,
the linear model conditions on $\hat{z}$ directly.  In another it
conditions on the reconstructed causes, e.g. for the quadratic model
and for age,
\begin{align}
  a_{\marital,i}(\hat{z}_i) = \E{\quadr}{A_{\marital} \g Z = \hat{z}_i}.
\end{align}
See \Cref{eq:condition-on-a-hat}.

We use predictive checks to evaluate the outcome models. Conditioning
on $\hat{z}$ gives a predictive score of 0.05; conditioning on
$a(\hat{z})$ gives a predictive score of 0.18. The model with
reconstructed causes is better.

If the outcome model is good and if the substitute confounder captures
the true confounders then the estimated coefficients for age and
exposure will be close to the true $\beta_{\marital}$ and
$\beta_{\exposure}$ of \Cref{eq:true-model}.  We emphasize that
\Cref{eq:true-model} is the true mechanism of the simulated world,
which the deconfounder does not have access to.  The linear model we
posit for $\E{}{Y(A_{\marital}, A_{\exposure}) \g A, Z}$ is a
functional form for the expectation we are trying to estimate.

\parhead{Performance.}  We compare all combinations of factor model
(linear, quadratic) and outcome-expectation model (conditional on
$\hat{z}_i$ or $a(\hat{z}_i)$).  \Cref{tab:twocause} gives the
results, reporting the total bias and variance of the estimated causal
coefficients $\beta_{\marital}$ and $\beta_{\exposure}$. We compute
the variance by drawing posterior samples of the substitute confounder
and the resulting posterior samples of the causal coefficients.

\Cref{tab:twocause} also reports the estimates if we had observed the
age confounder (oracle), and the estimates if we neglect causal
inference altogether and fit a regression to the confounded data.
Neglecting causal inference gives biased causal estimates; observing
the confounder corrects the problem.

How does the deconfounder fare?  Using the deconfounder with a linear
factor model yields biased causal estimates, but we predicted this
peril with a predictive check.  Using the deconfounder with the
quadratic assignment model, which passed its predictive check,
produces less biased causal estimates. (The estimate with
one-dimensional $z_{\textrm{quad}}$ was still biased, but the outcome
check revealed this issue.)

We also use this simulation study to illustrate a few questions
discussed in \Cref{subsec:faq}:
\begin{itemize}

\item \parhead{What if multiple factor models pass the check?
    (\Cref{subsubsec:multiplefacpass})} We fit to the causes
  one-dimensional, two-dimensional, and three-dimensional quadratic
  factor models. All three models pass the check. \Cref{tab:twocause}
  shows that they yield estimates with similar bias. However, factor
  models with higher capacity in general lead to higher variance. The
  one-dimensional factor model, which is the smallest factor model
  that passes the check achieve the best mean squared error.

\item \parhead{Should we additionally condition on the observed
    covariates?  (\Cref{subsubsec:knownconf})} \Cref{tab:twocause}
  shows that using the deconfounder, along with covariates, preserves
  the unbiasedness of the causal estimates, but it inflates the
  variance. (The covariates include gender, race, seat belt usage,
  education level, and the age of starting to smoke.)

\item \parhead{What if some causes are causally dependent among themselves?
  (\Cref{subsubsec:factorexist})} We repeat the above experiments with
  the same confounder $a_{\age}$ but three causes: $a_{\marital},
  a_{\exposure}$ and an additional cause $a_{\marital+}$. We assume
  $a_{\marital+}$ causally depend on $a_{\marital}$, where
  \begin{align} 
  a_{\marital+} = a_{\marital} +
  \varepsilon_{i,\marital+}, \qquad \varepsilon_{i,\marital+}\sim\cN(0,1).
  \label{eq:dependentcause}
  \end{align}
  We simulate the outcome from
  \begin{align}
  y_i = \beta_{\marital} \, a_{\marital,i}
  + \beta_{\exposure} \, a_{\exposure,i}
  + \beta_{\age} \, a_{\age,i}
  + \beta_{\marital+} \, a_{\marital+,i}
  + \varepsilon_i,
\end{align}
where $\varepsilon_i \sim \cN(0,1)$.  We generate the true causal
coefficients from
\begin{align}
  \beta_{\marital} \sim\cN(0,1) \quad
  \beta_{\exposure} \sim\cN(0,1) \quad
  \beta_{\age}\sim\cN(0,1) \quad 
  \beta_{\marital+}\sim\cN(0,1).
\end{align}
\Cref{eq:dependentcause} implies that theoretically there exists no
substitute confounders that can both satisfy overlap and render the
causes conditionally independent; see discussion in
\Cref{subsubsec:factorexist}.

Nevertheless, we apply the deconfounder to this data. We model the
three causes with one-dimensional linear and quadratic factor model;
both pass the predictive check, with a predictive score of 0.28 and
0.20. \Cref{tab:dependentcause} shows the bias and variance of the
deconfounder estimate of $\beta_{\marital}$ and $\beta_{\exposure}$.
With causally dependent causes (\Cref{tab:dependentcause}), the
deconfounder estimates have much larger variance than usual
(\Cref{tab:twocause}); it signals that the substitute confounder we
constructed is close to breaking overlap. That said, the deconfounder
is still able to correct for a substantial portion of confounding
bias.
\end{itemize}

% !TEX root = latent_confounder.tex

\begin{table}[t]
  \begin{center}
    \begin{tabular}{lcrrr} 
      \toprule
      & Check & Bias$^2 \times 10^{-2}$ & Variance $\times 10^{-2}$ & MSE $\times 10^{-2}$\\
      \midrule
      No control & -- & 24.19 & 0.28 & 24.48 \\      
      Control for age (oracle) & -- & 5.06 & 0.07 & 5.14\\
      \midrule
      \textbf{Deconfounder}\\
      \midrule 
      Control for 1-dim $z_{\linear}$ & \xmark & 21.51 & 4.48 & 25.99\\
      Control for 1-dim $a(z_{\linear})$ & \xmark & 20.02 &  4.77 & 24.80\\
      Control for 1-dim $z_{\quadr}$ & \cmark & 17.77 & 5.59 & 23.36\\
      Control for 1-dim $a(z_{\quadr})$ & \cmark & \textbf{11.55} & 5.95 & \textbf{17.51}\\
      \cdashlinelr{1-5}
      Control for 2-dim $z_{\quadr}$ & \cmark & 15.08  & 7.49 & 22.58\\
      Control for 2-dim $a(z_{\quadr})$ & \cmark & \textbf{12.47} & 6.95 & \textbf{19.42}\\
      Control for 3-dim $z_{\quadr}$ & \cmark & 16.24 & 7.74 & 23.99\\
      Control for 3-dim $a(z_{\quadr})$ & \cmark & 13.62 & 8.91 & 22.53\\
      \midrule
      \textbf{Deconfounder with covariates}\\
      \midrule
      Control for 1-dim $z_{\quadr}, x$ & \cmark & 16.15 & 6.22 & 22.38\\
      Control for 1-dim $a(z_{\quadr}), x$ & \cmark & \textbf{14.47} & 7.55 & \textbf{22.03}\\
\bottomrule
    \end{tabular}
    \caption{Total bias and variance of the estimated causal
coefficients $\beta_{\exposure}$ and
$\beta_{\marital}$\label{tab:twocause}. (``Control for xxx'' means we
include xxx as a covariate in the linear outcome model. The
\cmark~symbol indicates the factor model gives a predictive score
larger than 0.1; the \xmark~symbol indicates otherwise.) Not
controlling for confounders yields biased causal estimates. So does
using deconfounder with a poor $Z$-model that fails model checking.
Deconfounder with a good $Z$-model and a good outcome model
significantly reduces the bias in causal estimates; controlling for
the ``reconstructed causes'' $\hat{a}$ yields less biased estimates
than the substitute confounder $Z$. Models that pass the check usually
yield estimates with similar bias, but their variance grows as the
capacity of the model grows. Using deconfounder along with covariates
preserves the reduction in bias; yet, it inflates the variance.}
  \end{center}
\end{table}

\begin{table}[t]
  \begin{center}
    \begin{tabular}{lcrrr} 
      \toprule
      & Check & Bias$^2 \times 10^{-2}$ & Variance $\times 10^{-2}$ & MSE $\times 10^{-2}$\\
      \midrule
      No control & -- & 41.89 & 0.01 & 41.90 \\   
      Control for age (oracle) & -- & 22.57 & 0.01 & 22.57\\   
      \midrule 
      Control for 1-dim $z_{\linear}$ & \cmark & 29.98 & 16.97 & 46.96\\
      Control for 1-dim $a(z_{\linear})$ &  \cmark & 28.01 & 18.49 & 46.50\\
      \cdashlinelr{1-5}
      Control for 1-dim $z_{\quadr}$ & \cmark & \textbf{25.10} & 16.70 & \textbf{41.80}\\
      Control for 1-dim $a(z_{\quadr})$ & \cmark & 27.46 & 15.77 & 43.23\\
      \bottomrule
    \end{tabular}
    \caption{Total bias and variance of the estimated causal
coefficients $\beta_{\exposure}$ and $\beta_{\marital}$ when there is a
third cause dependent on $a_{\marital}$. The nonlinear factor model
outperforms linear factor model. The deconfounder estimate has much
higher variance than usual (\Cref{tab:twocause}) when two of the
causes are dependent.\label{tab:dependentcause}}
  \end{center}
\end{table}

This study provides two takeaway messages: (1) It is crucial to check
both the assignment model and the outcome model; (2) Unless a
single-cause confounder believably exists, we do not need to accompany
the deconfounder with other observed covariates; (3) Use the
deconfounder.

\subsection{Many causes: Genome-wide association studies}
\label{subsec:gwasstudy}
\glsresetall

Analyzing gene-wide association studies (GWAS) is an important problem
in modern genetics \citep{stephens2009bayesian, visscher201710}.  The
GWAS problem involves large datasets of human genotypes and a trait of
interest; the goal is to determine how genetic variation is causally
connected to the trait.  GWAS is a problem of multiple causal
inference: for each individual, the data contains a trait and hundreds
of thousands of \glspl{SNP}, measurements on various locations on the
genome.

One benefit of GWAS is that biology guarantees that genes are
(typically) cast in advance; they are potential causes of the trait,
and not the other way around.  However there are many confounders. In
particular, any correlation between the SNPs could induce confounding.
Suppose the value of SNP $i$ is correlated with the value of SNP $j$,
and SNP $j$ is causal for the outcome.  Then a naive analysis will
find a connection between gene $i$ and the outcome.  There can be many
sources of correlation; common sources include population structure,
i.e., how the genetic codes of an individuals exhibits their ancestral
populations, and lifestyle variables.  We study how to use the
deconfounder to analyze GWAS data. (Many existing methods to analyze
GWAS data can be seen as versions of the deconfounder; see
\Cref{subsec:connections}.)

\parhead{Simulated GWAS data and the causal inference problem.}  We
put the GWAS problem into our notation.  The data are tuples
$(\mba_i, y_i)$, where $y_{i}$ is a real-valued trait and
$a_{ij} \in \{0,1,2\}$ is the value of SNP $j$ in individual $i$.
(The coding denotes ``unphased data,'' where $a_{ij}$ codes the number
of minor alleles---deviations from the norm---at location $j$ of the
genome.)  As usual, our goal is to estimate aspects of the
distribution of $y_i(\mba)$, the trait of interest as a function of a
specific genotype.

We generate synthetic GWAS data.  Following \citet{song2015testing},
we simulate genotypes $\mba_{1:n}$ from an array of realistic models.
These include models generated from real-world fits, models that
simulate heterogeneous mixing of populations, and models that simulate
a smooth spatial mixing of populations.  For each model, we produce
datasets of genotypes with 100,000 SNPs and 1000-5000 individuals.
\Cref{sec:genesimdetail} details the configurations of the simulation.

With the individuals in hand, we next generate their traits.  Still
following \citet{song2015testing}, we generate the outcome (i.e., the
trait) from a linear model,
\begin{align}
  \label{eq:gwas-outcome}
  y_i = \sum_{j} \beta_j a_{ij} + \lambda_{c_i} + \varepsilon_i.
\end{align}
To introduce further confounding effects, we group the individuals by
their SNPs; the $i$th individual is in group $c_i$.
(\Cref{sec:genesimdetail} describes how individuals are grouped.) Each
group is associated with a per-group intercept term $\lambda_c$ and a
per-group error variance $\sigma_c$, where the noise
$\varepsilon_i \sim \cN(0, \sigma_c^2)$. In our empirical study, the
group indicator of each individual is an unobserved confounder.

In \Cref{eq:gwas-outcome}, SNP $j$ is associated with a true causal
coefficient $\beta_j$.  We draw this coefficient from $\cN(0, 0.5^2)$
and truncate so that 99\% of the coefficients are set to zero (i.e.,
no causal effect).  Such truncation mimics the sparse causal effects
that are found in the real world.  Further, we impose a low
signal-to-noise ratio setting; we design the intercept and random
effects such that the SNPs $\sum_{j} \beta_j a_{ij}$ contributes 10\%
of the variance, the per-group intercept $\lambda_{c_i}$ contributes
20\% , and the error $\varepsilon_i$ contributes 70\%. We also study a
high signal-to-noise ratio setting where the SNPs signal contributes
40\%, the per-group intercept contributes 40\% and the error
contributes 20\%.

In a separate set of studies, we generate binary outcomes.  They come
from a generalized linear model,
\begin{align}
  \label{eq:gwas-outcome}
  y_i \sim \text{Bernoulli}\left(
  \frac{1}{1 +
  \exp(\sum_{j} \beta_j a_{ij} + \lambda_{c_i} + \varepsilon_i)}
  \right).
\end{align}
We will study the deconfounder for both binary or real-valued
outcomes.

For each true assignment model of $\mba_{i}$, we simulate 100 datasets
of genotypes $\mba_i$, causal coefficients $\beta_j$, and outcomes
$y_i$ (real and binary).  For each, the causal inference problem is to
infer the causal coefficients $\beta_j$ from tuples $(\mba_i, y_i)$.
The unobserved confounding lies in the correlation structure of the
SNPs and the unobserved groups.  We correct it with the deconfounder.

\parhead{Deconfounding GWAS.} We apply the deconfounder with five
assignment models discussed in \Cref{subsec:deconfounder}: \gls{PPCA},
\gls{PF}, \glspl{GMM}, the three-layer
\gls{DEF}, and \gls{LFA}; none of these models is the true assignment
model. (We use $50$ latent dimensions so that most pass the predictive
check; for the \gls{DEF} we use the structure $[100, 30, 15]$.) We fit
each model to the observed SNPs and check them with the per-individual
predictive checks from \Cref{subsec:assignment-model}.

With the fitted assignment model, we estimate the causal effects of
the SNPs.  For real-valued traits, we use a linear model conditional
on the \gls{SNP}s and the reconstructed causes $a(\hat{z})$; see
\Cref{eq:condition-on-a-hat}. Each assignment model gives a different
form of $a(\hat{z})$.  For the binary traits, we use a logistic
regression, again conditional on the SNPs and reconstructed causes.
We emphasize that these are not the true model of the outcome, but
rather models of the random potential outcome function.

\glsreset{RMSE}

\parhead{Performance.}  We study the deconfounder for
GWAS.
\Cref{tab:highSNRBN,tab:highSNRTGP,tab:highSNRHGDP,tab:highSNRPSD,tab:highSNRspatial,tab:BN,tab:TGP,tab:HGDP,tab:PSD,tab:spatial}
present the full results across the 11 different configurations and
both high and low \gls{SNR} settings. Each table is attached to a true
assignment model and reports results across different factor models of
the SNPs.  For each factor model, the tables report the results of the
predictive check and the \gls{RMSE} of the estimated causal
coefficients (for real-valued and binary-valued outcomes).
\Cref{tab:highSNRBN,tab:highSNRTGP,tab:highSNRHGDP,tab:highSNRPSD,tab:highSNRspatial,tab:BN,tab:TGP,tab:HGDP,tab:PSD,tab:spatial}
also report the error if we had observed the confounder and if we
neglect causal inference by fitting a regression to the confounded
data.

\glsreset{LMM}

On both real and binary outcomes, the deconfounder gives good causal
estimates with \gls{PPCA}, \gls{PF}, \gls{LFA}, \glspl{LMM}, and
\glspl{DEF}: they produce lower \gls{RMSE}s than blindly fitting
regressions to the confounded data. (The linear mixed model does not
explicitly posit an assignment model so we omit the predictive check.
It can be interpreted as the deconfounder though; see
\Cref{subsec:connections}.)  Notably, the deconfounder often
outperforms the regression where we include the (unobserved)
confounder as a covariate under the low \gls{SNR} setting; see
\Cref{tab:BN,tab:TGP,tab:HGDP,tab:PSD}.

In general, predictive checks of the factor models reveal downstream
issues with causal inference: better factor models of the assigned
causes, as checked with the predictive checks, give closer-to-truth
causal estimates. For example, the \gls{GMM} does not perform well as
a factor model of the assignments; it struggles with fitting
high-dimensional data and can amplify the causal effects (see
e.g. \Cref{tab:spatial}). But checking the \gls{GMM} signals this
issue beforehand; the \gls{GMM} constantly yields close-to-zero
predictive scores in predictive checks.

Among the assignment models, the three-layer \gls{DEF} almost always
produces the best causal estimates. Inspired by deep neural networks,
the \gls{DEF} has layered latent variables; see
\Cref{subsec:assignment-model}.  The \gls{DEF} model of SNPs uses
Gamma distributions on the latent variables (to induce sparsity) and a
bank of Poisson distributions to model the observations.

The deconfounder is most challenged when the assigned SNPs are
generated from a spatial model; see
\Cref{tab:highSNRspatial,tab:spatial}. The spatial model produces
spatially-correlated individuals; its parameter $\tau$ controls the
spatial dispersion. (Consider each individual to sit in a unit square;
as $\tau \rightarrow 0$, the individuals are placed closer to the
corners of the unit square while when $\tau = 1$ they are distributed
uniformly.)  The five factor models---\gls{PPCA},
\gls{PF}, \gls{LFA}, \gls{GMM}, \gls{LMM}, and \gls{DEF}---all produce
closer-to-truth causal estimates than when ignoring confounding
effects.  But they are farther from the truth than the estimates that
use the (unobserved) confounder. Again, the predictive check hints at
this issue.  When the true distribution of SNPs is a spatial model,
the predictive scores are generally more extreme (i.e., closer to zero).

\begin{figure}[t]
  \centering
  \begin{subfigure}[b]{0.2\textwidth}
    \centering
    \includegraphics[scale=0.6]{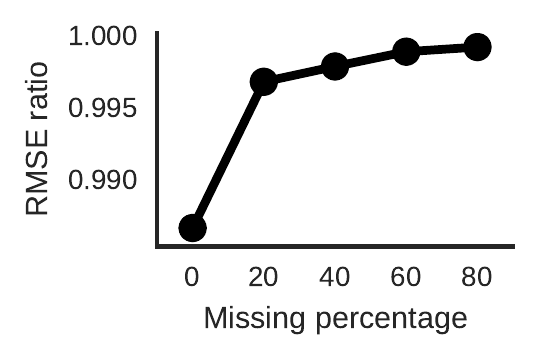}
    \caption{Balding-Nichols}
  \end{subfigure}%
  \begin{subfigure}[b]{0.2\textwidth}
    \centering
    \includegraphics[scale=0.6]{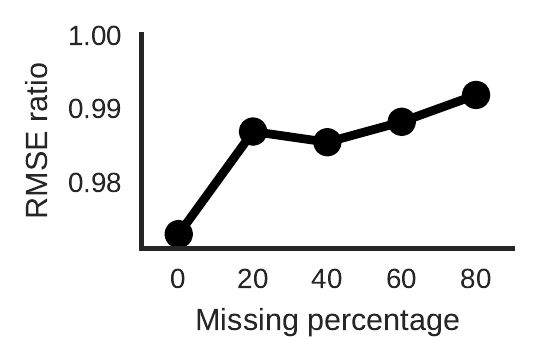}
    \caption{TGP}
  \end{subfigure}%
  \begin{subfigure}[b]{0.2\textwidth}
    \centering
    \includegraphics[scale=0.6]{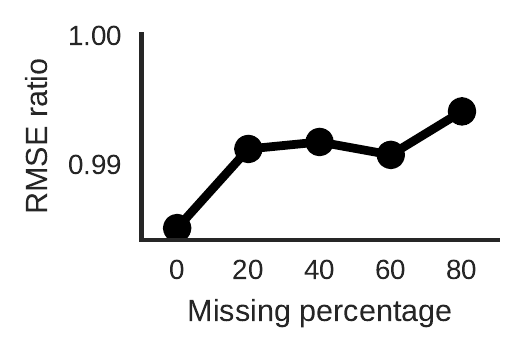}
    \caption{HDGP}
  \end{subfigure}%
  \begin{subfigure}[b]{0.2\textwidth}
    \centering
    \includegraphics[scale=0.6]{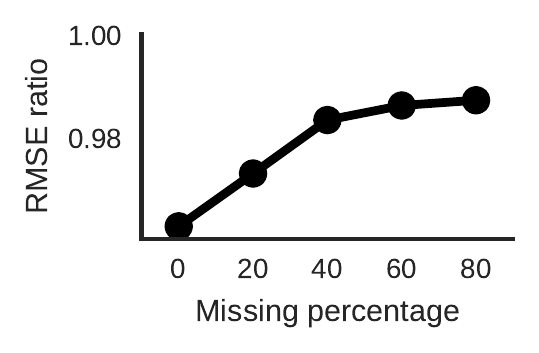}
    \caption{PSD $(\alpha=0.01)$}
  \end{subfigure}%
  \begin{subfigure}[b]{0.2\textwidth}
    \centering
    \includegraphics[scale=0.6]{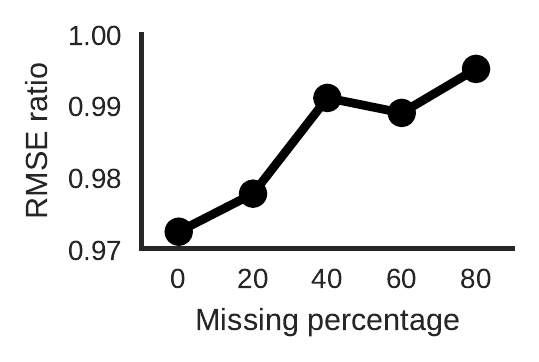}
    \caption{Spatial $(\tau=0.1)$}
  \end{subfigure}%
  \caption{The \gls{RMSE} ratio between the deconfounder with
    \gls{DEF} and ``No control'' across simulations when only a subset
    of causes are unobserved. (Lower ratios means more correction.) As
    the percentage of observed causes decreases, the single strong
    ignorability is compromised; the deconfounder can no longer
    correct for all latent confounders. }
    \label{fig:missingtreat}
\end{figure}

\parhead{Partially observed causes.} Finally, we study the situation
where some assigned causes are unobserved, that is, where some of the
SNPs are not measured.  Recall that the deconfounder assumes
\emph{single strong ignorability}, that all single-cause confounders
are observed.  This assumption may be plausible when we measure all
assigned causes but it may well be compromised when we only observe a
subset---if a confounder affects multiple causes but only one of those
causes is observed then the confounder becomes a single-cause
confounder.

Using the simulated GWAS data, we randomly mask a percentage of the
causes. We then use the deconfounder to estimate the causal effects of
the remaining causes.  To simplify the presentation, we focus on the
DEF factor model.  \Cref{fig:missingtreat} shows the ratio of the
\gls{RMSE} between the deconfounder and ``no control''; a ratio closer
to one indicates a more biased causal estimate.  Across simulations,
the \gls{RMSE} ratio increases toward one as the percentage of
observed causes decreases.  With fewer observed causes, it becomes
more likely for single-strong ignorability to be compromised.

\parhead{Summary.}  These studies provide three take-away messages:
(1) The deconfounder can produce closer-to-truth causal estimates,
especially when we observe many assigned causes; (2) Predictive checks
reveal downstream issues with causal inference, and better factor
models give better causal estimates; (3) \glspl{DEF} can be a handy
class of factor models in the deconfounder.

% !TEX root = latent_confounder.tex
\subsection{Case study: How do actors boost movie earnings?}
\label{subsec:moviereal}

We now return to the example from \Cref{sec:introduction}: How much
does an actor boost (or hurt) a movie's revenue? We study the
deconfounder with the TMDB 5000 Movie
Dataset.\footnote{https://www.kaggle.com/tmdb} It contains 901 actors
(who appeared in at least five movies) and the revenue for the 2,828
movies they appeared in. The movies span 18 genres and 58 languages.
(More than 60\% of the movies are in English.) We focus on the cast
and the log of the revenue. Note that this is a real-world
observational data set. We no longer have ground truth of causal
estimates.

The idea here is that actors are potential causes of movie earnings:
some actors result in greater revenue.  But confounders abound.
Consider the genre of a movie; it will affect both who is in the cast
and its revenue.  For example, an action movie tends to cast action
actors, and action movies tend to earn more than family movies.  And
genre is just one possible confounder: movies in a series, directors,
writers, language, and release season are all possible confounders.

We are interested in estimating the causal effects of individual
actors on the revenue.  The data are tuples of $(\mba_i, y_i)$, where
$a_{ij} \in \{0,1\}$ is an indicator of whether actor $j$ in movie
$i$, and $y_i$ is the revenue. \Cref{tab:toprevmov} shows a snippet of
the highest-earning movies in this dataset.  The goal is to estimate
the distribution of $Y_i(\mba)$, the (potential) revenue as a function
of a movie cast.

\glsreset{GMM} \glsreset{PPCA} \glsreset{PF} \glsreset{DEF}

\parhead{Deconfounded causal inference.} We apply the deconfounder.
We explore four assignment models: \gls{PPCA}, \gls{PF}, \glspl{GMM},
and \glspl{DEF}.  (Each has 50 latent dimensions; the \gls{DEF} has
structure $[50, 20, 5]$.)  We fit each model to the observed movie
casts and check the models with a predictive check on held-out data;
see \Cref{subsec:assignment-model}.

The \gls{GMM} fails its check, yielding a predictive score < 0.01. The
other models adequately capture patterns of actors: the checks return
predictive scores of 0.12 (\gls{PPCA}), 0.14 (\gls{PF}), and 0.15
(\gls{DEF}).  These numbers give a green light to estimate how each
actor affects movie earnings.

With a fitted and checked assignment model, we estimate the causal
effects of individual actors with a log-normal regression, conditional
on the observed casts and ``reconstructed casts,''
\Cref{eq:condition-on-a-hat}.

\parhead{Results: Predicting the revenue of uncommon movies.} We
consider test sets of uncommon movies, where we simulate an
``intervention'' on the types of movies that are made.  This changes
the distribution of casts to be different from those in the training
set.

For such data, a good causal model will provide better predictions
than a purely predictive model.  The reason is that predictions from a
causal model will work equally well under interventions as for
observational data. In contrast, a non-causal model can produce
incorrect predictions if we intervene on the causes
\citep{peters2016causal}.  This idea of invariance has also been
discussed in \citet{haavelmo1944probability, aldrich1989autonomy,
  lanes1988logic, Pearl:2009a, scholkopf2012causal,
  dawid2010identifying} under the terms ``autonomy,'' ``modularity,''
and ``stability.''

In one test set, we hold out 10\% of non-English-language movies.
(Most of the movies are in English.)  \Cref{table:languagepredict}
compares different models in terms of the average predictive log
likelihood.  The deconfounder predicts better than both the purely
predictive approach (no control) and a classical approach, where we
condition on the observed (pre-treatment) covariates.

In another test set, we hold out 10\% of movies from uncommon genres,
i.e., those that are not comedies, action, or dramas.
\Cref{table:genrepredict} shows similar patterns of performance. The
deconfounder predicts better than purely predictive models and than
those that control for available confounders.

For comparison, we finally analyze a typical test set, one drawn
randomly from the data.  Here we expect a purely predictive method to
perform well; this is the type of prediction it is designed for.
\Cref{table:moviepredict} shows the average predictive log likelihood
of the deconfounder and the purely predictive method. The deconfounder
predicts slightly worse than the purely predictive method.

\parhead{Exploratory analysis of actors and movies.}  We show how to
use the deconfounder to explore the data, understanding the causal
value of actors and movies.\footnote{This section illustrates how to
  use the deconfounder to explore data.  It is about these methods and
  the particular dataset that we studied, not a comment about the
  ground-truth quality of the actors involved.  The authors of this
  paper are statisticians, not film critics.}

First we examine how the coefficients of individual actors differ
between a non-causal model and a deconfounded model.  (In this
section, we study the deconfounder with \gls{PF} as the assignment
model.)  We explore actors with $n_j \beta_j$, their estimated
coefficients scaled by the number of movies they appeared in.  This
quantity represents how much of the total log revenue is ``explained''
by actor $j$.

Consider the top 25 actors in both the corrected and uncorrected
models. In the uncorrected model, the top actors are movie stars such
as Tom Cruise, Tom Hanks, and Will Smith.  Some actors, like Arnold
Schwartzenegger, Robert De Niro, and Brad Pitt, appear in the top-25
uncorrected coefficients but not in the top-25 corrected coefficients.
In their place, the top 25 causal actors include actors that do not
appear in as many blockbusters, such as Owen Wilson, Nick Cage, Cate
Blanchett, and Antonio Banderes.

Also consider the actors whose estimated contribution improves the
most from the non-causal to the causal model.  The top five ``most
improved'' actors are Stanley Tucci, Willem Dafoe, Susan Sarandon, Ben
Affleck, and Christopher Walken.  These (excellent) actors often
appear in smaller movies.

Next we look at how the deconfounder changes the causal estimates of
movie casts.  We can calculate the movie casts whose causal estimates
are decreased most by the deconfounder.  The ``causal estimate of a
cast'' is the predicted revenue \textit{without} including the term
that involves the confounder; this is the portion of the predicted log
revenue that is attributed to the cast.

At the top of this list are blockbuster series.  Among the top 25
include all of the \textit{X-Men} movies, all of the \textit{Avengers}
movies, and all of the \textit{Ocean's X} movies.  Though unmeasured
in the data, being part of a series is a confounder. It affects both
the casting and the revenue of the movie: sequels must contain
recurring characters and they are only made when the producers expect
to profit.  In capturing the correlations among casts, the
deconfounder corrects for this phenomenon.

%%% Local Variables:
%%% mode: latex
%%% TeX-master: "latent_confounder"
%%% End:

%%% Local Variables:
%%% mode: latex
%%% TeX-master: "latent_confounder"
%%% End:

% !TEX root = latent_confounder.tex
\section{Theory}
\label{sec:theory}

We develop theoretical results around the deconfounder. (All proofs
and proof sketches are in the appendix.)

We first justify the use of factor models by connecting them to the
ignorability assumption. We show that factor models imply
ignorability. We next establish theoretical properties of the
substitute confounder: it captures all multi-cause confounders and it
does not capture any mediators. These results imply that if the factor
model captures the distribution of the assigned causes then the
substitute confounder renders the assignment ignorable. Moreover, such
a factor model always exists.

We then discuss a collection of identification results around the
deconfounder. Under \gls{SUTVA} and single ignorability, we prove that
the deconfounder identifies the average causal effects and the
conditional potential outcomes under different conditions.

\subsection{Factor models and the substitute confounder}

To study the deconfounder, we first connect ignorability to factor
models. Recall the definitions of ignorability and factor model.

Ignorability assumes that the assigned causes are conditionally
independent of the potential outcomes \citep{rosenbaum1983central}:

\begin{defn}{(Ignorability)}
  \label{def:SI}
  Assigned causes are ignorable given $Z_i$ if
  \begin{align}
    \label{eq:SI-technical}
    (A_{i1}, \ldots, A_{im}) \independent
    Y_i(\mba) \g Z_i
  \end{align}
  for all
  $(a_{1}, \ldots, a_{m})\in \mathcal{A}_1\otimes \cdots
  \otimes\mathcal{A}_m, $, and $i = 1, \ldots, n$.
\end{defn}

Roughly, the assigned causes are ignorable given $Z_i$ if all
confounders are captured by $Z_i$.  More precisely, the assigned
causes are ignorable if all confounders are measurable with
respect to the $\sigma$-algebra generated by $Z_i$.

A factor model of assigned causes describes each assigned cause of a
individual with a latent variable specific to this individual and another specific
to this cause:

\begin{defn}{(Factor model of assigned causes)}
\label{defn:factormodels}
Consider the assigned causes $\mbA_{1:n}$, a set of latent variables
$Z_{1:n}$ and a set of parameters $\theta_{1:m}$.  A factor model of
the assigned causes is a latent-variable model,
\begin{align}
  \label{eq:probmodel_confoundjoint}
  p(z_{1:n}, \mba_{1:n}\,;\, \theta_{1:m}) =
  p(z_{1:n})
  \prod_{i=1}^{n} \prod_{j=1}^{m} p(a_{ij} \g z_i, \theta_j).
\end{align}
The distribution of assigned causes is the corresponding marginal,
\begin{align}
  \label{eq:probmodel_confound}
  p(\mba_{1:n}) =
  \int
  p(z_{1:n}, \mba_{1:n}\,;\, \theta_{1:m})
  \dif z_{1:n}.
\end{align}
\end{defn}

In a factor model, each latent variable $Z_i$ of individual $i$ renders its
assigned causes $A_{ij}, j=1, \ldots, m,$ conditionally independent.
Each cause is accompanied with an unknown parameter $\theta_j$. As we
mentioned in \Cref{subsec:assignment-model}, many common models from
Bayesian statistics and machine learning can be written as factor
models.

To connect ignorability to factor models, consider an intermediate
construct, the ``Kallenberg construction.''  The Kallenberg
construction is inspired by the idea of randomization variables,
Uniform[0,1] variables from which we can construct a random variable
with an arbitrary distribution \citep{kallenbergfoundations}.  The
Kallenberg construction of assigned causes will bridge the conditional
independence statement in \Cref{eq:SI-technical} with the factor
models of the deconfounder.

\begin{defn}{(Kallenberg construction of assigned causes)}
\label{def:kallenberg}
Consider a random variable $Z_i$ taking values in $\mathcal{Z}$.  The
distribution of assigned causes $(A_{i1}, \ldots, A_{im})$ admits a
Kallenberg construction if there exists (deterministic) measurable
functions, $f_j:\mathcal{Z}\times [0,1]\rightarrow \mathcal{A}_j$ and
random variables $U_{ij} \in [0,1]$ ($j=1, \ldots, m$) such that
\begin{align}
  \label{eq:funcstrongignr}
  A_{ij} \stackrel{a.s.}{=} f_j(Z_i, U_{ij}).
\end{align}
The variables $U_{ij}$ must marginally follow Uniform[0,1] and jointly
satisfy
\begin{align}
  (U_{i1}, \ldots, U_{im}) \independent (Z_i, Y_i(a_{1}, \ldots, a_{m}))
\end{align}
for all $(a_{1}, \ldots, a_{m}) \in \mathcal{A}_1\otimes \cdots \otimes
\mathcal{A}_m$.
\end{defn}

Using these definitions, the first lemma relates ignorability
to the Kallenberg construction.

\begin{lemma}{(Kallenberg construction $\Leftrightarrow$ strong
    ignorability)}
\label{lemma:strong_ignorability_functional}
The assigned causes are ignorable given a random variable
$Z_i$ if and only if the distribution of the assigned causes
$(A_{i1}, \ldots, A_{im})$ admits a Kallenberg construction from $Z_i$.
\end{lemma}

What \Cref{lemma:strong_ignorability_functional} says is that if the
distribution of the assigned causes has a Kallenberg construction from
a random variable $Z_i$ then $Z_i$ is a valid substitute confounder:
it renders the causes ignorable. Moreover, a valid substitute
confounder must always come from a Kallenberg construction.

We next relate the Kallenberg construction to factor models. We show
that factor models admit a Kallenberg construction.  This fact
suggests the deconfounder: if we fit a factor model to capture the
distribution of assigned causes then we can use the fitted factor
model to construct a substitute confounder.

\begin{lemma}{(factor models $\Rightarrow$ Kallenberg construction)}
    \label{lemma:factormodel} Under weak regularity conditions and
    single ignorability, every factor model of the assigned causes
    $p(\theta_{1:m}, z_{1:n}, \mba_{1:n})$ admits a Kallenberg
    construction from $Z_i$.
\end{lemma}

\Cref{lemma:strong_ignorability_functional} and
\Cref{lemma:factormodel} connect ignorability to Kallenberg
constructions and then Kallenberg constructions to factor models. The
two lemmas together connect factor models to ignorability. These
connections enable the deconfounder: they explain how the distribution
of assigned causes relates to the substitute confounder $Z$ in a
Kallenberg construction. They justify why we can take a set of
assigned causes and do inference on $Z$ via factor models.

Next we establish two properties of the substitute confounder. We
assume the substitute confounder comes from a factor model that
captures the population distribution of the causes.

The first property is that the substitute confounder must capture all
multi-cause confounders. It implies that the inferred substitute
confounder, together with all single-cause confounders (if there is
any), deconfounds causal inference.

\begin{lemma}
\label{prop:all_confounder} Any multi-cause confounder $C_i$ must be
measurable with respect to the $\sigma$-algebra generated by the
substitute confounder $Z_i$.
\end{lemma}

A multi-cause confounder is a confounder that confounds two or more
causes. (Its technical definition stems from Definition 4 of
\citet{vanderweele2013definition}; see \Cref{sec:prop3proof}.)
\Cref{fig:graphical-arg} gives the intuition with a graphical model
and \Cref{sec:prop3proof} gives a detailed proof.

\Cref{prop:all_confounder} shows that the deconfounder captures
unobserved confounders.  But might the inferred substitute confounder
pick up a mediator? If the substitute confounder also picks up a
mediator then conditioning on it will yield conservative causal
estimates \citep{baron1986moderator, imai2010identification}. The next
proposition alleviates this concern.

\begin{lemma}
\label{prop:no_mediator} Any mediator is almost surely not measurable
with respect to the $\sigma$-algebra generated by the substitute
confounder $Z_i$ and the pre-treatment observed covariates $X_i$.
\end{lemma}

\Cref{prop:no_mediator} implies that the substitute confounder does
not pick up mediators, variables along the path between causes and
effects. This property greenlights us for treating the inferred
substitute confounder as a pre-treatment covariate.

\Cref{prop:all_confounder} and \Cref{prop:no_mediator} qualify the
substitute confounder for mimicking confounders. We condition on the
substitute confounder and proceed with causal inference.

These lemmas lead to justifications of the deconfounder algorithm. We
first describe their implications on the substitute confounders and
factor models.

\begin{prop} (Substitute confounders and factor models) Under weak
regularity conditions,
\label{prop:main1}
\begin{enumerate}
\item Under single ignorability, the assigned causes are ignorable
given the substitute confounder $Z_i$ and the pre-treatment covariates
$X_i$ if the true distribution $p(\mba_{1:n})$ can be written as a
factor model that uses the substitute confounder, $p(z_{1:n},
\mba_{1:n} \g \theta_{1:m})$.

\item There always exists a factor model that captures the
  distribution of assigned causes.
\end{enumerate}
\end{prop}

\begin{proofsk} The first part follows from
\Cref{lemma:strong_ignorability_functional,lemma:factormodel}. The
second part follows from the Reichenbach's common cause
principle~\citep{Peters2017, sober1976simplicity} and Sklar's
theorem~\citep{sklar1959fonctions}: any multivariate joint
distribution can be factorized into the product of univariate marginal
distributions and a copula which describes the dependence structure
between the variables. The full proof is in
\Cref{sec:factormodelproof}.
\end{proofsk}

\Cref{prop:main1} justifies the use of factor models in the
deconfounder. The first part of \Cref{prop:main1} suggests how to find
a valid substitute confounder, one that renders the causes strongly
ignorable. Two conditions suffice: (1) the substitute confounder comes
from a factor model; (2) the factor model captures the population
distribution of the assigned causes. The assignment model in the
deconfounder stems from this result: fit a factor model to the
assigned causes, check that it captures their population distribution,
and finally use the fitted factor model to infer a substitute
confounder. The first part of the theorem says that the deconfounder
does deconfound. The second part ensures that there is hope to find a
deconfounding factor model. There always exists a factor model that
captures the population distribution of the assigned causes.

\subsection{Causal identification of the deconfounder}

Building on the characterizations of the substitute confounder
(\Cref{lemma:factormodel,lemma:strong_ignorability_functional,prop:all_confounder,prop:no_mediator}),
we discuss a collection of causal identification results around the
deconfounder. We prove that the deconfounder can identify three causal
quantities under suitable conditions.\footnote{Here ``identify'' means
the causal quantity can be written as a function of the observed data.
Moreover, the deconfounder can unbiasedly estimate it.} These causal
quantities include the average causal effect of all the causes, the
average causal effect of subsets of the causes, and the conditional
potential outcome.

Before stating the identification results, we first describe the
notion of a \emph{consistent} substitute confounder; we will rely on
this notion for identification.

\begin{defn}{(Consistency of substitute confounders)}
  \label{def:consistsub}
The factor model $p(\theta, z, \mba)$ admits
consistent estimates of the substitute confounder $Z_i$ if, for some
function $f_\theta$,
\begin{align}
p(z_i\g \mba_i, \theta) = \delta_{f_\theta(\mba_i)}.
\end{align} 
\end{defn}

Consistency of substitute confounders requires that we can estimate
the substitute confounder $Z_i$ from the causes $\mbA_i$ with
certainty; it is a deterministic function of the causes. Nevertheless,
the substitute confounder need not coincide with the true
data-generating $Z_i$; nor does it need to coincide with the true
unobserved confounder. We only need to estimate the substitute
confounder $Z_i$ up to some deterministic bijective transformations
(e.g. scaling and linear transformations).

Many factor models admit consistent substitute confounder estimates
when the number of causes is large. For example, probabilistic PCA and
Poisson factorization lead to consistent $Z_i$ as $(n+m)\cdot \log(nm)
/ (nm)\rightarrow 0$, where $n$ is the number of individuals and $m$
is the number of causes \citep{chen2017structured}. Many studies also
involve many causes, e.g. the \gls{GWAS} study in
\Cref{subsec:gwasstudy} and the movie-actor study in
\Cref{subsec:moviereal}.

We now describe three identification results under \gls{SUTVA}, single
ignorability, and consistency of substitute confounders. We first
study the average causal effect of all the causes.

\begin{thm} (Identification of the average causal effect of all the
causes)
\label{thm:deconfounderfactor}
Assume \gls{SUTVA}, single ignorability, and consistency of substitute
confounders. Then, under conditions described below, the deconfounder
non-parametrically identifies the average causal effect of all the
causes.  The average causal effect of changing the causes from $\mba =
(a_1,
\ldots, a_m)$ to $\mba'=(a'_1, \ldots, a'_m)$ is
\begin{align}
\label{eq:identifyate}
  \E{Y}{Y_i(\mba)} &- \E{Y}{Y_i(\mba')}
                     =\E{Z, X}{\E{Y}{Y_i\g \mbA_i=\mba, Z_i, X_i}} - \E{Z, X}{\E{Y}{Y_i\g \mbA_i=\mba', Z_i, X_i}}.
\end{align}
This holds with the following two conditions: (1) The substitute
confounder is a piece-wise constant function of the (continuous)
causes: $\nabla_{\mba} f_\theta(\mba) = 0$ up to a
set of Lebesgue measure zero; (2) The outcome is separable,
\begin{align*}
  \E{}{Y_i(\mba)\g Z_i=z, X_i=x}
  &= f_1(\mba, x) + f_2(z), \\
  \E{}{Y_i\g \mbA_i=\mba, Z_i=z, X_i=x}
  &= f_3(\mba, x) + f_4(z),
\end{align*}
for all
$(\mba, x, z)\in \mathcal{A}\times\mathcal{X}\times\mathcal{Z}$
and some continuously differentiable\footnote{For binary causes, we
can analogously assume that there exists $\mba_{\mathrm{new}}$ and
$\mba'_{\mathrm{new}}$ such that $\mba_{\mathrm{new}} -
\mba'_{\mathrm{new}} = \mba - \mba'$ and they lead to the same
substitute confounder estimate $f(\mba_{\mathrm{new}}) =
f(\mba'_{\mathrm{new}})$. Further, the outcome model is separable:
$\E{}{Y_i(\mba)-Y_i(\mba')\g Z_i=z, X_i=x} = f_1(\mba-\mba', x) +
f_2(z).$} functions $f_1$, $f_2$, $f_3$,
and $f_4$.\footnote{The
expectation over $Z_i$ and $X_i$ is taken over $P(Z_i, X_i)$ in
\Cref{eq:identifyate}: $\E{Z_i, X_i}{\E{Y}{Y_i\g \mbA_i=\mba, Z_i,
X_i}} = \int \E{Y}{Y_i\g \mbA_i=\mba, Z_i, X_i} P(Z_i, X_i)\dif Z_i\dif
X_i$.}
\end{thm}

\begin{proofsk} \Cref{thm:deconfounderfactor} rely on two results: (1)
Single ignorability and \Cref{prop:all_confounder} ensure $(Z_i, X_i)$
capture all confounders; (2) The pre-treatment nature of $X_i$ and
\Cref{prop:no_mediator} ensure $(Z_i, X_i)$ capture no mediators.
These results assert ignorability given the substitute confounder
$Z_i$ and the observed covariates $X_i$. They greenlight us for causal
inference given consistency of substitute confounder estimates.
\Cref{thm:deconfounderfactor} then leverages two additional conditions
to identify average causal effects without assuming overlap. The full
proof is in \Cref{sec:ateallidentifyproof}.
\end{proofsk}

\Cref{thm:deconfounderfactor} shows that the deconfounder can
unbiasedly estimate the average causal effect of all the causes. It
requires two conditions beyond single ignorability,
\gls{SUTVA}, and consistency of substitute confounders. The first
condition requires that the substitute confounder be a piece-wise
constant function; it is satisfied when the substitute confounder is
discrete and the causes are continuous. The second condition requires
that the potential outcome be separable in the substitute confounder
and the causes; the observed data also respects this separability.
This condition is satisfied when the substitute confounder does not
interact with the causes. For example, this condition is often
satisfied in \gls{GWAS} studies: the effect of \gls{SNP}s on a trait
does not depend on an individual's ancestry.

When the separability condition of \Cref{thm:deconfounderfactor} does
not hold, we can still use the deconfounder to handle the unobserved
multi-cause confounders that do not interact with the causes. As long
as the observed covariates include those that do interact with the
causes, the deconfounder produces unbiased estimates of the average
causal effect.

We next discuss the identification of the average causal effect for
subsets of the causes.

\begin{thm} (Identification of the average causal effect of subsets of
the causes)
\label{thm:atesubsetidentify}
Assume \gls{SUTVA}, single ignorability, and consistency of substitute
confounders. Then, under the condition described below, the
deconfounder non-parametrically identifies the average causal effect
of subsets of causes. The average causal effect of changing the first
$k~(k<m)$ causes from $a_{1:k} = (a_1, \ldots, a_k)$ to
$a'_{1:k}=(a'_1, \ldots, a'_k)$ is
\begin{align*}
\label{eq:subset-consistent}
  &\E{A_{(k+1):m}}{\E{Y}{Y_i(a_{1:k}, A_{i,(k+1):m})}} -  \E{A_{(k+1):m}}{\E{Y}{Y_i(a'_{1:k}, A_{i,(k+1):m})}}\\
=&\E{Z,X}{\E{Y}{Y_i\g Z_i, X_i, A_{i,1:k}=a_{1:k}}} - \E{Z,X}{\E{Y}{Y_i\g Z_i, X_i, A_{i,1:k}=a'_{1:k}}}.
\end{align*}
This holds with the following condition: The first $k$ causes $A_{i1},
\ldots, A_{ik}$ satisfy overlap, $P((A_{i1}, \ldots, A_{ik})\in
\mathcal{A}\g Z_i, X_i) > 0$ for any set $\mathcal{A}$ such that
$P(\mathcal{A}) > 0$.\footnote{In full notation,
  $\E{A_{(k+1):m}}{\E{Y}{Y_i(a_{1:k}, A_{i,(k+1):m})}} =
  \E{A_{(k+1):m}}{\E{Y}{Y_i(a_1, \ldots, a_k, A_{ik+1}, \ldots,
      A_{im})}}$.}
\end{thm}

\begin{proofsk}
Similar to \Cref{thm:deconfounderfactor}, \Cref{thm:atesubsetidentify}
uses \Cref{prop:all_confounder} and \Cref{prop:no_mediator} to
greenlight the use of a substitute confounder. It then relies on
overlap to identify the average causal effect; we follow the classical
argument that identifies the average treatment
effect~\citep{Imbens:2015}. The full proof is in
\Cref{sec:atesubsetidentifyproof}.
\end{proofsk}

\Cref{thm:atesubsetidentify} shows that the deconfounder can
unbiasedly estimate the average causal effect of subsets of the
causes. It lets us answer ``how would the movie revenue change, on
average, if we place Meryl Streep and Sean Connery into a movie?''
Beyond single ignorability, \gls{SUTVA}, and consistency of substitute
confounders, \Cref{thm:atesubsetidentify} requires overlap.  Overlap
ensures that $\E{Y}{Y_i\g Z_i, X_i, A_{i,1:k}=a_{1:k}}$ is estimable
from the observed data for all possible values of $(Z_i, X_i,
A_{i,1:k})$. The overlap assumption about the causes in
\Cref{thm:atesubsetidentify} replaces the separability assumption
about the outcome model required by \Cref{thm:deconfounderfactor}.

We note that the overlap condition and the consistency of substitute
confounders are compatible. Though consistency requires $P(Z_i\g
\mbA_i)=\delta_{f_\theta(\mbA_i)}$, it is still possible for subsets
of the causes to satisfy overlap; the consistency condition only
prevents the complete set of $m$ causes from satisfying overlap. For
example, consider a consistent estimate of the substitute confounder
that is one-dimensional, $Z_i = \sum_{j=1}^m\alpha_jA_{ij}$. Any
$k\leq m-1$ causes satisfy overlap, but the complete set of $m$ causes
do not.

Finally, we discuss the identification of the conditional mean
potential outcome.

\begin{thm} (Identification of the conditional mean potential outcome)
\label{thm:conditionalpoidentify}
Assume \gls{SUTVA}, single ignorability, and consistency of substitute
confounders. Then, under the condition described below, the
deconfounder non-parametrically identifies the mean potential outcome
of an individual given its current assigned causes.  If an individual
is assigned with $\mba = (a_1, \ldots, a_m)$, then its potential
outcome under a different assignment $\mba'=(a'_1, \ldots, a'_m)$ is
\begin{align*}
  \E{Y}{Y_i(\mba')\g \mbA_i=\mba} = \E{Z,X}{\E{Y}{Y_i\g  Z_i, X_i, \mbA_i=\mba}}.
\end{align*}
This holds with the following condition: The cause assignment of
interest $\mba'$ leads to the same substitute confounder estimate as
the observed assigned causes: $P(Z_i\g \mbA_i=\mba) = P(Z_i\g
\mbA_i=\mba')$.
\end{thm}

\begin{proofsk} As with \Cref{thm:deconfounderfactor} and
\Cref{thm:atesubsetidentify}, \Cref{thm:conditionalpoidentify} relies
on the ignorability given the substitute confounders $Z_i$ and the
observed covariates $X_i$ due to \Cref{prop:all_confounder} and
\Cref{prop:no_mediator}. It then identifies the potential outcome by
focusing on the data points with the same substitute confounder
estimate. We note that this identification result does not require
overlap. The full proof is in \Cref{sec:conditionalpoidentifyproof}.
\end{proofsk}

Given consistency of substitute confounders,
\Cref{thm:conditionalpoidentify} nonparametrically identifies the mean
potential outcome of an individual $Y_i(\mba')$ given its current
assigned causes $\mbA_i=\mba$. The only requirement is about the
configurations of cause assignments we can query, $\mba'$; these
configurations should lead to the same substitute confounder estimate
as the current assigned causes.

We illustrate this condition with actors causing movie revenue.  For
simplicity, assume the substitute confounder captures the genre of
each movie. Start with one of the James Bond movie; it is a spy film.
We can ask what its revenue would be if we make its cast to be that of
``The Bourne Trilogy'' (also a spy film). Alternatively, we can query
what if we make its cast to include some actors from ``The Bourne
Trilogy'' and other actors from ``North By Northwest''; both are spy
films.  However, we can not query what if we make its cast to be that
of ``The Shawshank Redemption'' (which is not a spy film).

\Cref{thm:deconfounderfactor,thm:atesubsetidentify,thm:conditionalpoidentify}
confirm the validity of the deconfounder by providing three sets of
nonparametric identification results. When the assumptions in
\Cref{thm:deconfounderfactor,thm:atesubsetidentify,thm:conditionalpoidentify}
do not hold, we recommend evaluating the uncertainty of the
deconfounder estimate. \Cref{subsubsec:uncertainty} discusses how;
\Cref{subsec:smoking} gives an example. The posterior distribution of
the deconfounder estimate reflects how the (finite) observed data
informs causal quantities of interest. When the causal quantity is
non-identifiable, the posterior distribution of the deconfounder
estimate will reflect this non-identifiability. For example, if the
causal quantity is non-identifiable over $\mathcal{R}$, the posterior
distribution of the deconfounder estimate will be uniform over
$\mathcal{R}$ (with non-informative priors).

We finally remark that the identification results in
\Cref{thm:deconfounderfactor,thm:atesubsetidentify,thm:conditionalpoidentify}
do not contradict the negative results of \citet{d2019multi}.
\citet{d2019multi} explore nonparametric non-identification of a
particular causal quantity, the mean potential outcome
$\E{}{Y_i(\mba)}$.  In this paper,
\Cref{thm:deconfounderfactor,thm:atesubsetidentify,thm:conditionalpoidentify}
establish the nonparametric identification of different causal
quantities. \citet{d2019multi} do not make the same assumptions as in
\Cref{thm:deconfounderfactor,thm:atesubsetidentify,thm:conditionalpoidentify}.
More specifically, under consistency of substitute confounders and
other suitable conditions, \Cref{thm:deconfounderfactor} shows that
the average causal effect of all the causes $\E{}{Y_i(\mba)} -
\E{}{Y_i(\mba')}$ is nonparametrically identifiable;
\Cref{thm:atesubsetidentify} shows that the average causal effect of
subsets of the causes $\E{A_{(k+1):m}}{\E{Y}{Y_i(a_{1:k},
A_{i,(k+1):m})}} -  \E{A_{(k+1):m}}{\E{Y}{Y_i(a'_{1:k},
A_{i,(k+1):m})}}$ is nonparametrically identifiable;
\Cref{thm:conditionalpoidentify} shows that the conditional mean
potential outcome $\E{}{Y_i(\mba')\g \mbA_i=\mba}$ is nonparametrically
identifiable.

%%% Local Variables:
%%% mode: latex
%%% TeX-master: "latent_confounder"
%%% End:

% !TEX root = latent_confounder.tex
\section{Discussion}
\label{sec:discussion}

Classical causal inference studies how a univariate cause affects an
outcome. Here we studied \emph{multiple causal inference}, where there
are multiple causes that contribute to the effect.  Multiple causes
might at first appear to be a curse, but we showed that it can be a
blessing.  Multiple causal inference liberates us from ignorability,
providing causal inference from observational data under weaker
assumptions than the classical approach requires.

We developed the \emph{deconfounder}: first fit a good factor model of
assigned causes; then use the factor model to infer a substitute
confounder; finally perform causal inference.  We showed how a
substitute confounder from a good factor model must capture all
multi-cause confounders, and we demonstrated that whether a factor
model is satisfactory is a checkable proposition.

There are many directions for future work.

Here we estimate the potential outcomes under \emph{all}
configurations of the causes.  Which of these potential outcomes can
be reliably estimated? Can we optimally trade off confounding bias and
estimation variance?

Here we checked factor models for downstream causal unbiasedness.  But
model checking is still an imprecise science. Can we develop rigorous
model checking algorithms for causal inference?

Here we focused on estimation. Can we develop a testing counterpart?
How can we identify significant causes while still preserving
family-wise error rate or false discovery rate?

Here we analyzed univariate outcomes. Can we work with both multiple
causes and multiple outcomes.  Can dependence among outcomes further
help causal inference?

%%% Local Variables:
%%% mode: latex
%%% TeX-master: "latent_confounder"
%%% End:

\vspace{20pt}
\clearpage

\parhead{Acknowledgments. } We have had many useful discussions about
the previous versions of this manuscript. We thank Edo Airoldi, Elias
Barenboim, L\'{e}on Bottou, Alexander D'Amour, Barbara Engelhart,
Andrew Gelman, David Heckerman, Jennifer Hill, Ferenc Husz\'{a}r,
George Hripcsak, Daniel Hsu, Guido Imbens, Thorsten Joachims, Fan Li,
Lydia Liu, Jackson Loper, David Madigan, Suresh Naidu, Xinkun Nie,
Elizabeth Ogburn, Georgia Papadogeorgou, Judea Pearl, Alex
Peysakhovich, Rajesh Ranganath, Jason Roy, Cosma Shalizi, Dylan Small,
Hal Stern, Amos Storkey, Wesley Tansey, Eric Tchetgen Tchetgen, Dustin
Tran, Victor Veitch, Stefan Wager, Kilian Weinberger, Jeannette Wing,
Linying Zhang, Qingyuan Zhao, and Jos\'{e} Zubizarreta.

\clearpage
\putbib[BIB1]
\end{bibunit}

\clearpage
\begin{bibunit}[apa]
{\onecolumn
% !TEX root = latent_confounder.tex
\appendix
\onecolumn
{\Large\textbf{Appendix}}

\section{Detailed Results of the GWAS Study}

In this section, we present tables of results from the \gls{GWAS}
study in \Cref{subsec:gwasstudy}.

\Cref{tab:highSNRBN,tab:highSNRTGP,tab:highSNRHGDP,tab:highSNRPSD,tab:highSNRspatial}
contain the result under the high \gls{SNR} setting.

% !TEX root = latent_confounder.tex

\begin{table*}[htp]

  \begin{center}
    \begin{tabular}{lccc}
     \toprule
       &  & {\textbf{Real-valued outcome}} & {\textbf{Binary outcome}} \\
       &Pred. check&\gls{RMSE}$\times 10^{-2}$ & \gls{RMSE}$\times 10^{-2}$ \\
          \midrule
          No control &---&49.66&39.39\\
          Control for confounders$^*$  &---&40.27&31.09\\
          \cdashlinelr{1-4}
          (G)LMM &---&46.22&37.81\\
          PPCA &0.13&46.05&36.01\\
          PF &0.15&44.58& 36.30\\
          LFA &0.14&43.02&36.65\\
          GMM &0.01&47.33&40.24\\
          DEF &0.18&\bfseries{41.05}&\bfseries{33.88}\\        
      \bottomrule
    \end{tabular}
    \caption{GWAS high-\gls{SNR} simulation I: Balding-Nichols Model.
    (``Control for all confounders'' means including the unobserved
    confounders as covariates.) The deconfounder outperforms (G)LMM;
    DEF performs the best among the five factor models. Predictive
    checking offers a good indication of when the deconfounder fails.
    \label{tab:highSNRBN}}
  \end{center}
\end{table*}

\begin{table*}[t]

  \begin{center}
    \begin{tabular}{lccc} 
     \toprule
       &  & {\textbf{Real-valued outcome}} & {\textbf{Binary outcome}} \\
       &Pred. check&\gls{RMSE}$\times 10^{-2}$ & \gls{RMSE}$\times 10^{-2}$ \\
          \midrule
          No control &---&68.78&38.16\\
          Control for confounders$^*$  &---&60.29&32.76\\
          \cdashlinelr{1-4}
          (G)LMM &---&65.25&35.41\\
          PPCA &0.15&65.98&36.11\\
          PF &0.17&64.25&34.79\\
          LFA &0.17&64.00&37.08\\
          GMM &0.02&67.23&35.40\\
          DEF &0.20&\bfseries{63.73}&\bfseries{33.71}\\        
      \bottomrule
    \end{tabular}
    \caption{GWAS high-\gls{SNR} simulation II: 1000 Genomes Project
    (TGP). (``Control for all confounders'' means including the
    unobserved confounders as covariates.) The deconfounder
    outperforms (G)LMM; DEF performs the best among the five factor
    models. Predictive checking offers a good indication of when the
    deconfounder fails.    \label{tab:highSNRTGP}}
  \end{center}
\end{table*}

\begin{table*}[t]

  \begin{center}
    \begin{tabular}{lccc} 
     \toprule
       &  & {\textbf{Real-valued outcome}} & {\textbf{Binary outcome}} \\
       &Pred. check&\gls{RMSE}$\times 10^{-2}$ & \gls{RMSE}$\times 10^{-2}$ \\
          \midrule
          No control &---&77.35&45.93\\
          Control for confounders$^*$  &---&67.53&39.43\\
          \cdashlinelr{1-4}
          (G)LMM &---&74.38&42.79\\
          PPCA &0.14&74.45&43.27\\
          PF &0.14&71.40&42.75\\
          LFA &0.13&72.11&42.34\\
          GMM &0.03&76.27&46.88\\
          DEF &0.16&\bfseries{69.86}&\bfseries{41.61}\\        
      \bottomrule
    \end{tabular}
    \caption{GWAS high-\gls{SNR} simulation III: Human Genome
    Diversity Project (HGDP). (``Control for confounders'' means
    including the unobserved confounders as covariates.)  The
    deconfounder outperforms (G)LMM; DEF performs the best among the
    five factor models. Predictive checking offers a good indication
    of when the deconfounder fails.
    \label{tab:highSNRHGDP}}
  \end{center}
\end{table*}

\begin{table*}[t]

  \begin{center}
    \begin{tabular}{rlllll} 
      \toprule
       &  && {\textbf{Real-valued outcome}} & {\textbf{Binary outcome}} \\
       &&Pred. check&\gls{RMSE}$\times 10^{-2}$ & \gls{RMSE}$\times 10^{-2}$ \\
          \midrule
          $\alpha = 0.01$&No control &---&40.68&30.37\\
          &Control for confounders$^*$  &---&34.35&28.21\\
          \cdashlinelr{1-5}
          &(G)LMM &---&39.09&28.36\\
          &PPCA &0.15&38.14&28.97\\
          &PF &0.16&\bfseries{34.77}&28.67\\
          &LFA &0.16&35.87&28.33\\
          &GMM &0.02&38.15&29.69\\
          &DEF &0.18&34.84&\bfseries{28.04}\\        
          \midrule
          $\alpha = 0.1$&No control &---&43.87&36.77\\
         &Control for confounders$^*$  &---&37.62&33.89\\
          \cdashlinelr{1-5}
          &(G)LMM &---&39.97&35.76\\
          &PPCA &0.21&39.60&35.61\\
          &PF &0.19&38.95&\bfseries{34.28}\\
          &LFA &0.18&39.28&34.73\\
          &GMM &0.00&44.38&36.44\\
          &DEF &0.20&\bfseries{38.75}&34.85\\        
          \midrule
          $\alpha = 0.5$&No control &---&47.38&41.84\\
          &Control for confounders$^*$  &---&43.63&39.86\\
          \cdashlinelr{1-5}
          &(G)LMM &---&47.28&42.91\\
          &PPCA &0.14&46.90&41.41\\
          &PF &0.16&43.29&40.69\\
          &LFA &0.17&43.60&40.77\\
          &GMM &0.02&46.95&42.47\\
          &DEF &0.18&\bfseries{43.09}&\bfseries{40.03}\\        
          \midrule
          $\alpha = 1.0$&No control &---&53.94&49.32\\
          &Control for confounders$^*$  &---&47.12&45.96\\
          \cdashlinelr{1-5}
          &(G)LMM &---&49.21&48.96\\
          &PPCA &0.21&50.57&47.58\\
          &PF &0.19&48.07&46.16\\
          &LFA &0.17&49.27&46.16\\
          &GMM &0.02&52.28&50.31\\
          &DEF &0.23&\bfseries{47.82}&\bfseries{45.62}\\        
        \bottomrule
    \end{tabular}
    \caption{GWAS high-\gls{SNR} simulation IV:
    Pritchard-Stephens-Donnelly (PSD). (``Control for confounders''
    means including the unobserved confounders as covariates.) The
    deconfounder outperforms (G)LMM; DEF often performs the best among
    the five factor models. Predictive checking offers a good
    indication of when the deconfounder fails.
    \label{tab:highSNRPSD}}
  \end{center}
\end{table*}

\begin{table*}[t]

  \begin{center}
    \begin{tabular}{rlllll} 
      \toprule
       &  && {\textbf{Real-valued outcome}} & {\textbf{Binary outcome}} \\
       &&Pred. check&\gls{RMSE}$\times 10^{-2}$ & \gls{RMSE}$\times 10^{-2}$ \\
          \midrule
          $\tau = 0.1$&No control &---&47.47&45.16\\
          &Control for confounders$^*$  &---&44.22&43.85\\
          \cdashlinelr{1-5}
          &(G)LMM &---&47.35&44.15\\
          &PPCA &0.08&47.61&44.36\\
          &PF &0.09&47.13&43.69\\
          &LFA &0.09&47.16&43.87\\
          &GMM &0.01&47.55&45.95\\
          &DEF &0.10&\bfseries{46.95}&\bfseries{43.62}\\        
          \midrule
          $\tau = 0.25$&No control &---&44.68&41.10\\
          &Control for confounders$^*$  &---&41.23&39.65\\
          \cdashlinelr{1-5}
          &(G)LMM &---&43.42&\bfseries{40.67}\\
          &PPCA &0.11&\bfseries{43.26}&41.28\\
          &PF &0.12&43.30&41.10\\
          &LFA &0.13&43.62&41.65\\
          &GMM &0.01&44.81&41.02\\
          &DEF &0.13&43.35&40.97\\        
          \midrule
          $\tau = 0.5$&No control &---&45.18&40.92\\
          &Control for confounders$^*$  &---&41.33&37.35\\
          \cdashlinelr{1-5}
          &(G)LMM &---&44.83&40.59\\
          &PPCA &0.10&43.78&\bfseries{39.99}\\
          &PF &0.09&43.65&40.23\\
          &LFA &0.10&43.88&40.04\\
          &GMM &0.01&46.08&40.76\\
          &DEF &0.12&\bfseries{43.57}&40.02\\        
          \midrule
          $\tau = 1.0$&No control &---&56.57&57.70\\
          &Control for confounders$^*$  &---&52.98&55.46\\
          \cdashlinelr{1-5}
          &(G)LMM &---&56.44&56.33\\
          &PPCA &0.14&55.18&57.36\\
          &PF &0.12&55.29&56.31\\
          &LFA &0.13&\bfseries{54.75}&56.66\\
          &GMM &0.01&57.15&57.55\\
          &DEF &0.12&55.07&\bfseries{56.22}\\        
        \bottomrule
    \end{tabular}
    \caption{GWAS high-\gls{SNR} simulation V: Spatial model.
    (``Control for confounders'' means including the unobserved
    confounders as covariates.) The deconfounder often outperforms
    (G)LMM. Predictive checking offers a good indication of when the
    deconfounder fails: GMM poorly captures the SNPs; it can amplify
    the error in causal estimates.
    \label{tab:highSNRspatial}}
  \end{center}
\end{table*}

%%% Local Variables:
%%% mode: latex
%%% TeX-master: "latent_confounder"
%%% End:

\Cref{tab:BN,tab:TGP,tab:HGDP,tab:PSD,tab:spatial} contain the result
under the low \gls{SNR} setting.

% !TEX root = latent_confounder.tex

\begin{table*}[htp]

  \begin{center}
    \begin{tabular}{lccc} 
     \toprule
       &  & {\textbf{Real-valued outcome}} & {\textbf{Binary outcome}} \\
       &Pred. check&\gls{RMSE}$\times 10^{-2}$ & \gls{RMSE}$\times 10^{-2}$ \\
          \midrule
          No control &---&6.55&5.75\\
          Control for confounders$^*$  &---&6.54&5.75\\
          \cdashlinelr{1-4}
          (G)LMM &---&6.54&\bfseries{5.74}\\
          PPCA &0.14&6.52&\bfseries{5.74}\\
          PF &0.16&6.53&\bfseries{5.74}\\
          LFA &0.14&6.54&\bfseries{5.74}\\
          GMM &0.01&6.54&\bfseries{5.74}\\
          DEF &0.19&\bfseries{6.47}&\bfseries{5.74}\\        
      \bottomrule
    \end{tabular}
    \caption{GWAS low-\gls{SNR} simulation I: Balding-Nichols Model.
    (``Control for all confounders'' means including the unobserved
    confounders as covariates.) The deconfounder outperforms LMM; DEF
    performs the best among the five factor models; it also
    outperforms using the (unobserved) confounder information.
    Predictive checking offers a good indication of when the
    deconfounder fails. \label{tab:BN}}
  \end{center}
\end{table*}

\begin{table*}[t]

  \begin{center}
    \begin{tabular}{lccc} 
     \toprule
       &  & {\textbf{Real-valued outcome}} & {\textbf{Binary outcome}} \\
       &Pred. check&\gls{RMSE}$\times 10^{-2}$ & \gls{RMSE}$\times 10^{-2}$ \\
          \midrule
          No control &---&8.31&4.85\\
          Control for confounders$^*$  &---&8.28&4.85\\
          \cdashlinelr{1-4}
          (G)LMM &---&8.29&4.85\\
          PPCA &0.14&8.29&4.85\\
          PF &0.15&8.29&4.85\\
          LFA &0.17&8.26&4.85\\
          GMM &0.02&8.30&4.85\\
          DEF &0.20&\bfseries{8.11}&\bfseries{4.84}\\        
      \bottomrule
    \end{tabular}
    \caption{GWAS low-\gls{SNR} simulation II: 1000 Genomes Project
    (TGP). (``Control for all confounders'' means including the
    unobserved confounders as covariates.) The deconfounder
    outperforms LMM; DEF performs the best among the five factor
    models; it also outperforms using the (unobserved) confounder
    information. Predictive checking offers a good indication of when
    the deconfounder fails.    \label{tab:TGP}}
  \end{center}
\end{table*}

\begin{table*}[t]

  \begin{center}
    \begin{tabular}{lccc} 
     \toprule
       &  & {\textbf{Real-valued outcome}} & {\textbf{Binary outcome}} \\
       &Pred. check&\gls{RMSE}$\times 10^{-2}$ & \gls{RMSE}$\times 10^{-2}$ \\
          \midrule
          No control &---&9.59&5.84\\
          Control for confounders$^*$  &---&9.52&5.84\\
          \cdashlinelr{1-4}
          (G)LMM &---&9.57&5.84\\
          PPCA &0.14&9.55&5.84\\
          PF &0.13&9.56&5.84\\
          LFA &0.14&9.54&5.84\\
          GMM &0.03&9.59&5.84\\
          DEF &0.16&\bfseries{9.47}&\bfseries{5.83}\\        
      \bottomrule
    \end{tabular}
    \caption{GWAS low-\gls{SNR} simulation III: Human Genome Diversity
    Project (HGDP). (``Control for confounders'' means including the
    unobserved confounders as covariates.)  The deconfounder
    outperforms LMM; DEF performs the best among the five factor
    models; it also outperforms using the (unobserved) confounder
    information. Predictive checking offers a good indication of when
    the deconfounder fails.     \label{tab:HGDP}}
  \end{center}
\end{table*}

\begin{table*}[t]

  \begin{center}
    \begin{tabular}{rlllll} 
      \toprule
       &  && {\textbf{Real-valued outcome}} & {\textbf{Binary outcome}} \\
       &&Pred. check&\gls{RMSE}$\times 10^{-2}$ & \gls{RMSE}$\times 10^{-2}$ \\
          \midrule
          $\alpha = 0.01$&No control &---&3.73&3.23\\
          &Control for confounders$^*$  &---&3.71&3.23\\
          \cdashlinelr{1-5}
          &(G)LMM &---&3.71&3.23\\
          &PPCA &0.13&3.64&3.23\\
          &PF &0.16&3.67&3.23\\
          &LFA &0.16&3.66&3.23\\
          &GMM &0.02&3.72&3.23\\
          &DEF &0.18&\bfseries{3.59}&\bfseries{3.22}\\        
          \midrule
          $\alpha = 0.1$&No control &---&4.09&3.84\\
         &Control for confounders$^*$  &---&4.09&3.84\\
          \cdashlinelr{1-5}
           &(G)LMM &---&4.09&3.84\\
          &PPCA &0.20&4.08&3.84\\
          &PF &0.18&4.08&3.84\\
          &LFA &0.18&4.07&3.84\\
          &GMM &0.00&4.09&3.84\\
          &DEF &0.20&\bfseries{4.05}&\bfseries{3.83}\\        
          \midrule
          $\alpha = 0.5$&No control &---&4.82&4.14\\
          &Control for confounders$^*$  &---&4.81&4.14\\
          \cdashlinelr{1-5}
          &(G)LMM &---&4.82&4.14\\
          &PPCA &0.14&4.81&\bfseries{4.13}\\
          &PF &0.17&\bfseries{4.80}&\bfseries{4.13}\\
          &LFA &0.16&4.81&4.14\\
          &GMM &0.03&4.82&4.14\\
          &DEF &0.19&\bfseries{4.80}&\bfseries{4.13}\\        
          \midrule
          $\alpha = 1.0$&No control &---&5.43&4.58\\
          &Control for confounders$^*$  &---&5.38&4.57\\
          \cdashlinelr{1-5}
          &(G)LMM &---&5.40&4.58\\
          &PPCA &0.21&5.38&\bfseries{4.57}\\
          &PF &0.16&5.41&\bfseries{4.57}\\
          &LFA &0.19&5.40&\bfseries{4.57}\\
          &GMM &0.02&5.43&4.58\\
          &DEF &0.24&\bfseries{5.37}&\bfseries{4.57}\\        
        \bottomrule
    \end{tabular}
    \caption{GWAS low-\gls{SNR} simulation IV:
    Pritchard-Stephens-Donnelly (PSD). (``Control for confounders''
    means including the unobserved confounders as covariates.) The
    deconfounder outperforms LMM; DEF performs the best among the five
    factor models; it also outperforms using the (unobserved)
    confounder information. Predictive checking offers a good
    indication of when the deconfounder fails.     \label{tab:PSD}}
  \end{center}
\end{table*}

\begin{table*}[t]

  \begin{center}
    \begin{tabular}{rlllll} 
      \toprule
       &  && {\textbf{Real-valued outcome}} & {\textbf{Binary outcome}} \\
       &&Pred. check&\gls{RMSE}$\times 10^{-2}$ & \gls{RMSE}$\times 10^{-2}$ \\
          \midrule
          $\tau = 0.1$&No control &---&4.66&4.74\\
          &Control for confounders$^*$  &---&4.63&4.73\\
          \cdashlinelr{1-5}
          &(G)LMM &---&4.57&\bfseries{4.73}\\
          &PPCA &0.09&4.62&4.74\\
          &PF &0.08&4.58&4.74\\
          &LFA &0.09&4.54&\bfseries{4.73}\\
          &GMM &0.02&4.70&4.74\\
          &DEF &0.10&\bfseries{4.53}&\bfseries{4.73}\\        
          \midrule
          $\tau = 0.25$&No control &---&4.30&3.81\\
          &Control for confounders$^*$  &---&3.81&3.79\\
          \cdashlinelr{1-5}
          &(G)LMM &---&4.28&\bfseries{3.80}\\
          &PPCA &0.10&4.26&\bfseries{3.80}\\
          &PF &0.12&4.26&\bfseries{3.80}\\
          &LFA &0.12&4.27&\bfseries{3.80}\\
          &GMM &0.01&4.30&3.81\\
          &DEF &0.13&\bfseries{4.25}&\bfseries{3.80}\\        
          \midrule
          $\tau = 0.5$&No control &---&4.30&3.85\\
          &Control for confounders$^*$  &---&3.82&3.83\\
          \cdashlinelr{1-5}
          &(G)LMM &---&4.28&\bfseries{3.83}\\
          &PPCA &0.11&4.27&\bfseries{3.83}\\
          &PF &0.09&4.28&3.84\\
          &LFA &0.11&4.27&3.84\\
          &GMM &0.01&4.29&3.84\\
          &DEF &0.13&\bfseries{4.25}&3.84\\        
          \midrule
          $\tau = 1.0$&No control &---&6.71&5.52\\
          &Control for confounders$^*$  &---&5.43&5.51\\
          \cdashlinelr{1-5}
          &(G)LMM &---&6.70&5.52\\
          &PPCA &0.14&6.70&5.52\\
          &PF &0.12&6.70&5.52\\
          &LFA &0.12&6.69&5.52\\
          &GMM &0.01&6.72&5.53\\
          &DEF &0.13&\bfseries{6.62}&\bfseries{5.51}\\        
        \bottomrule
    \end{tabular}
    \caption{GWAS low-\gls{SNR} simulation V: Spatial model.
    (``Control for confounders'' means including the unobserved
    confounders as covariates.) The deconfounder often outperforms
    LMM; DEF often performs the best among the five factor models.
    Yet, the deconfounder does not outperform using the (unobserved)
    confounder information. Spatially-induced SNPs challenge many
    latent variable models to capture its patterns and fully
    deconfound causal inference. Predictive checking offers a good
    indication of when the deconfounder fails: GMM poorly captures the
    SNPs; it can amplify the error in causal estimates.
    \label{tab:spatial}}
  \end{center}
\end{table*}

%%% Local Variables:
%%% mode: latex
%%% TeX-master: "latent_confounder"
%%% End:

\section{Detailed Results of the Movie Study}
In this section, we present tables of results from the movies study in
\Cref{subsec:moviereal}.

% !TEX root = latent_confounder.tex

\begin{table}[t]
  \begin{center}
    \begin{tabular}{lc} 
Control & Average predictive log-likelihood\\
\midrule
 No Control& \bfseries{-1.1}\\
Control for $X$& \bfseries{-1.1}\\
Control for $\hat{a}_{\gls{PPCA}}$ &-1.2\\
Control for $\hat{a}_{\gls{PF}}$  & -1.2\\
Control for $\hat{a}_{\gls{DEF}}$  & -1.2\\
Control for $(\hat{a}_{\gls{PPCA}}, X)$&-1.3\\
Control for $(\hat{a}_{\gls{PF}}, X)$ &-1.2\\
Control for $(\hat{a}_{\gls{DEF}}, X)$  & -1.2\\
    \end{tabular}
      \caption{Average predictive log-likelihood on a
         holdout set of all movies. ($X$ represents the observed
        covariates.) Causal models (the deconfounder) predicts
        slightly worse than prediction models.
        \label{table:moviepredict}}
  \end{center}
\end{table}

\begin{table}[t]
  \begin{center}
    \begin{tabular}{lc} 
      Control & Average predictive log-likelihood\\
      \midrule
      No Control& -2.5\\
      Control for $X$& -2.1\\
      Control for $\hat{a}_{\gls{PPCA}}$ &-1.6\\
      Control for $\hat{a}_{\gls{PF}}$  & \bfseries{-1.5}\\
      Control for $\hat{a}_{\gls{DEF}}$  & \bfseries{-1.5}\\
      Control for $(\hat{a}_{\gls{PPCA}}, X)$&-1.7\\
      Control for $(\hat{a}_{\gls{PF}}, X)$ &\bfseries{-1.5}\\
      Control for $(\hat{a}_{\gls{DEF}}, X)$  & -1.6\\
    \end{tabular}
    \caption{Average predictive log-likelihood on the holdout set of
      non-English movies. ($X$ represents the observed covariates.)
      On a test set of uncommon movies, causal models with the
      deconfounder predict better than prediction models.}
    \label{table:languagepredict}
  \end{center}
\end{table}

\begin{table}[t]
  \begin{center}
    \begin{tabular}{lc}
      Control & Average predictive log-likelihood\\
      \midrule
      No Control& -2.1\\
      Control for $X$& -1.9\\
      Control for $\hat{a}_{\gls{PPCA}}$ &-1.4\\
      Control for $\hat{a}_{\gls{PF}}$  & \bfseries{-1.2}\\
      Control for $\hat{a}_{\gls{DEF}}$  & -1.3\\
      Control for $(\hat{a}_{\gls{PPCA}}, X)$&-1.4\\
      Control for $(\hat{a}_{\gls{PF}}, X)$ &-1.3\\
      Control for $(\hat{a}_{\gls{DEF}}, X)$  & \bfseries{-1.2}\\
    \end{tabular}
    \caption{Average predictive log-likelihood on the holdout set of
      non-drama/comedy/action movies. ($X$ represents the observed
      covariates.) On a test set of uncommon movies, causal models
      with the deconfounder predict better than prediction models.}
    \label{table:genrepredict}
  \end{center}
\end{table}

%%% Local Variables:
%%% mode: latex
%%% TeX-master: "latent_confounder"
%%% End:

\section{Proof of \Cref{lemma:strong_ignorability_functional}}

\label{sec:thm1proof}

\begin{proofsk}
  First assume the Kallenberg construction in
  \Cref{eq:funcstrongignr}.  This form shows that the assigned causes
  $(A_{i1}, \ldots, A_{im})$ are captured by functions of $Z_i$ and
  randomization variables $U_{ij}$.  This fact, in turn, implies that
  the randomness in $(A_{i1}, \ldots, A_{im}) \g Z_i$ comes from the
  randomization variables which are (by definition) independent of
  $Y_i(\mba)$. Therefore $(A_{i1}, \ldots, A_{im})$ is conditionally
  independent of $Y_i$ given $Z_i$, i.e., ignorability holds. Now
  assume that ignorability holds.  We prove that this assumption
  implies a Kallenberg construction by building on the randomization
  variable construction of conditional
  distributions~\citep{kallenbergfoundations}.  The full proof is in
  \Cref{sec:thm1proof}.
\end{proofsk}

\begin{proof}
For notation simplicity, we suppress the $i$ subscript in this proof.

We assume $\mathcal{Z}$ is a measurable space and $\mathcal{A}_j, j=1, \ldots, m$ are Borel spaces.

We first prove the necessity. Assume that $A_j = f_j(Z, U_j),j=1, \ldots, m,$ where $f_j, j=1, \ldots, m$ are measurable and 
\begin{equation}
\label{eq:noiseindep}
(U_1, \ldots, U_m)\independent (Z, Y(a_1, \ldots, a_m))
\end{equation} 
for all $(a_1, \ldots, a_m)$. By Proposition 5.18 in \citet{kallenbergfoundations}, \Cref{eq:noiseindep} implies 
\[(U_1, \ldots, U_m)\independent_Z Y(a_1, \ldots, a_m),\]
and so 
\[(Z, U_1, \ldots, U_m)\independent_Z Y(a_1, \ldots, a_m)\]
by Corollary 5.7 in \citet{kallenbergfoundations}. It implies 
\[(A_1, \ldots, A_m)\independent_Z Y(a_1, \ldots, a_m)\] for all
$(a_1, \ldots, a_m)\in \mathcal{A}_1\otimes \cdots \otimes
\mathcal{A}_m$. The last step is because $A_j$'s are measurable
functions of $(Z, U_1, \ldots, U_m)$.

Now we prove the sufficiency. Assume that $Y(a_1, \ldots, a_m)\independent_Z (A_1, \ldots, A_m)$. Marginalizing out all but one $A_j$ gives 
\[Y(a_1, \ldots, a_m)\independent_Z A_j, j=1, \ldots, m.\]
By Theorem 5.10 in \citet{kallenbergfoundations}, there exists a measurable function $f_j:\mathcal{Z}\times [0,1]\rightarrow \mathcal{A}_j$ and a Uniform[0,1] random variable $\tilde{U}_j$ satisfying $\tilde{U}_j\independent (Z, Y(a_1, \ldots, a_m))$ such that the random variable $\tilde{A}_j = f_j(Z, \tilde{U}_j)$ satisfies 
\[\tilde{A}_j\stackrel{d}{=}A_j \text{  and  }(\tilde{A}_j, Z)\stackrel{d}{=}(A_j, Z).\]
Moreover, we have 
\[\tilde{A}_j\independent_Z Y(a_1, \ldots, a_m)\] with the same argument as the above necessity part.

Hence, by Proposition 5.6 in \citet{kallenbergfoundations}, 
\[P(\tilde{A}_j\in \cdot\mid Z, Y(a_1, \ldots, a_m))=P(\tilde{A}_j\in \cdot\mid Z)=P(A_j\in \cdot\mid Z) = P(A_j\in \cdot\mid Z, Y(a_1, \ldots, a_m)),\]
and so 
\[(\tilde{A}_j, Z, Y(a_1, \ldots, a_m))\stackrel{d}{=}(A_j, Z, Y(a_1, \ldots, a_m)).\]
By Theorem 5.10 in \citet{kallenbergfoundations}, we may choose some random variable $U_j$ such that 
\[U_j\stackrel{d}{=}\tilde{U}_j \text{   and    }(\tilde{A}_j, Z, Y(a_1, \ldots, a_m), U_j)\stackrel{d}{=}(A_j, Z, Y(a_1, \ldots, a_m), \tilde{U}_j).\]
In particular, we have 
\[U_j\independent (Z, Y(a_1, \ldots, a_m))\]
and 
\[(A_j, f_j(Z, U_j))\stackrel{d}{=}(\tilde{A}_j, f_j(Z, \tilde{U}_j).\] 
Since 
\[\tilde{A}_j=f_j(Z, \tilde{U}_j)\]
and the diagonal in $S^2$ is measurable, we have 
\[A_j \stackrel{a.s.}{=} f_j(Z, U_j).\] 
We then show $(U_1, \ldots, U_m)\independent (Z, Y(a_1, \ldots, a_m))$. By Theorem 5.10 in \citet{kallenbergfoundations}, there exists a measurable function $g_1:\mathcal{Y}\times \mathcal{Z}\times [0,1]\rightarrow [0,1]$ and a Uniform[0,1] random variable $\hat{U}_1$ satisfying $\hat{U}_1 \independent (Y(a_1, \ldots, a_m), Z)$ and
\[(Y(a_1, \ldots, a_m), Z, U_1)\stackrel{d}{=}(Y(a_1, \ldots, a_m), Z, g_1(Y(a_1, \ldots, a_m), Z, \hat{U}_1)).\]
Moreover, by 
\[U_1\independent_Z Y(a_1, \ldots, a_m),\]
we have 
\[g_1(Y(a_1, \ldots, a_m), Z, \hat{U}_1)\independent_Z Y(a_1, \ldots, a_m)\]
there exists some measurable function $g'_1:\mathcal{Z}\times [0,1]\rightarrow [0,1]$ such that 
\[g_1(Y(a_1, \ldots, a_m), Z, \hat{U}_1) = g'_1(Z, \hat{U}_1)\]
and
\[\hat{U}_1\independent (Z, Y(a_1, \ldots, a_m)).\] 
In other words, we have 
\[(Y(a_1, \ldots, a_m), Z, U_1)\stackrel{d}{=}(Y(a_1, \ldots, a_m), Z, g'_1(Z, \hat{U}_1)).\]
Repeating these steps, we again have from Theorem 5.10 in \citet{kallenbergfoundations} that there exists a measurable function $g_2:\mathcal{Y}\times \mathcal{Z}\times [0,1]^2\rightarrow [0,1]$ and a Uniform[0,1] random variable $\hat{U}_2$ satisfying 
\begin{align*}
&~(Y(a_1, \ldots, a_m), Z, U_1, U_2)\\
\stackrel{d}{=}&~(Y(a_1, \ldots, a_m), Z, g'_1(Z, \hat{U}_1), g_2(Y(a_1, \ldots, a_m), Z, \hat{U}_1, \hat{U}_2))
\end{align*}
and 
\[\hat{U}_2\independent (Z, Y(a_1, \ldots, a_m), \hat{U}_1).\] 
Again by 
\[U_1\independent_Z Y(a_1, \ldots, a_m),\]
we have a measurable function $g'_2:\mathcal{Z}\times [0,1]^2\rightarrow [0,1]$ that satisfies 
\begin{align*}
&~(Y(a_1, \ldots, a_m), Z, U_1, U_2)\\
\stackrel{d}{=}&~(Y(a_1, \ldots, a_m), Z, g'_1(Z, \hat{U}_1), g'_2(Z, \hat{U}_1, \hat{U}_2)).
\end{align*}
Repeating these steps $m$ times, we have 
\begin{align*}
&~(Y(a_1, \ldots, a_m), Z, U_1, U_2, \ldots, U_m)\\
\stackrel{d}{=}&~(Y(a_1, \ldots, a_m), Z, g'_1(Z, \hat{U}_1), g'_2(Z, \hat{U}_1, \hat{U}_2), \ldots, g'_m(Z, \hat{U}_1, \hat{U}_2, \ldots, \hat{U}_m))
\end{align*}
with 
\[\hat{U}_j\independent (Z, Y(a_1, \ldots, a_m), \hat{U}_1, \ldots, \hat{U}_{j-1}), j=1, \ldots, m.\]
We notice that the right side of the equation have conditional independence property 
\[(g'_1(Z, \hat{U}_1), g'_2(Z, \hat{U}_1, \hat{U}_2), \ldots, g'_m(Z, \hat{U}_1, \hat{U}_2, \ldots, \hat{U}_m))\independent_Z Y(a_1, \ldots, a_m).\]
This implies the same property holds for the left side of the equation, that is
\[(U_1, \ldots, U_m)\independent_Z Y(a_1, \ldots, a_m).\]
\end{proof}

\section{Proof of \Cref{lemma:factormodel}}
\label{sec:factormodelproof}

\begin{proofsk}
The lemma is an immediate consequence of Lemma 2.22 in
  \citet{kallenbergfoundations}, single ignorability, and the
  following observation: $\theta_{1:m}$ are point masses, so they are
  \textit{a priori} independent of the potential outcomes and the
  other latent variables,
\begin{align}
  \label{eq:factor-ind-requirement}
  (\theta_{1}, \ldots, \theta_m) \independent (Y_i(\mba), Z_i),
\end{align}
for any $\mba \in {\cal A}_1 \times \ldots \times {\cal A}_m$. See
  \Cref{sec:factormodelproof} for the full proof.
\end{proofsk}

\begin{proof}
For simplicity, we consider continuous random variables $A_{ij},
Z_{i}, \theta_j$. Also, we assume there are no single-cause
confounders. The proof can be easily extended to accommodate discrete
random variables and observed single-cause confounders.

We first state the regularity condition: The domains of the causes,
$\mathcal{A}_j,~j = 1, \ldots, m$ are Borel subsets of compact
intervals. Without loss of generality, we could assume $\mathcal{A}_j
= [0,1]$, $j = 1, \ldots, m.$

By Lemma 2.22 in \citet{kallenbergfoundations}, there exists some measurable
function $f_j:\mathcal{Z}
\times [0,1]\rightarrow [0,1]$ such that $\gamma_{ij}\independent Z_i$ and
\[A_{ij} = f_j(Z_i,
\gamma_{ij}).\]

Furthermore, there exists some measurable function $h_{ij}:
\mathcal{\Theta}\times[0,1]\rightarrow[0,1]$ such that
\[\gamma_{ij} =
h_{ij}(\theta_j, \omega_{ij}),\]
where  $\omega_{ij}\independent (Z_i,\theta_j)$ and $\omega_{ij}\sim \textrm{Uniform}[0,1].$
Lastly, we write 
\[U_{ij} =
F_{ij}^{-1}(\gamma_{ij})\sim \textrm{Uniform}[0,1],\] where $F_{ij}$
is the cumulative distribution function of $\gamma_{ij}$.

\Cref{eq:probmodel_confoundjoint} implies that $\omega_{ij}, i = 1, \ldots, n, j= 1,
\ldots, m$ are jointly independent: if they were not, then $A_{ij} = f_j(Z_i,
h_{ij}(\theta_j, \omega_{ij}))$ would not have been conditionally independent given
$Z_i, \theta_j$. 

We thus have \[A_{ij} = f_j(Z_i, U_{ij}),\] where $
U_{ij}:=F_{ij}^{-1}(h_{ij}(\theta_i, \omega_{ij}))$. 

In particular, $U_{ij}$ satisfies
\[(U_{i1}, \ldots, U_{im})\independent (Z_i, Y_i(a_{1},
\ldots, a_{m})).\] It is because $\theta_{1:m}$ are point masses; they satisfy
$(\theta_{1}, \ldots, \theta_{m})\independent (Z_i, Y_i(a_{1},
\ldots, a_{m}))$. 

Moreover, $\omega_{ij}\stackrel{iid}{\sim} \textrm{Uniform}[0,1].$ We thus
have 
\[(\omega_{i1}, \ldots, \omega_{im})
\independent Y_i(a_{1},
\ldots, a_{m}) \g Z_i.\] 
It is because we assume no single-cause
confounders: a single-cause confounder can induce dependence between
one of $\omega_{ij}$ and $Y_i(a_1, \ldots, a_m)$; a multi-cause
confounder cannot induce dependence between $(\omega_{i1}, \ldots,
\omega_{im})$ and $Y_i(a_{1},
\ldots, a_{m})$ because $\omega_{ij}$'s are independent. 

More precisely, no single-cause confounder implies 
\[\omega_{ij}
\independent Y_i(a_{1},
\ldots, a_{m}), j = 1, \ldots, m.\]
Because $\omega_{ij}, j=1,\ldots,m$
are jointly independent, we have $(\omega_{i1}, \ldots, \omega_{im})$
and $Y_i(a_{1},
\ldots, a_{m})$. In particular, for $m=2$, we have
\begin{align*}
&p(Y_i(a_{1},\ldots, a_{m}), \omega_{i1}, \omega_{i2}) \\
= &p(\omega_{i1})\cdot p(Y_i(a_{1},\ldots, a_{m})\g\omega_{i1})\cdot p(\omega_{i2}\g \omega_{i1}, Y_i(a_{1},\ldots, a_{m}))\\
=&p(\omega_{i1})\cdot p(Y_i(a_{1},\ldots, a_{m}))\cdot p(\omega_{i2})
\end{align*}
This implies 
\[(\omega_{i1}, \ldots, \omega_{im})
\independent Y_i(a_{1},
\ldots, a_{m}).\]
The last equality is because $\omega_{i2}$ is independent of
$\omega_{i1}$ and $Y_i(a_{1},\ldots, a_{m})$. Given $Z_i$ is inferred
without any knowledge of $Y_i(a_{1},\ldots, a_{m}))$, we have
$(\omega_{i1}, \ldots, \omega_{im})
\independent Y_i(a_{1},
\ldots, a_{m}) \g Z_i$.

If all pre-treatment single-cause confounders $W_i$ are observed, we
can simply expand $Z_i$; we consider $Z'_i := (Z_i, W_i)$ in the place of $Z_i$. The same argument
applies.
\end{proof}

\section{Proof of \Cref{prop:all_confounder}}
\label{sec:prop3proof}

We first define multi-cause confounders.  A multi-cause confounder is
a confounder that confounds two or more causes. The following
definition formalizes this idea. This definition stems from Definition
4 of \citet{vanderweele2013definition}.

\begin{defn}{(Multi-cause confounder)} A pretreatment covariate
$C_i$ is a multi-cause confounder if there exists a set of pre-treatment
covariates $V_i$ (possibly empty) and a set $J \subset\{1, \ldots,
m\}$ with $|J| \geq 2$ such that $(A_{ij})_{j\in J} \independent
Y_i(a_{i1}, \ldots, a_{im}) \g \sigma(V_i, C_i).$ Moreover, there is
no proper subset $S_i$ of $\sigma(V_i, C_i)$ and no proper subset $J'$
of $J$ such that $(A_{ij})_{j\in J'} \independent Y_i(a_{i1},
\ldots, a_{im}) \g S_i.$
\end{defn}

\begin{proofsk}
  This proposition is a consequence of
  \Cref{lemma:strong_ignorability_functional},
  \Cref{lemma:factormodel}, and a proof by contradiction. The
  intuition is that if a confounder affects two or more causes then
  the substitute confounder $Z_i$ must have captured it.  Why?  Obtain
  the substitute confounder $Z_i$ from a factor model;
  \Cref{lemma:strong_ignorability_functional} ensures that it
  satisfies ignorability.  Now suppose we omitted a multi-cause
  confounder $C_i$.  Then the substitute confounder $Z_i$ could not
  have satisfied ignorability: the omitted confounder $C_i$
  renders the causes and potential outcomes conditionally dependent,
  even given $Z_i$. \Cref{fig:graphical-arg} gives the intuition with
  a graphical model and \Cref{sec:prop3proof} gives a detailed proof.
\end{proofsk}

\begin{proof}
Without loss of generality, we work with two-cause confounders. The
proof is directly applicable to general multi-cause confounders. 

We prove the proposition by contradiction. Suppose there exists such a
multi-cause confounder $W_{i, bad}$ that is not measurable with
respect to $\sigma(Z_i)$; we show that $Z_i$ could not have satisfied
the factor model \Cref{eq:probmodel_confound}.

By Lemma 2.22 in \citet{kallenbergfoundations}, there exist some
function $f_j$ such that $A_{ij} = f_j(Z_i, U_{ij}),$ where
$U_{ij}\independent Z_i$. ($f_j$ is non-constant in $Z_i$.) 

Then $W_{i,bad}$ being a multi-cause confounder has two implications:
\begin{enumerate}
\item There exist $j_1, j_2$ and nontrivial functions $g_1, g_2$
such that $U_{ij_1} = g_1(W_{i, bad},
\gamma_{ij_1})$ and  $U_{ij_2} = g_2(W_{i, bad}, \gamma_{ij_2})$,
where $(\gamma_{ij_1}, \gamma_{ij_2})\independent W_{i, bad}$;
\item There exists a nontrivial function $h$ such that $Y_i({a_{i1}, \ldots,
a_{im}}) = h(W_{i, bad}, \epsilon),$ where $\epsilon\independent W_{i, bad}.$ 
\end{enumerate}
These two
statements implies that 
\[(U_{ij_1},
U_{ij_2})\centernot{\independent}Y_i({a_{i1}, \ldots, a_{im}}) \g
Z_i,\] 
because $W_{i, bad}$ is not measurable with respect to
$\sigma(Z_i)$. This implies 
\[(U_{i1}, \ldots,
U_{im})\centernot{\independent} Y_i(a_{i1}, \ldots, a_{im})\g Z_i.\]
It
contradicts the fact that $Z_i$ comes from the factor model
(\Cref{eq:probmodel_confoundjoint}) with $(U_{i1}, \ldots, U_{im})
\independent Y_i(a_{i1}, \ldots, a_{im})\g Z_i.$ Therefore, there does
not exist such a multi-cause confounder.
\end{proof}

\begin{corollary}
\label{corollary:allconfounder}
Under single ignorability, any confounder must be measurable
with respect to the $\sigma$-algebra generated by the substitute
confounder $Z_i$ and the observed covariates $X_i$.
\end{corollary}

\begin{proof} Because of single ignorability, a single-cause
  confounder must be measurable with respect to the observed
  covariates $X_i$.  Because of \Cref{prop:all_confounder}, a
  multi-cause confounder must be measurable with respect to the
  substitute confounder $Z_i$.  Thus all confounders must be
  measurable with respect to the union of the substitute confounders
  and the observed covariates $(Z_i, X_i)$.
\end{proof}

\Cref{corollary:allconfounder} shows how the ``no unobserved
single-cause confounder'' assumption is necessary for the
deconfounder; the substitute confounder $Z_i$ can only handle
multi-cause confounders. We also note that, technically, the single
ignorability assumption $A_{ij} \perp Y_i(\mba) \g X_i$ can be
weakened. Technically, ``no unobserved single-cause confounder'' only
requires that, for $j=1, \ldots, m$, 
\begin{enumerate}
\item There exist some random variable
$V_{ij}$ such that
\begin{align}
A_{ij} &\perp Y_i(\mba) \g X_i, V_{ij},\\
A_{ij} &\perp A_{i,-j} \g V_{ij},\label{eq:condindepV}
\end{align}
where $A_{i,-j}=\{A_{i1}, \ldots, A_{im}\}\backslash A_{ij}$ is the
complete set of $m$ causes excluding the $j$th cause;
\item There exists no proper subset of the sigma algebra
$\sigma(V_{ij})$ satisfies \Cref{eq:condindepV}.
\end{enumerate}
While this more technical version of single ignorability is weaker,
all theoretical results (i.e. identification results) in this paper
still hold. These results can be proved with the exact same arguments
as the current ones developed for the stronger version of single
ignorability $A_{ij} \perp Y_i(\mba) \g X_i$.

\section{Proof of \Cref{prop:no_mediator}}
\label{sec:nomediatorproof}

\begin{proofsk}
  The deconfounder separates inference of the substitute confounder
  from estimation of causal effects; see \Cref{alg:deconfounder}.
  This two-stage procedure guarantees that the substitute confounder
  is ``pre-treatment'' ; it does not contain a mediator.  The reason
  is that a mediator is, by definition, a post-treatment variable that
  affects the potential outcome.  Thus it (almost surely) cannot be
  identified with only the assigned causes and it is not measurable
  with respect to the observed (pre-treatment) covariates $X_i$.
  \Cref{sec:nomediatorproof} provides a detailed proof.
\end{proofsk}

\begin{proof}

We prove the proposition by contradiction. 

Consider a mediator $M$. We denote $M_i(a)$ as the potential value of
the mediator $M$ for unit $i$ when the assigned cause is $a$. We show
that $M_i(\mba_i)$ is almost surely not measurable with respect to
$Z_i$.

The deconfounder operating in two stages. Inferring the substitute
confounder $Z_i$ is separated from estimating the potential outcome.
It implies that the substitute confounder is independent of the
potential outcomes conditional on the causes $\mbA_i$: $Z_i \independent
Y_i(\mbA_i)\g \mbA_i.$ The intuition is that, without looking at
$Y_i(\cdot)$, the only dependence between $Z_i$ and $Y_i$ must come
from $\mbA_i$.

However, a mediator must satisfy $M_i(\mbA_i)\not \independent Y_i(\mbA_i)\g
\mbA_i$; otherwise, it has no mediation effect
\citep{imai2010identification}. If a mediator is measurable with
$Z_i$, then $Z_i\not \independent Y_i(\mbA_i)\g \mbA_i$. This contradicts
the conditional independence of $Z_i$ and $Y_i(\mbA_i)$ given $\mbA_i$. We
ensured this conditional independence by inferring the substitute
confounder $Z_i$ based only on the causes $\mbA_i$.
\end{proof}

As a consequence of single ignorability, the substitute
confounder, together with the observed covariates, captures
all confounders.

\section{Proof of \Cref{prop:main1}}

The first part is a direct consequence of
  \Cref{lemma:strong_ignorability_functional,lemma:factormodel}.

We now prove the second part. We provide two constructions.

We start with the first trivial one. For any assigned causes $\mbA_i$, we
consider a special case when $\mbA_i\stackrel{a.s.}{=} Z_i$. We have
\begin{align}
\label{eq:pointmassfactorizable} p(a_{i1},\ldots, a_{im}\g z_i) =
\delta_{z_i} =
\prod_{j=1}^m\delta_{z_{ij}}= \prod_{j=1}^m p(a_{ij}\g z_i)
\end{align} 
This step is due to point masses are factorizable.
Therefore, we can write the distribution of $\mbA_i$ in the form of a
factor model; we set $\theta_j \stackrel{a.s.}{=} 0, j = 1, \ldots, m$
and $Z_i \stackrel{a.s.}{=} \mbA_i$:
\begin{align}
p(\theta_{1:m}, z_{1:n}, \mba_{1:n}) &= p(\theta_{1:m})p(z_{1:n}\g
\theta_{1:m})p(\mba_{1:n}\g z_{1:n}, \theta_{1:m})\\ & =
p(\theta_{1:m}) p(z_{1:n})p(\mba_{1:n}\g z_{1:n})\\ & =
p(\theta_{1:m}) p(z_{1:n}) \prod_{i=1}^n\prod_{j=1}^m p(a_{ij}\g z_i)
\end{align} 
The second equality is due to $Z_i\independent
\theta_{1:m}$ and $\mbA_i\independent \theta_{1:m}\g Z_i$. They are
because $\theta_j$'s are point masses. The third equality is due to
the SUTVA assumption and \Cref{eq:pointmassfactorizable}.

Choosing $Z_i \stackrel{a.s.}{=} \mbA_i$, that is letting the substitute
confounder $Z_i$ be the same as the assigned causes $\mbA_i$, does not
help with causal inference; see a related discussion on overlap around
\Cref{eq:overlapmain}.

This result is only meant to exemplify the large capacity of factor
models. Finally, this $Z_i
\stackrel{a.s.}{=} \mbA_i$ example also illustrates the fact that a
factor model capturing $p(\mba_i)$ is not necessarily the true
assignment model.

We now present the second construction. It relies on copulas and the
Sklar's theorem. We follow the modified distribution function from
\citet{ruschendorf2009distributional}. Let $X$ be a real random
variable with distribution function $F$ and let $V\sim U(0,1)$ be
uniformly distributed on $(0,1)$ and independent of $X$. The modified
distribution function $F(x, \lambda)$ is defined by
\begin{align}
F(x, \lambda) := P(X<x)+\lambda P(X=x).
\end{align}
Then if we construct $U$ variables as
\begin{align}
U:= F(X, V),
\end{align}
then we have 
\begin{align}
U &= F(X-)+V(F(X) - F(X-)),\\
U &\stackrel{d}{=}Uniform(0,1),\\
X &\stackrel{a.s.}{=}F^{-1}(U).
\end{align}
Now we set $Z_{ij} = F_{ij}^{-1}(A_{ij}),$ where $F_{ij}$  is the
modified distribution function of $A_{ij}$. We also set $\theta_{j}, j
= 1, \ldots, m$ as point masses. The Sklar's theorem then implies
\begin{align}
p(\theta_{1:m}, z_{1:n}, \mba_{1:n}) &= p(\theta_{1:m})p(z_{1:n}\g
\theta_{1:m})p(\mba_{1:n}\g z_{1:n}, \theta_{1:m})\\ 
& =
p(\theta_{1:m}) p(z_{1:n})p(\mba_{1:n}\g z_{1:n}, \theta_{1:m})\\ 
& =
p(\theta_{1:m}) p(z_{1:n}) \prod_{i=1}^n\prod_{j=1}^m p(a_{ij}\g z_i, \theta_j)
\end{align}
The second equality is due to $\theta_{1:m}$ being point masses;
$\theta_j, j =1, \ldots, m$ can be considered as parameters of the
marginal distribution of $A_{ij}$. The third equality is due to the
SUTVA assumption and the Sklar's theorem.

This construction aligns more closely with the idea of the
deconfounder; it aims to capture multi-causes confounders that induces
the dependence structure, i.e. the copula. However, the deconfounder
is different from directly estimating the copula; the latter is a more
general (and harder) problem.

\section{Proof of \Cref{thm:deconfounderfactor}}

\label{sec:ateallidentifyproof}

\begin{proofsk}
\Cref{thm:deconfounderfactor} rely on two results: (1) single ignorability
and \Cref{prop:all_confounder} ensure $(X_i, Z_i)$ capture all
confounders; (2) the pre-treatment nature of $X_i$ and
\Cref{prop:no_mediator} ensure $(X_i, Z_i)$ capture no mediators.
These results assert ignorability given the substitute confounders
$Z_i$ and the observed covariates $X_i$. They greenlight us for causal
inference if the factor model admits consistent estimates of $Z_i$,
i.e. the substitute confounder has a degenerate distribution $P(Z_i\g
\mbA_i)=\delta_{f(\mbA_i)}$.

Given these results, \Cref{thm:deconfounderfactor} identifies the
average causal effect of all the causes by assuming
$\nabla_{\mba} f(a_1, \ldots, a_m) = 0$ almost everywhere and a
separable outcome model. These two assumptions let us identify the
average causal effect without assuming overlap.

More specifically, $\nabla_{\mba} f(a_1, \ldots, a_m) = 0$
roughly requires that the substitute confounder is a step function of
the all causes. In other words, we can partition all possible values
of $(a_1, \ldots, a_m)$ into countably many regions. In each region,
the value of the substitute confounder must be a constant. But the
substitute confounder can take different values in different regions.
This condition ensures that the average causal effect
$\E{Y}{Y_i(\mba)} - \E{Y}{Y_i(\mba')}$ is identifiable if $\mba$ and
$\mba'$ belong to the same region.

Further, we assume the outcome model be separable in the substitute
confounder and the causes. It roughly requires that there is no
interaction between the substitute confounder and the causes. This
separability condition lets us identify the average causal effect for
all values of $\mba$ and $\mba'$. The full proof is in
\Cref{sec:ateallidentifyproof}.
\end{proofsk}

\begin{proof}

For notational simplicity, denote $\mba=(a_1,
\ldots, a_m)$, $\mba'=(a'_1,
\ldots, a'_m)$, and $\mbA_i=(A_{i1},
\ldots, A_{im})$. We also write $f_\theta(\cdot) = f(\cdot).$

We start with rewriting $\E{Y}{Y_i(\mba)} -
\E{Y}{Y_i(\mba')}$ using the ignorability assumption and the separability
assumption.

First notice that
\begin{align}
\E{Y}{Y_i(\mba)}
=&\E{Z, X}{\E{Y}{Y_i(\mba)\g X_i, Z_i}}\\
% =&\E{Z, X}{\E{Y}{Y_i(\mba)\g A_{i1}=a_1, \ldots, A_{im}=a_m, X_i, Z_i}}\\
% =&\E{Z, X}{\E{Y}{Y_i\g A_{i1}=a_1, \ldots, A_{im}=a_m, X_i, Z_i}}\\
% =&\E{Z, X}{f_1(\mba, X_i) + f_2(Z_i)}\\
=&\E{X}{f_1(\mba, X_i)} + \E{Z}{f_2(Z_i)}.
\end{align}

The first equality is due to the tower property. The second equality
is due to the separability assumption. The third equality is due to
linearity of expectations.

Hence we have
\begin{align}
\E{Y}{Y_i(\mba)} - \E{Y}{Y_i(\mba')}
=&\E{X}{f_1(\mba, X_i)} - \E{X}{f_1(\mba', X_i)}\\
=&\int_{C(\mba,\mba')} \nabla_{\mba}\E{X}{f_1(\mba, X_i)} \dif \mba,\label{eq:integralrep}
\end{align}
where $C(\mba,\mba')$ is a line where $\mba$ and
$\mba'$ are the end points. The second equality is due to the
fundamental theorem of calculus.

Next we see how the gradient of the potential outcome function
$\nabla_{\mba}\E{X}{f_1(\mba, X_i)}$ relates to the
gradient of the outcome model we fit. The key idea here is that the
two gradients are equal in regions $\{\mba: f(\mba) = c\}$
for each $c$.

We will rely on the consistent substitute confounder assumption.
Notice that, for almost all $\mba$, we have
\begin{align}
\label{eq:f1f3}
\nabla_{\mba}\E{X}{f_1(\mba)} =
\nabla_{\mba}\E{X}{f_3(\mba)}
\end{align}

It is due to two observations. The first observation is that
\begin{align}
&\nabla_{\mba}\E{X}{\E{Y}{Y_i\g Z_i=f(\mba), A_{i}=\mba,X_i}}\\
=&\nabla_{\mba}\E{X}{\E{Y}{Y_i(\mba)\g Z_i=f(\mba), A_{i}=\mba,X_i}}\\ 
=&\nabla_{\mba}\E{X}{\E{Y}{Y_i(\mba)\g Z_i=f(\mba), X_i}}\\
=&\nabla_{\mba}\E{X}{f_1(\mba, X_i)} + \nabla_{\mba}
f_2(f(\mba))\\ 
=&\nabla_{\mba}\E{X}{f_1(\mba, X_i)}
+ \nabla_{f(\mba)} f_2 \cdot \nabla_{\mba} f(\mba)\\
=&\nabla_{\mba}\E{X}{f_1(\mba, X_i)} \label{eq:derivativef1}
\end{align}

The first equality is due to \gls{SUTVA}. The second equality is due
to is due to \Cref{prop:main1}.1: $Y_i(\mba)\perp \mbA_i \g X_i, Z_i$. The third equality is due to the
separability condition. The fourth equality is due to the chain rule.
The fifth equality is due to $\nabla_{\mba} f(\mba) = 0$
up to a set of Lebesgue measure zero.

The second observation is that
\begin{align}
&\nabla_{\mba}\E{X}{\E{Y}{Y_i\g Z_i=f(\mba), \mbA_i=\mba, X_i}}\\
=&\nabla_{\mba}\E{X}{f_3(\mba, X_i)} + \nabla_{\mba}f_4(f(\mba))\\
=&\nabla_{\mba}\E{X}{f_3(\mba, X_i)}
\end{align}

Hence \Cref{eq:f1f3} is true because $f_1$ and $f_3$ are continuously
differentiable.

Therefore, we have
\begin{align}
&\E{Y}{Y_i(\mba)} - \E{Y}{Y_i(\mba')}\\
=&\int_{C(\mba,\mba')} \nabla_{\mba}\E{X}{f_1(\mba, X_i)} \dif \mba\\
=&\int_{C(\mba,\mba')} \nabla_{\mba}\E{X}{f_3(\mba, X_i)} \dif \mba\\
=&\E{X}{f_3(\mba, X_i)} - \E{X}{f_3(\mba', X_i)}\\
=&(\E{X}{f_3(\mba, X_i)} + \E{}{f_4(Z_i)}) - (\E{X}{f_3(\mba', X_i)} + \E{}{f_4(Z_i)}))\\
=&\int \E{Y}{Y_i\g \mbA_i=\mba', X_i, Z_i} P(Z_i, X_i)\dif Z_i\dif X_i \nonumber\\
&- \int\E{Y}{Y_i\g \mbA_i=\mba, X_i, Z_i}P(Z_i, X_i)\dif Z_i\dif X_i\\
=&\E{Z, X}{\E{Y}{Y_i\g \mbA_i=\mba, Z_i, X_i}} - \E{Z, X}{\E{Y}{Y_i\g \mbA_i=\mba', Z_i, X_i}}.
\end{align}

The first equality is due to \Cref{eq:integralrep}. The second
equality is due to \Cref{eq:f1f3}. The third equality is due to the
fundamental theorem of calculus. The fourth equality is due to simple
algebra. The fifth equality is due to the separability condition.

\end{proof}

\section{Proof of \Cref{thm:atesubsetidentify}}
\label{sec:atesubsetidentifyproof}

\begin{proof}

\Cref{lemma:strong_ignorability_functional} and
\Cref{lemma:factormodel}, together with single ignorability, ensures
that the substitute confounder $Z_i$ and the observed covariate $X_i$
satisfies
\begin{align}
\label{eq:fullstrongignorapp}
(A_{i1}, \ldots, A_{im})\independent Y_i(a_{i1}, \ldots,
a_{im})\g Z_i, X_i.
\end{align} 
Therefore, we have
\begin{align}
&\E{A_{(k+1):m}}{\E{Y}{Y_i(a_{1:k}, A_{i,(k+1):m})}}\\
=&\E{A_{(k+1):m}}{\E{Y}{Y_i(a_1, \ldots, a_k, A_{i,k+1}, \ldots, A_{im})}}\\
=&\E{Z, X}{\E{A_{(k+1):m}}{\E{Y}{Y_i(a_1, \ldots, a_k, A_{i,k+1}, \ldots, A_{im})\g Z_i, X_i}}}\\
=&\E{Z, X}{\E{A_{(k+1):m}}{\E{Y}{Y_i(a_1, \ldots, a_k, A_{i,k+1}, \ldots, A_{im})\g Z_i, X_i, A_{i1}=a_1, \ldots, A_{ik}=a_k}}}\\
=&\E{Z, X}{\E{A_{(k+1):m}}{\E{Y}{Y_i(A_{i1}, \ldots, A_{ik}, A_{i,k+1}, \ldots, A_{im})\g Z_i, X_i, A_{i1}=a_1, \ldots, A_{ik}=a_k}}}\\
=&\E{Z, X}{\E{A_{(k+1):m}}{\E{Y}{Y_i\g Z_i, X_i, A_{i1}=a_1, \ldots, A_{ik}=a_k}}}\\
=&\E{Z, X}{\E{Y}{Y_i\g Z_i, X_i, A_{i1}=a_1, \ldots, A_{ik}=a_k}}\\
=&\E{Z,X}{\E{Y}{Y_i\g Z_i, X_i, A_{i,1:k}=a_{1:k}}}
\end{align}
The first equality is an expansion of the notations. The second
equality is due to the tower property. The third equality is due to
\Cref{eq:fullstrongignorapp}. The fourth equality is due to
$A_{i1}=a_1, \ldots, A_{ik}=a_k$. The fifth equality is due to
\gls{SUTVA}. The sixth equality is due to the inner expectation does
not depend on $A_{(k+1):m}$.

Therefore, we have
\begin{align*}
  &\E{A_{(k+1):m}}{\E{Y}{Y_i(a_{1:k}, A_{i,(k+1):m})}} -  \E{A_{(k+1):m}}{\E{Y}{Y_i(a'_{1:k}, A_{i,(k+1):m})}}\\
=&\E{Z,X}{\E{Y}{Y_i\g Z_i, X_i, A_{i,1:k}=a_{1:k}}} - \E{Z,X}{\E{Y}{Y_i\g Z_i, X_i, A_{i,1:k}=a'_{1:k}}}
\end{align*}
by the linearity of expectation.

Finally, $\E{Z,X}{\E{Y}{Y_i\g Z_i, X_i, A_{i,1:k}=a_{1:k}}}$ can be
estimated from the observed data because (1) $A_{i,1:k}$ satisfy
overlap with respect to $(Z_i, X_i)$ and (2) the substitute confounder
$Z$ can be consistently estimated.
\end{proof}

\section{Proof of \Cref{thm:conditionalpoidentify}}
\label{sec:conditionalpoidentifyproof}

\begin{proof} As with \Cref{thm:deconfounderfactor} and
\Cref{thm:atesubsetidentify}, \Cref{thm:conditionalpoidentify} relies
on the ignorability given the substitute confounders $Z_i$ and the
observed covariates $X_i$ due to \Cref{prop:all_confounder} and
\Cref{prop:no_mediator}.

Given this ignorability, \Cref{thm:conditionalpoidentify} identifies
the mean potential outcome of an individual given its current cause
assignment $A_{i} = (a_1, \ldots, a_m)$; it only requires that the new
cause assignment of interest $(a'_1,
\ldots, a'_m)$ lead to the same substitute confounder estimate:
$f(a_1, \ldots, a_m) = f(a'_1, \ldots, a'_m)$.

To prove identification, we rewrite this conditional mean potential
outcome
\begin{align}
&\E{Y}{Y_i(a'_1, \ldots, a'_m)\g A_{i1}=a_1, \ldots, A_{im}=a_m}\\
=&\E{Z, X}{\E{Y}{Y_i(a'_1, \ldots, a'_m)\g A_{i1}=a_1, \ldots,
A_{im}=a_m, Z_i, X_i}}\\
=&\E{X}{\E{Y}{Y_i(a'_1, \ldots, a'_m)\g A_{i1}=a_1, \ldots,
A_{im}=a_m, Z_i = f(a_1, \ldots, a_m), X_i}}\\
=&\E{X}{\E{Y}{Y_i(a'_1, \ldots, a'_m)\g
A_{i1}=a'_1, \ldots, A_{im}=a'_m, Z_i=f(a_1, \ldots, a_m), X_i}}\\
=&\E{Z,X}{\E{Y}{Y_i \g
A_{i1}=a'_1, \ldots, A_{im}=a'_m, Z_i, X_i}}
\end{align}

The first equality is due to the tower property. The second equality
is due to the consistency requirement on the substitute confounder
$P(Z_i\g \mbA_i)=\delta_{f(\mbA_i)}$. The third equality is due to
ignorability given $Z_i, X_i$. The fourth equality is estimable from
the data because $f(a_1, \ldots, a_m) = f(a'_1, \ldots, a'_m)$. Hence
the nonparametric identification of $\E{Y}{Y_i(a'_1, \ldots, a'_m)\g
A_{i1}=a_1, \ldots, A_{im}=a_m}$ is established. We note that this
identification result does not require overlap. 
\end{proof}

\section{Details of \Cref{subsec:gwasstudy}}

\label{sec:genesimdetail}

We follow \citet{song2015testing} in simulating the allele frequencies. We
present the full details here.

We simulate the $n\times m$ matrix of genotypes $A$ from $A_{ij}\sim
\textrm{Binomial}(2, F_{ij}),$ where $F$ is the $n\times m$ matrix of allele
frequencies. Let $F = \Gamma S,$ where $\Gamma$ is $n\times d$ and $S$ is
$d\times m$ with $d\leq m$. The $d\times m$ matrix $S$ encodes the genetic
population structure. The $n\times d$ matrix $\Gamma$ maps how the structure
affects the allele frequencies of each SNP. \Cref{tab:allelesim} details how
we generate $\Gamma$ and $S$ for each simulation setup.

\begin{table}[t]
  \begin{center}
    \begin{tabular}{p{3cm}p{12cm}} 
      \toprule
      Model & Simulation details\\
      \midrule      Balding-Nichols Model (Balding-Nichols)&Each row $i$ of
$\Gamma$ has i.i.d. three independent and identically distributed
draws from the Balding- Nichols model:
$\gamma_{ik}\stackrel{iid}{\sim}\textrm{BN}(p_i, F_i),$ where
$k\in\{1,2,3\}$. The pairs $(p_i, F_i)$ are computed by randomly
selecting a SNP in the HapMap data set, calculating its observed
allele frequency and estimating its $F_{ST}$ value using the Weir \&
Cockerham estimator \citep{weir1984estimating}. The columns of $S$
were Multinomial$(60/210, 60/210, 90/210)$, which reflect the
subpopulation proportions in the HapMap data set. We simulate $n =
100000$ SNPs and $m = 5000$ individuals.\\
      \midrule

1000 Genomes Project (TGP) &The matrix $\Gamma$ was generated by sampling
$\gamma_{ik}\stackrel{iid}{\sim}0.9\times$Uniform$(0,0.5),$ for $k = 1, 2$ and
setting $\gamma_{i3} = 0.05$. In order to generate $S$, we compute the first
two principal components of the TGP genotype matrix after mean centering each
SNP. We then transformed each principal component to be between $(0, 1)$ and
set the first two rows of $S$ to be the transformed principal components. The
third row of $S$ was set to 1, i.e. an intercept. We simulate $m = 100000$ and
$n = 1500$, where $m$ was determined by the number of individuals in the TGP
data set.\\
      \midrule

Human Genome Diversity Project (HGDP)& Same as TGP but generating $S$ with the
HGDP genotype matrix.\\ 
      \midrule

Pritchard-Stephens-Donnelly (PSD)&Each row $i$ of $\Gamma$ has i.i.d. three
independent and identically distributed draws from the Balding- Nichols model:
$\gamma_{ik}\stackrel{iid}{\sim}\textrm{BN}(p_i, F_i),$ where $k\in\{1,2,3\}$. The
pairs $(p_i, F_i)$ are computed by randomly selecting a SNP in the HGPD data
set, calculating its observed allele frequency and estimating its $F_{ST}$
value using the Weir \& Cockerham estimator \citep{weir1984estimating}. The
estimator requires each individual to be assigned to a subpopulation, which
were made according to the $K = 5$ subpopulations from the analysis in
\citet{rosenberg2002genetic}. The columns of $S$ were sampled $(s_{1j},
s_{2j}, s_{3j}\stackrel{iid}{\sim}$Dirichlet$(\alpha, \alpha, \alpha)$ for $j
= 1, \ldots, m, \alpha = 0.01, 0.1, 0.5, 1.$ We simulate $m = 100000$ and $n =
5000$.
\\ 
      \midrule

Spatial&The matrix $\Gamma$ was generated by sampling
$\gamma_{ik}\stackrel{iid}{\sim}0.9\times$Uniform$(0,0.5),$ for $k = 1, 2$ and
setting $\gamma_{i3} = 0.05$. The first two rows of $S$ correspond to
coordinates for each individual on the unit square and were set to be
independent and identically distributed samples from Beta$(\tau, \tau), \tau = 0.1,
0.25, 0.5, 1$, while the third row of $S$ was set to be 1, i.e. an intercept.
As $\tau\rightarrow 0$, the individuals are placed closer to the corners of the
unit square, while when $\tau = 1$, the individuals are distributed uniformly. We
simulate $m = 100000$ and $n = 5000$.\\
      \bottomrule
    \end{tabular}
    \caption{Simulating allele frequencies.\label{tab:allelesim}}
  \end{center}
\end{table}

For each simulation scenarios, we generate 100 independent studies. We then
simulate a trait; we consider two types: one continuous and one binary. For
each trait, three components contributing to the trait: causal signals
$\sum^m_{j=1}\beta_ja_{ij}$, confounder $\lambda_i$, and random effects
$\epsilon_i$.

First, without loss of generality, we set the first $1\%$ of the $m$ SNPs to
be the true causal SNPs $(\beta_j\ne 0,
\beta_j\stackrel{iid}{\sim}\cN(0,0.5)$. We set $\beta_j = 0$ for the rest of
the SNPs.

Notice that the SNPs are affected by some latent population structure. We
simulate the confounder $\lambda_i$ and the random effects $\epsilon_i$ so
that they depend on the latent population structure as well.

For the confounder $\lambda_i$, we first perform $K$-means clustering on the
columns of $S$ with $K = 3$ using Euclidean distance. This assigns each
individual $i$ to one of three mutually exclusive cluster sets
$\mathcal{S}_1,\mathcal{S}_2,\mathcal{S}_3$, where $\mathcal{S}_k\subset\{1,
2, \ldots, n\}.$ Set $\lambda_j = k$ if $j\in\mathcal{S}_k, k = 1, 2, 3.$

We then simulate the random effects $\epsilon_i$. Let $\tau^2_1, \tau^2_2,
\tau^2_3\stackrel{iid}{\sim}$InvGamma(3,1), and set $\sigma^2_i=\tau_k^2$ for
all $j\in\mathcal{S}_i, k = 1, 2, 3.$ Draw $\epsilon_i\sim \cN(0,
\sigma_i^2).$

We control the \gls{SNR} to mimic the highly noisy nature of
\gls{GWAS} data sets: in the low \gls{SNR} setting, we let the causal
signals $\sum^m_{j=1}\beta_ja_{ij}$ contribute to $\nu_{gene} = 0.1$
of the variance, the confounder $\lambda_i$ contribute $\nu_{conf} =
0.2$, and the random effects $\epsilon_i$ contribute $\nu_{noise} =
0.7$. In the high \gls{SNR} setting, we have $\nu_{gene} = 0.4$,
$\nu_{conf} = 0.4$, and $\nu_{noise} = 0.2$.

We set 
\begin{align}
\lambda_i&\leftarrow \left[\frac{s.d.\{\sum^m_{j=1}\beta_ja_{ij}\}_{i=1}^n}{
\sqrt{\nu_{gene}}}\right]
\left[\frac{\sqrt{\nu_{conf}}}{s.d.\{\lambda_i\}_{i=1}^n}\right]\lambda_i,\\
\epsilon_i&\leftarrow \left[\frac{s.d.\{\sum^m_{j=1}\beta_ja_{ij}\}_{i=1}^n}{
\sqrt{\nu_{gene}}}\right]
\left[\frac{\sqrt{\nu_{noise}}}{s.d.\{
\epsilon_i\}_{i=1}^n}\right]
\epsilon_i.
\end{align}

We finally generate a real-valued outcome from a linear model and a binary
outcome from a logistic model: 
\begin{align}
y_{i, quant} &= \sum^m_{j=1}\beta_ja_{ij} +
\lambda_i + \epsilon_i,\\
y_{i, binary} &\sim
\text{Bernoulli}\left(\frac{1}{1+\exp(\sum^m_{j=1}\beta_ja_{ij} + \lambda_i +
\epsilon_i)}\right).
\end{align}

}
\clearpage
\putbib[BIB1]
\end{bibunit}

\end{document}